\documentclass[10pt]{article}
\usepackage[accepted]{tmlr}

\usepackage{amsmath,amsfonts,bm}

\def\eqref#1{equation~\ref{#1}}

\def\1{\bm{1}}

\DeclareMathAlphabet{\mathsfit}{\encodingdefault}{\sfdefault}{m}{sl}
\SetMathAlphabet{\mathsfit}{bold}{\encodingdefault}{\sfdefault}{bx}{n}

\usepackage{titlesec}
\usepackage{hyperref}
\usepackage{url}
\usepackage[utf8]{inputenc}
\usepackage[T1]{fontenc}
\usepackage{booktabs}
\usepackage{amsfonts}
\usepackage{nicefrac}
\usepackage{microtype}
\usepackage{xcolor}
\usepackage{pifont}
\usepackage{multirow}
\usepackage{colortbl}
\usepackage{graphicx}
\usepackage{subfigure}
\usepackage{subcaption}
\usepackage{enumitem}

\usepackage{amsmath}
\usepackage{amssymb}
\usepackage{mathtools}
\usepackage{amsthm}
\usepackage[capitalize,noabbrev]{cleveref}
\usepackage[textsize=tiny]{todonotes}
\usepackage{natbib}
\usepackage{float}
\usepackage{listings}
\newcommand{\added}[2][]{#2}
\newcommand{\deleted}[2][]{}

\newcommand{\comment}[2][]{}
\newcommand{\listofchanges}[1][]{}
\newcommand{\definechangesauthor}[2][]{}
\definecolor{codebg}{rgb}{0.95,0.95,0.95}
\DeclareUnicodeCharacter{03C0}{\ensuremath{\pi}}

\definecolor{purple}{rgb}{0.65, 0.5, 1.0}
\definecolor{light}{rgb}{0.82, 0.75, 1.0}
\definecolor{mid}{rgb}{0.75,0.58,1.0}
\definecolor{lighter}{rgb}{0.9, 0.78, 1.0}

\theoremstyle{plain}
\newtheorem{theorem}{Theorem}[section]

\newtheorem{lemma}[theorem]{Lemma}
\newtheorem{corollary}[theorem]{Corollary}
\theoremstyle{definition}

\theoremstyle{remark}
\newtheorem{remark}[theorem]{Remark}

\title{A Unified Theory of Sinusoidal Activation Families for Implicit Neural Representations}

\author{\name Alireza Morsali \email alireza.morsali@mail.mcgill.ca \\
\addr McGill University
\AND
\name MohammadJavad Vaez
\email{mvaez@student.unimelb.edu.au,\quad mjvaez.cs@gmail.com}\\
\addr University of Melbourne\\
\addr ARC Centre of Excellence for the Mathematical Analysis of Cellular Systems (MACSYS)
\AND
\name Mohammadhossein Soltani \email mohammad.hossein19soltani@gmail.com \\
\addr Independent Researcher
\AND
\name Amirhossein Kazerouni \email amirhossein@cs.toronto.edu \\
\addr University of Toronto\\
Vector Institute\\
University Health Network
\AND
\name Babak Taati \email babak.taati@uhn.ca \\
\addr University of Toronto\\
Vector Institute\\
University Health Network
\AND
\name Morteza Mohammad-Noori \email mmnoori@ut.ac.ir \\
University of Tehran
}

\def\month{6}
\def\year{2026}
\def\openreview{\url{https://openreview.net/forum?id=ZDmBPYptbL}}

\begin{document}

\maketitle

\begin{abstract}
\vspace{-1em}
Implicit Neural Representations (INRs) model continuous signals with compact neural networks and have become a standard tool in vision, graphics, and signal processing. A central challenge is accurately capturing fine detail without heavy hand-crafted encodings or brittle training heuristics. Across the literature, periodic activations have emerged as a compelling remedy: from SIREN, which uses a single sinusoid with a fixed global frequency, to more recent architectures employing multiple sinusoids and, in some cases, trainable frequencies and phases. 
We study this \emph{family} of sinusoidal activations and develop a principled theoretical and practical framework for trainable sinusoidal activations in INRs. Concretely, we instantiate this framework with \textbf{S}inusoidal \textbf{T}rainable \textbf{A}ctivation \textbf{F}unctions \textbf{(STAF)}, a Fourier-like activation whose amplitudes, frequencies, and phases are learned. Our analysis (i) establishes a Kronecker-equivalence construction that expresses trainable sinusoidal activations with standard sine networks and quantifies expressive growth, (ii) characterizes how the Neural Tangent Kernel (NTK) spectrum changes under trainable sinusoidal parameterization, and (iii) provides an initialization that yields standard normal post-activations without asymptotic central limit theorem (CLT) arguments. 
Empirically, on images, audio, shapes, inverse problems (super-resolution, denoising) and NeRF, \added{STAF is competitive and often stronger on distortion-oriented reconstruction metrics such as PSNR/SSIM across the evaluated INR tasks, with favorable parameter efficiency under layer-wise sharing.} While periodic activations can alleviate \emph{practical manifestations} of spectral bias, our results indicate they do not eliminate it; instead, \added{trainable sinusoids can improve the observed capacity–optimization trade-off in the evaluated settings.}
\noindent\textbf{Project page:} \url{https://alirezamorsali.github.io/staf/}
\end{abstract}

\addtocontents{toc}{\protect\setcounter{tocdepth}{-1}} 
\section{Introduction}
Implicit Neural Representations (INRs) approximate continuous signals by mapping coordinates to signal values via multilayer perceptrons (MLPs). INRs underpin modern view synthesis, 3D geometry, and compact signal representations~\citep{mildenhall2020nerf,rahaman2019spectral,Siren,tancik2020fourier,liu2024finer,kazerouni2024incode,saragadam2023wire}. A recurring observation is that standard ReLU networks exhibit a training-time preference toward low frequencies~\citep{rahaman2019spectral}, which complicates faithful recovery of high-detail content unless aided by positional encodings or architectural changes.

A parallel thread replaces or augments activations themselves. SIREN~\citep{Siren} showed that a global sinusoid dramatically improves fidelity, and follow-up work explored alternative periodic bases (e.g., Gabor-like~\citep{saragadam2023wire}) and multi-sine constructions; some of these make the sinusoidal parameters trainable. Collectively, these results suggest that \emph{sinusoidal activations} are a broad and useful family for INRs.

In this paper, we organize and extend this line of work by providing theory and practice for \emph{trainable} sinusoidal activations. We instantiate the framework with \textbf{S}inusoidal \textbf{T}rainable \textbf{A}ctivation \textbf{F}unction \textbf{(STAF)}, which represents each activation by a Fourier-like series with learnable amplitudes, frequencies, and phases. Our goals are twofold:
(i) \textbf{Theoretical grounding.} We formalize how trainable sinusoids increase effective frequency support and how this affects the Neural Tangent Kernel (NTK) spectrum and optimization dynamics. (ii) \textbf{Practical recipe.} We provide a layer-wise shared activation parameterization and a simple initialization producing unit-variance post-activations. \added{In practice, the theory suggests three concrete guidelines: use STAF when fine detail or mixed frequencies are important, prefer layer-wise sharing for parameter efficiency, and treat $\tau$ as a capacity knob that should remain small by default unless quality gains justify the extra compute.} \textbf{Contributions:}
\begin{itemize}[leftmargin=*]
\item \textbf{Unified view of sinusoidal activations.} We position SIREN and subsequent multi-sinusoid variants within a common class and study \emph{trainable} sinusoidal activations for INRs, clarifying when and why they help.
\item \textbf{Kronecker-equivalence \& expressive growth.} We prove that networks with trainable sinusoidal activations admit an equivalent construction using plain sine activations with structured (Kronecker) weights. This yields quantifiable growth of the set of potential frequencies by a factor of $\tau^K$ (Theorems~\ref{Kronecker_theorem}, \ref{asymptotic_behavior}).
\item \textbf{Capacity--convergence analysis via NTK.} We analyze how trainable sinusoids reshape the NTK eigen-spectrum and relate larger leading eigenvalues to faster learning of specific components, providing criteria that connect activation parameters to convergence behavior.
\item \textbf{Initialization without CLT assumptions.} We give an initialization that directly yields unit-variance post-activations for sinusoidal series (Theorem~\ref{initialization_theorem}), removing dependence on distributional approximations.
\item \textbf{Empirical study across tasks.} On images, audio, shapes, inverse problems, and NeRF, STAF is competitive and often better in PSNR/SSIM and convergence speed, with favorable parameter efficiency under layer-wise sharing. These advancements stem from faster convergence and enhanced accuracy, \added{“positioning STAF as a strong and efficient alternative within the family of sinusoidal INR activations}, including INCODE~\citep{kazerouni2024incode}, FINER~\citep{liu2024finer}, WIRE~\citep{saragadam2023wire}, SIREN~\citep{Siren}, Gaussian~\citep{ramasinghe2022beyond}, FFN~\citep{tancik2020fourier}.
\end{itemize}

\label{sec:intro}

\begin{figure*}[t]
    \centering
    \includegraphics[width=0.8\textwidth, keepaspectratio]{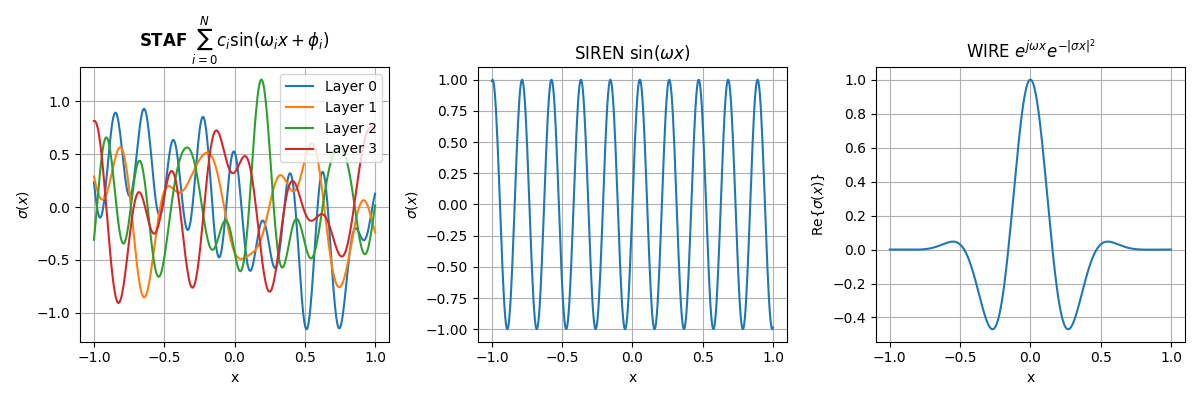} 
    \vspace{-1em}
    \caption{Activation functions used in INRs plotted over the range [-1, 1]. STAF utilizes a parameterized Fourier series activation, offering flexible frequency-domain adaptation. SIREN employs a sinusoidal function, providing a periodic activation landscape. WIRE employs a complex Gabor wavelet activation, balancing spatial and frequency localization.}
    \label{fig:activation}
\end{figure*}

\section{Related Works}
\label{sec:relatedworks}

INRs have been applied to a wide range of signals (including images, audio, signed distance fields (SDFs), and compressed representations) by fitting coordinate MLPs to continuous targets. Early work with sine activations demonstrated that periodic nonlinearities can substantially boost fidelity in coordinate-based modeling~\citep{Siren}. Broader architectural strategies further improved capacity and efficiency, including dual-MLP modulation~\citep{mehta2021modulated}, dividing inputs into grids or ensembles~\citep{aftab2022multi,kadarvish2021ensemble}, and adaptive allocation of model resources~\citep{acorn,saragadam2022miner}.

For 3D scenes, NeRF couples INRs with differentiable volume rendering~\citep{mildenhall2020nerf}, inspiring numerous advances that trade off fidelity, speed, and memory. These include higher-quality rendering and antialiasing~\citep{barron2023zip,wu2023neural}, factorized or accelerated pipelines~\citep{chen2024far,reiser2021kilonerf}, Jacobian-aware or regularized training~\citep{xu2023jacobinerf}, fast super-resolution and compact encodings~\citep{lin2024fastsr,kazerouni2024incode}, and extensions to more general or dynamic settings~\citep{uy2024nerf,li2025nerf}.

Activation design has shaped neural network trainability and expressivity. Saturating sigmoids suffer from vanishing gradients, motivating unbounded or piecewise-linear alternatives such as ReLU and its variants~\citep{nair2010rectified,maas2013rectifier,elfwing2018sigmoid,hendrycks2016gaussian}. Trainable or adaptive activations (e.g., Swish, TanhSoft, SinLU) further tailor nonlinearity to data~\citep{ramachandran2017searching,biswas2021tanhsoft,paul2022sinlu}. In the INR setting, training dynamics interact with frequency content: ReLU-based networks tend to fit lower frequencies first~\citep{rahaman2019spectral}. This observation motivated periodic activations, which embed high-frequency structure directly into the activation space. Although early attempts at periodic networks highlighted practical training difficulties~\citep{lapedes1987nonlinear,parascandolo2016taming}, subsequent INR work demonstrated stable recipes and strong reconstructions with sinusoidal activations and modulated architectures~\citep{Siren,mehta2021modulated}. More recently, variable-periodic activations explicitly \emph{tune} spectral bias in INRs~\citep{liu2024finer}, and training protocols for sinusoidal networks have been refined for robustness and stability~\citep{novello2025tuner}.

Beyond changing the activation itself, explicit frequency mappings inject sinusoidal structure at the input. Classical Fourier Neural Networks proposed trigonometric expansions inside the network~\citep{gallant1988there}, while modern Fourier feature mappings provide a simple, widely adopted mechanism for enriching frequency support in INRs~\citep{tancik2020fourier}. These approaches are complementary to periodic activations and are often combined in practice. \added{More recently, Fourier Learning Machines (FLMs)~\citep{rubel2025fourier} study explicit Fourier-output architectures, in which the network is designed so that its output directly represents a multidimensional nonharmonic Fourier series with learnable frequencies, amplitudes, and phase shifts, including a complete separable-basis formulation in multiple dimensions. In contrast, our focus is not on parameterizing the final predictor as an explicit Fourier series, but on trainable sinusoidal hidden activations for deep coordinate MLPs used as INRs. Thus, Fourier features act at the input level, FLMs at the output-architecture level, whereas STAF addresses activation design, expressivity, initialization, and optimization within standard INR backbones.}

Recent work also explores alternative functional parameterizations for coordinate models, including Kolmogorov--Arnold Networks (KANs) and polynomial KAN variants~\citep{liu2024kan,ss2024chebyshev}. \added{Such methods target expressive yet structured representations, but they differ from sinusoidal INR models in both mechanism and inductive bias. KANs replace linear weights by learnable univariate edge functions, typically spline-parameterized, and thus define an alternative computational graph to standard MLPs rather than a new activation family. STAF, in contrast, preserves the standard INR MLP architecture and changes only the hidden activation family to a trainable sinusoidal series. Consequently, KANs are best viewed as an alternative to MLP parameterization itself, whereas STAF is an activation design within the MLP/INR paradigm, specifically aimed at learned periodic nonlinearities.}

\section{A unified view of sinusoidal activations}
We formulate \emph{STAF} as a family of sinusoidal activation functions that \emph{encompasses} prior INR activations, from SIREN’s single fixed-frequency sine to multi-sine and trainable-frequency variants ~\citep{liu2024finer,novello2025tuner}. This framing enables us to study common mechanisms (expressivity, initialization, and optimization dynamics) without enumerating each prior design choice. For clarity and to avoid repetition, we, therefore, use \emph{SIREN as the canonical base model} in comparisons and discussions, noting that it arises as a special case within our formulation.
\subsection{INR Problem Formulation}

INRs employ MLPs as a method for representing continuous data.
\added{Let $\boldsymbol{r} \in \mathbb{R}^D$} denote the input coordinate (e.g., a pixel coordinate, spatial point, or spatiotemporal location).
At the core of INR is the function $ f_{\boldsymbol{\theta}}: \mathbb{R}^{F_0} \rightarrow \mathbb{R}^{F_L} $, where $ F_0 $ and $ F_L $ represent the dimensions of the input and output spaces, respectively, and $ \boldsymbol{\theta} $ denotes the parameters of the MLP. The objective is to approximate a target function \added{$ g(\boldsymbol{r}) $ such that $ g(\boldsymbol{r}) \approx f_{\boldsymbol{\theta}}(\boldsymbol{r}) $. For example, in image processing, $ g(\boldsymbol{r}) $} could be a function mapping pixel coordinates to their respective values.

As mentioned in \citep{yuce2022structured}, the majority of INR architectures can be decomposed into a mapping function $\gamma:\mathbb{R}^D \rightarrow \mathbb{R}^T$ followed by an MLP, with weights $\boldsymbol{W^{(l)}} \in \mathbb{R}^{F_l\times F_{l-1}}$ and activation function  $\rho^{(l)}: \mathbb{R}\rightarrow \mathbb{R}$, applied element-wise at each layer $l = 1,\ldots,L-1$. In other words, if we represent $\boldsymbol{z^{(l)}}$ as the post-activation output of each layer, most INR architectures compute
\begin{align} \label{Network}
\boldsymbol{z^{(0)}} &= \gamma(\boldsymbol{r}), \nonumber\\
\boldsymbol{z^{(l)}} &= \rho^{(l)}(\boldsymbol{W^{(l)}}\boldsymbol{z^{(l-1)}}+\boldsymbol{B^{(l)}}), \quad l= 1,..., L- 1, \\
f_{\boldsymbol{\theta}}(\boldsymbol{r}) &= \boldsymbol{W^{(L)}}\boldsymbol{z^{(L-1)}}+\boldsymbol{B^{(L)}}. \nonumber
\end{align} 
Additionally, corresponding to the $i$'th neuron of the $l$'th layer, we employ the symbols $a^{(l)}_i$ and $z^{(l)}_i$ for the pre-activation and post-activation functions, respectively. The choice of the activation function $ \rho $ is pivotal in INR, as it influences the network's ability to represent signals. Traditional functions, such as ReLU, may not effectively capture high-frequency components. The novel parametric periodic activation function, i.e., STAF, enhances the network's capability to accurately model and reconstruct complex, high-frequency signals.

\subsection{STAF: SINUSOIDAL TRAINABLE ACTIVATION FUNCTION}
We adopt a Fourier–like activation drawn from the broader family of sinusoidal activations used in INRs (see~\cref{fig:activation}):
\begin{equation} \label{STAF}
\rho^*(x) = \sum_{i=1}^\tau C_i \sin(\Omega_i x + \Phi_i),
\end{equation}
where \(C_i\), \(\Omega_i\), and \(\Phi_i\) represent the \textit{amplitude}, \textit{frequency}, and \textit{phase} parameters, respectively. These parameters are learned dynamically during training, enabling the network to adapt its activation function to the specific characteristics of the signal being modeled. The use of a Fourier series is motivated by its ability to represent signals efficiently, capturing essential components with a small number of coefficients. This adaptability allows STAF to provide a compact and flexible representation for complex patterns in various tasks. 

\begin{figure*}[t]
  \centering
  \includegraphics[width=0.7\textwidth]{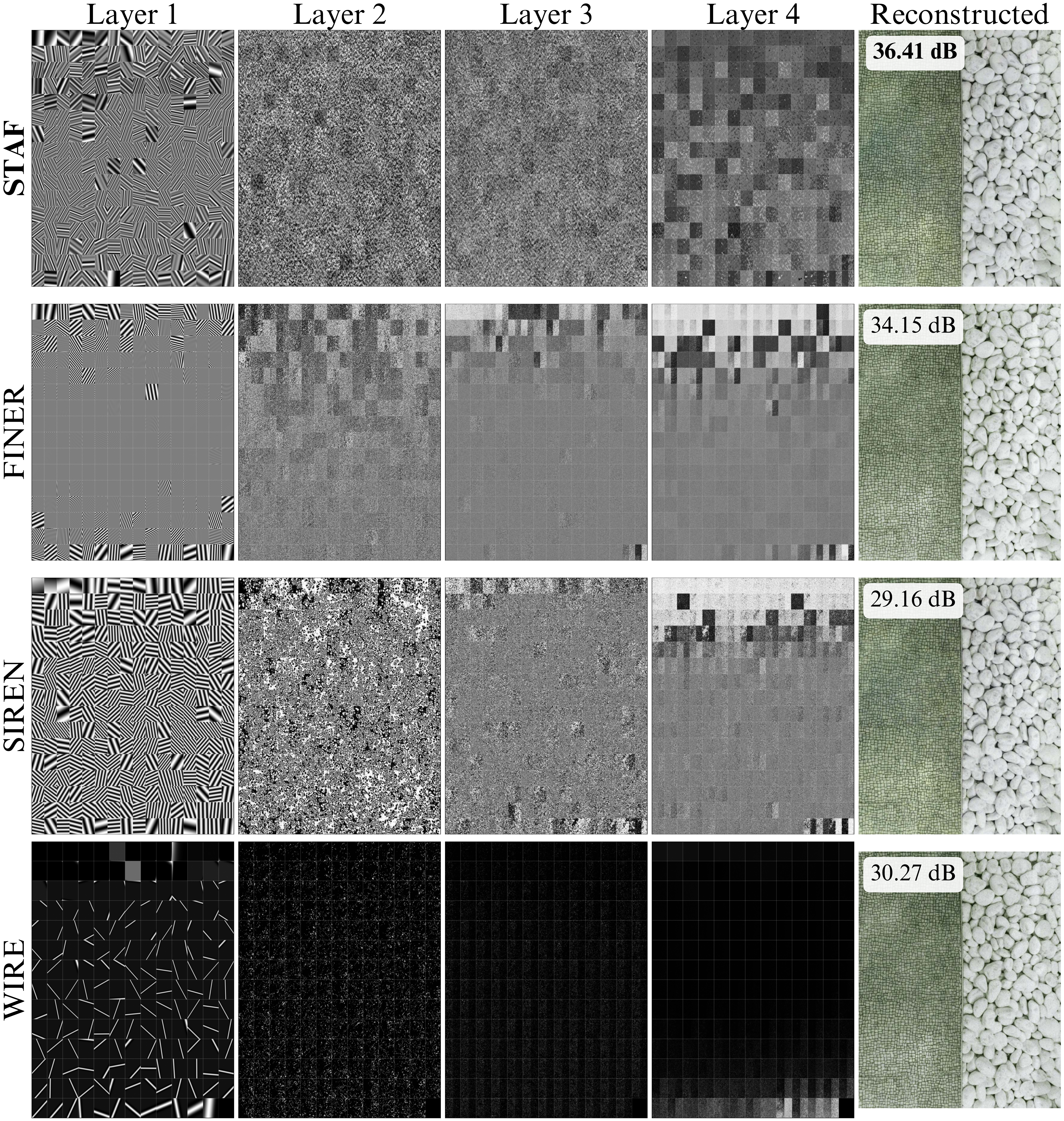}
  \caption{\textbf{Activation maps of STAF, FINER, SIREN, and WIRE learned during image reconstruction.} \textbf{STAF} preserves rich, spatially diverse activations across depth (from fine gratings to mid-level textures) and avoids late-layer blocky/saturated patterns, which correlates with the highest reconstruction quality (PSNR shown). In contrast, \textbf{FINER} becomes progressively smoother and more block-structured, while \textbf{SIREN} and \textbf{WIRE} exhibit periodic or sparse responses and deeper-layer collapse, leading to lower PSNR.}
  \label{fig:map}
\end{figure*}

\subsection{STAF Training Process}
During training, STAF optimizes not only the traditional MLP parameters (weights and biases), but also the coefficients of the activation function. This dual optimization approach ensures that the network learns both an optimal set of transformations (through weights and biases) and an ideal way of activating neurons (through the parametric activation function) for each specific task. The training employs a reconstruction loss function designed to minimize the difference between the target function $g(\boldsymbol{x})$ and the network's approximation $f_{\boldsymbol{\theta}}(\boldsymbol{x})$, while also encouraging efficient representation inspired by Fourier series.
\subsection{Implementation Strategies} 
\label{sec:implement_str}
The implementation of STAF's parametric activation functions can be approached in three ways:

\ding{202} \textbf{Individual Neuron Activation:} This method assigns a unique activation function to each neuron. It offers high expressiveness but leads to a significant increase in the number of trainable parameters, making it impractical for large networks due to potential overfitting and computational inefficiencies.

\ding{203} \textbf{Uniform Network-wide Activation:} Here, a single shared activation function is used across the entire network. This approach simplifies the model by reducing the number of additional parameters but limits the network's expressiveness and adaptability. It may struggle to capture diverse patterns and details in complex signals.

\ding{204} \textbf{Layer-wise Shared Activation:} This balanced strategy employs a distinct shared activation function for each layer, which is used for all experiments in this paper. For example, in a 3-layer MLP with $\tau=25$ terms, only 225 additional parameters are required. This approach optimally balances expressiveness and efficiency, allowing each layer to develop specialized activation dynamics for the features it processes. It aligns with the hierarchical nature of MLPs, where different layers capture different signal abstractions, providing an efficient learning mechanism tailored to each layer's role.


\subsection{Initialization} 
\label{sec:model-initialization}
In this section, we present an initialization strategy tailored for networks utilizing STAF as the activation function. While STAF shares similarities with SIREN~\citep{Siren}, which employs \(\sin\) as its activation function, our initialization scheme is specifically designed to leverage the unique parameterization of STAF. To provide context, we first revisit the key aspects of SIREN's initialization scheme as discussed in~\citep{Siren}, and then highlight how our approach builds upon and extends these principles to enhance network performance and stability.

In SIREN, the input $X$ of a single neuron follows a uniform distribution $U(-1,1)$, and the activation function employed is $\rho(u) = \sin(u)$. Consequently, the output of the neuron is given by $Y = \sin(aX+b)$, where $a,b \in \mathbb{R}$. The authors of \citep{Siren} claim that regardless of the choice of $b$, if $a > \frac{\pi}{2}$, the output $Y$ follows an arcsine distribution, denoted as $Arcsine(-1,1)$. However, it becomes apparent that this claim is not correct upon further examination. If the claim were true, the moments of $Y$ would be independent of $b$.
However, this only occurs when 
$a = n\pi$. We have demonstrated this in the Appendix \ref{assessing_dist}.

In the subsequent parts of \citep{Siren}, the authors assumed that the outputs of the first layer follow a distribution of \textit{arcsine} and fed those outputs into the second layer. By relying on the central limit theorem (CLT), they demonstrated that the output of the second layer, for each neuron, conforms to a normal distribution. Additionally, in Lemma 1.6~\citep{Siren}, they established that if $X \sim \mathcal{N}(0,1)$ and $Y=\sin(\frac{\pi}{2}X)$, then $Y \sim Arcsine(-1,1)$. However, it should be noted that to prove this result, they relied on several approximations. Through induction, they asserted that the inputs of subsequent layers follow an arcsine distribution, while the outputs of these layers exhibit a normal distribution.

In contrast to the approach taken by~\citep{Siren}, the method presented in this study does not depend on the specific distributions of the input vector $\boldsymbol{r}$ and weight matrices $\boldsymbol{W^{(l)}}$. As a result, there is no need to map the inputs to the interval $[-1,1]$. Additionally, this method does not rely on making any approximations or the central limit theorem, which assumes large numbers. Overall, it offers a more rigorous mathematical framework. To pursue this goal, notice the following theorem.

\begin{theorem} \label{initialization_theorem}
Consider a neural network as defined in \eqref{Network} with a sinusoidal trainable activation function (STAF) as in \eqref{STAF}. For convenience, define $\Gamma_i = \Omega_i \, \boldsymbol{w} \cdot \boldsymbol{x}.$ Suppose that for each $i$ we have: $\Phi_i \sim U(-\pi, \pi)$, and the random variables $C_i$ follow the probability density function
    \begin{equation}
        f_{C_i}(c_i) = \frac{\tau |c_i|}{2} e^{-\tfrac{\tau c_i^2}{2}}.
    \end{equation}
In addition, assume the following independence conditions:
\begin{itemize}
    \item the variables $C_i$ are mutually independent;
    \item for each $i$, $C_i$ is independent of $(\Gamma_i, \Phi_i)$; and
    \item the collections $\{(C_i, \Gamma_i, \Phi_i)\}_{i=1}^{\tau}$ are mutually independent.\footnote{Note that this condition does not necessarily imply internal independence within each triplet. In other words, the independence of the triplets does not imply that $C_i$ is independent of $(\Gamma_i, \Phi_i)$.}
\end{itemize}
Then, every post-activation will follow a $\mathcal{N}(0, 1)$ distribution (refer to the proof in Appendix \ref{app:proof1}.)
\end{theorem}

\added{\emph{Intuition for Theorem~3.1.} The result is a moment-matching statement rather than a CLT argument. The uniform random phases $\Phi_i \sim U(-\pi,\pi)$ make each sinusoidal component centered and prevent coherent phase alignment across terms. When one expands moments of \(\sum_i C_i \sin(\Gamma_i + \Phi_i)\), averaging over the random phases removes all nonconstant Fourier modes, so only the constant terms in the even-power expansions remain. Because the amplitude variables \(C_i\) are symmetric, all odd moments vanish. The specific density in Eq.~(3) is then chosen so that the even moments of \(C_i\) satisfy exactly the identities needed for the full sum to match the moment sequence of \(\mathcal{N}(0,1)\). In this sense, the Gaussian output arises from exact phase averaging and exact amplitude calibration, not from an asymptotic averaging approximation.}

This initial setting, where every post-activation follows a standard normal distribution, is beneficial because it prevents the post-activation values from vanishing or exploding. This ensures that the signals passed from layer to layer remain within a manageable range, particularly in the first epoch, which establishes the foundation for subsequent learning~\citep{yuce2022structured}. If the learning process is well-posed and there is sufficient data, the training process is likely to converge to a stable and accurate solution. Therefore, while it is important to monitor potential issues in later epochs, the concern about vanishing or exploding values is significantly greater during the initial stages. Proper initialization helps mitigate these risks early, facilitating smoother and more effective training overall.

\section{Theory: Expressivity, Capacity, and Convergence}
This section develops the theory behind trainable sinusoidal activations for INRs and states the main results used in the paper. (i) We begin with a \emph{Kronecker-equivalence} construction (Theorem~\ref{Kronecker_theorem}) showing how networks with trainable sinusoidal activations can be represented by sine networks with structured (Kronecker) weights; we use this to quantify the growth of potential frequencies via a Delannoy-number bound (Theorem~\ref{asymptotic_behavior}).
(ii) We then connect \emph{capacity} to \emph{convergence} by discussing how these activations reshape the NTK spectrum and what this implies for learning dynamics.

For readability, we present statements, intuition, and their consequences in the main text, while deferring full technical material to the Appendix:
\emph{(A)} the complete proof of Theorem~\ref{Kronecker_theorem} appears in Appendix~\ref{app:proof_Kronecker_theorem}; 
\emph{(B)} the injectivity/perturbation argument used to preserve the size of the potential-frequency set is formalized in Lemma~\ref{rational_elusive} with proof in Appendix~\ref{app:proof_rational_elusive}; 
\emph{(C)} the closed-form dual activation and derivative for STAF, and the resulting NTK recursion, are derived in Appendix~\ref{NTK-analysis} (Theorem~\ref{analytic_NTK}), which also includes empirical NTK computation details and extended eigenvalue/eigenfunction visualizations (Fig.~\ref{fig:ntk_combined}); to support that derivation, the Gaussian integral identities are proved in Appendix~\ref{proof_of_integral_lemma}; 
\emph{(D)} the unit-variance initialization result referenced by our analysis has its proof in Appendix~\ref{app:proof1}. 
We use these ingredients to justify the capacity growth, explain the observed eigen-spectra, and clarify why periodic activations \emph{mitigate} (rather than eliminate) practical manifestations of spectral bias.

Let us first examine the expressive power of our architecture, drawing upon the notable Theorem 1 from \citep{yuce2022structured}. This theorem is as follows:
\begin{theorem} \label{expressive_power}
    (Theorem 1 of \citep{yuce2022structured}) Let \(f_{\boldsymbol{\theta}}: \mathbb{R}^D \rightarrow \mathbb{R}\) be an INR of the form of Equation \eqref{Network} with $\rho^{(l)}(x) = \sum_{j=0}^J \alpha_jx^j$ for $l > 1$. Furthermore, let $\boldsymbol{\Psi} = [\boldsymbol{\Psi}_1, ..., \boldsymbol{\Psi}_T]^{tr} \in \mathbb{R}^{T\times D}$ and $\boldsymbol{\zeta}\in\mathbb{R}^T$ denote the matrix of frequencies and vector of phases, respectively, used to map the input coordinate \(r \in \mathbb{R}^D\) to $\gamma(r) = \sin(\boldsymbol{\Psi} \boldsymbol{r} + \boldsymbol{\zeta})$. This architecture can only represent functions of the form
    \[ 
    f_{\boldsymbol{\theta}}(r) = \sum_{\boldsymbol{w^\prime}\in \mathcal{H}(\Psi)} c_{w^\prime } \sin(\langle \boldsymbol{w^\prime}, \boldsymbol{r}\rangle +	\zeta_{\boldsymbol{w^\prime}}),
    \]
    where
    \[ \mathcal{H}(\boldsymbol{\Psi}) \subseteq  \Tilde{\mathcal{H}}(\boldsymbol{\Psi})=\Bigg\{\sum_{t=1}^{T}s_t\boldsymbol{\Psi}_t \Bigg|~s_t \in \mathbb{Z}\wedge \sum_{t=1}^{T} |s_t|\leq J^{L-1}\Bigg\}.
    \]
\end{theorem}
Please note the following remarks regarding this theorem:

\textbf{Remark 5.1.1.} We refer to $\Tilde{\mathcal{H}}$ as the set of potential frequencies.

\textbf{Remark 5.1.2.} The expression $\sum_{t=1}^{T}s_t\boldsymbol{\Psi}_t$ is equal to $\boldsymbol{\Psi}^{tr}[s_1,...,s_{T}]^{tr}$. This representation is more convenient for our subsequent discussion, as we will be exploring the kernel of $\boldsymbol{\Psi}$ in the sequel.

\textbf{Remark 5.1.3.} In the context of SIREN, where $\rho^{(l)}=\sin$, the post-activation function of the first layer, $z^{(0)}=\sin(\omega_0(\boldsymbol{W}^{(0)}\boldsymbol{r}+\boldsymbol{b}^{(0)}))$, can be interpreted as $\gamma(\boldsymbol{r})=\sin(\boldsymbol{\Psi r}+\boldsymbol{\zeta})$.

We next investigate the expressive power of the proposed activation. To facilitate comparison with SIREN, we express our network using $\sin$ as the activation function.

Let us consider a neural network with a parametric activation function defined in \eqref{STAF}. To represent our network using SIREN, we demonstrate that every post-activation function of our network from the second layer onwards ($z^{l+1}$) can be expressed using linear transformations and sine functions. Notably, the final post-activation function ($z^{(L-1)}$) can be constructed using SIREN, albeit requiring more neurons than STAF. In other words, our network can be described using a SIREN and some Kronecker products denoted by $\otimes$. This analysis resembles that provided in \citep{jagtap2022deep}, with a slight difference in the settings of the paper. In \citep{jagtap2022deep}, it was shown that an adaptive activation function of the form $\rho^*(x) = \sum_{i=1}^{\tau}C_i \rho_{i}(\Omega_ix)$ can be represented using a feed-forward neural network, where each layer has neurons with activation functions $\rho_{i}$. To align STAF with this theorem, we must have $\rho_i=\sin(\Omega_ix+\Phi_i)$. However, here we aim to represent STAF using an architecture that only employs sine activation functions (SIREN). For this purpose, we introduce the following theorem, which holds true for every parametric activation function:
\begin{theorem} \label{Kronecker_theorem}
Let $L\geq 2$ and $1\leq l \leq L$. Consider a neural network as defined in \eqref{Network} with $L$ layers. In addition, let $\boldsymbol{\Omega}=[\Omega_1,...,\Omega_\tau]^{tr}$, $\boldsymbol{\Phi}=[\Phi_1,...,\Phi_\tau]^{tr}$, and $\boldsymbol{C}=[C_1,\ldots,C_\tau]^{tr}$. If the trainable activation function is $\rho^*(x)=\sum_{m=1}^{\tau}\boldsymbol{C}_m\rho(\boldsymbol{\Omega}_mx+\boldsymbol{\Phi}_m)$, then an equivalent neural network with activation function $\rho(x)$ and $L+1$ layers can be constructed as follows (parameters of the equivalent network are denoted with an overline):
    \begin{align}
        \overline{\boldsymbol{z}^{(0)}} &= \gamma(\boldsymbol{r}), \nonumber\\
        \overline{\boldsymbol{z}^{(l)}} &= \rho\left(\overline{\boldsymbol{W}^{(l)}}~\overline{\boldsymbol{z}^{(l-1)}}+\overline{\boldsymbol{B}^{(l)}}\right), \quad l= 1,..., L, \\
        \overline{f}_{\overline{\theta}}(\boldsymbol{r}) &= \overline{\boldsymbol{W}^{(L+1)}}~\overline{\boldsymbol{z}^{(L)}}; \nonumber
    \end{align}
where
    \begin{equation} \label{weights_and_biases}
    \scalebox{0.78}{$
    \overline{\boldsymbol{W}^{(l)}} =
    \begin{cases}
        \boldsymbol{\Omega} \otimes \boldsymbol{W}^{(l)}, & \text{if $l = 1$}, \\[4pt]
        \left( \boldsymbol{\Omega} \otimes \boldsymbol{C}^{tr} \right) \otimes \boldsymbol{W}^{(l)}, & \text{if $l$ is even}, \\[4pt]
        \left( \boldsymbol{\Omega} \otimes \boldsymbol{W}^{(l)} \right) \left( \boldsymbol{C}^{tr} \otimes \boldsymbol{I}_{F_{l-1}} \right), & \text{if $l$ is odd, $l > 1$, and $l \neq L+1$}, \\[4pt]
        \boldsymbol{C}^{tr} \otimes \boldsymbol{I}_{F_{l-1}}, & \text{if $l$ is odd, $l > 1$, and $l = L+1$}.
    \end{cases}$}
\end{equation}
and
\begin{equation}
    \overline{\boldsymbol{B}^{(l)}} = \boldsymbol{\Phi} \otimes \boldsymbol{J}_{F_l}.
\end{equation}

in which $\boldsymbol{J}_{F_l}$ is an all-ones $F_l\times 1$ vector. Furthermore, if $L$ is even, then $\overline{f}_{\overline{\theta}}(\boldsymbol{r})=f_\theta(\boldsymbol{r})$ (we call these networks \lq Kronecker equivalent' in this sense).
\end{theorem}
The proof of this theorem is provided in the Appendix~\ref{app:proof_Kronecker_theorem}. As we observed, although a network with the activation function $\rho^*$ can be represented using the activation function $\rho$, it features a unique architecture. These networks are not merely typical MLPs with the activation function $\rho$, as the weights in the Kronecker equivalent network exhibit dependencies due to the Kronecker product.

It is desirable that Theorem \eqref{Kronecker_theorem} does not depend on the parity of $L$. To achieve this, consider the following remark:

\textbf{Remark 5.2.} \label{dummy_layer} We can introduce a dummy layer with the activation function $\rho^*$. Specifically, we define $\boldsymbol{z}^{(L)}=\rho^*\left(f_\theta(\boldsymbol{r})\right)$, and $\Tilde{f_\theta}(\boldsymbol{r}) = \boldsymbol{W}^{(L+1)}\boldsymbol{z}^{(L)}+\boldsymbol{B}^{(L+1)}$, where $\boldsymbol{W}^{(L+1)}=\boldsymbol{O}$. To ensure that $\Tilde{f_\theta}(\boldsymbol{r}) =f_\theta(\boldsymbol{r})$, we set $\boldsymbol{B}^{(L+1)}=f_\theta(\boldsymbol{r})$. This approach allows us to construct an equivalent neural network with one more layer.

As a result of Remark \hyperref[dummy_layer]{5.2}, the equivalent network of a network with a trainable activation function, has either one more layer, or the same number of layers.
As an immediate result of Theorem \eqref{Kronecker_theorem}, if we denote the embedding of the first layer of the SIREN equivalent of our network by $\boldsymbol{\overline{\Psi}}$, then
\begin{equation}
    \boldsymbol{\overline{\Psi}}=\overline{\boldsymbol{W}^{(1)}}=\boldsymbol{\Omega}\otimes\boldsymbol{W}^{(1)}\in\mathbb{R}^{\tau F_1\times F_0}
\end{equation}
which is $\tau$ times bigger than the embedding of the first layer of a SIREN with $\boldsymbol{W}^{(1)}\in \mathbb{R}^{F_1\times F_0}$. To understand the impact of this increase on expressive power, it suffices to substitute $T$ with $\tau T$ in Theorem \eqref{expressive_power}.  The next theorem will reveal how this change will affect the cardinality of the set of potential frequencies.

\begin{theorem} \label{asymptotic_behavior}
    (Page 4 of \citep{kiselman2012asymptotic}) Let $V(T,K)=\big\{(s_1,s_2,\ldots,s_{T}) \in \mathbb{Z}^T \big|~ \sum_{t=1}^{T} |s_t|\leq K\big\}$.\footnote{~We use \( V \) to denote these points as cells in a \( T \)-dimensional von Neumann neighborhood of \( K \), clarifying that \( V \) does not represent a vector space.} Then we have
    \begin{equation}
        |V(T,K)|=\sum_{i=0}^{min(K,T)}\binom{i}{K}\binom{i}{T}2^i
    \end{equation}
    This number is called the Delannoy number. Moreover, for fixed $K$,
    \begin{equation}
        |V(T,K)|\sim A_K(2T)^K,\quad T\rightarrow +\infty.
    \end{equation}
\end{theorem}

\added{\emph{Practical implication of Kronecker-equivalence}. Theorem~\ref{Kronecker_theorem} and Theorem~\ref{asymptotic_behavior} suggest that adding a small trainable multi-sinusoid activation can expand the set of potential frequencies much more efficiently than a fixed single-sine activation, especially in deep networks where expressive bottlenecks are frequency-related. For practitioners, this means STAF is most attractive when the target contains mixed frequencies or repetitive fine detail and the baseline struggles to recover these components without positional encodings. The result also suggests that increasing $\tau$ should be viewed as a direct capacity knob: larger $\tau$ enlarges the effective frequency dictionary, but with increased compute and diminishing practical returns.}

As an immediate result of this theorem, for large values of $T$, we have $\frac{|V(\tau T,K)|}{|V(T,K)|} \sim \tau^K$. (See Appendix \ref{Asymptotic_proof} for an alternative proof.) Now, it is time to analyze the cardinality of the set of potential frequencies:
\begin{equation}
	\Tilde{\mathcal{H}}(\boldsymbol{\Psi})=\Bigg\{\sum_{t=1}^{T}s_t\boldsymbol{\Psi}_t \Bigg|~(s_1,s_2,\ldots,s_T) \in V(T,J^{L-1})\Bigg\}
\end{equation}
or equivalently,
\begin{equation}
	\Tilde{\mathcal{H}}(\boldsymbol{\Psi})=\Bigg\{\boldsymbol{\Psi}^{tr}[s_1,...,s_T]^{tr} \Bigg|~s_t \in \mathbb{Z}\wedge \sum_{t=1}^{T} |s_t|\leq J^{L-1}\Bigg\}
\end{equation}

The cardinality of the set $\Tilde{\mathcal{H}}(\boldsymbol{\Psi})$ is bounded above by $V(T,J^{L-1})$. If $\boldsymbol{\Psi}^{tr}$, is injective on the integer lattice $\mathbb{Z}^T$, then $|\Tilde{\mathcal{H}}(\boldsymbol{\Psi})| = |V(T,J^{L-1})|$. However, in general, analyzing how a linear transformation affects the size of a convex body can be approached using the geometry of numbers \citep{matousek2013lectures} or additive geometry \citep{tao2006additive}. To simplify the analysis and preserve the size of $\Tilde{\mathcal{H}}(\boldsymbol{\Psi})$ as large as possible, we can slightly perturb the matrix $\boldsymbol{\Psi}^{tr}$ such that its kernel contains no points with rational coordinates, except the origin. This is a much stronger condition than having no integer lattice points in the kernel. To address this, we introduce a lemma. It's worth noting that we can assume the matrices are stored with rational entries, as they are typically represented in computers using floating-point numbers. In our subsequent analysis, however, assuming rational entries for just one column of the matrix $\boldsymbol{\Psi}$ is sufficient.
\begin{lemma} \label{rational_elusive}
    Let $\boldsymbol{A}\in\mathbb{R}^{D\times T}$, and for one of its rows, like $r$'th row, we have $\boldsymbol{A}_r\in\mathbb{Q}^T$. Then, in every neighborhood of $\boldsymbol{A}$, there is a matrix $\boldsymbol{\hat{A}}$ such that $Ker(\boldsymbol{\hat{A}})\cap \mathbb{Q}^{T}=\boldsymbol{O}$.
\end{lemma}
(The proof is provided in the Appendix~\ref{app:proof_rational_elusive}.) Consider Lemma \eqref{rational_elusive}, where we let $\boldsymbol{A}=\boldsymbol{\Psi}^{tr}$. Thus, for every neighborhood of $\boldsymbol{\Psi}^{tr}$, there exists a matrix $\boldsymbol{\hat{\Psi}^{tr}}$ such that $Ker(\boldsymbol{\hat{\Psi}^{tr}})\cap \mathbb{Q}^{T}=\boldsymbol{O}$; in other words, $\boldsymbol{\hat{\Psi}^{tr}}$ is injective over rational points, and consequently over integer lattice points. This guarantees that $|\Tilde{\mathcal{H}}(\boldsymbol{\hat{\Psi}})| = |V(T,J^{L-1})|$. Therefore, this section demonstrated that, in comparison to SIREN, STAF can substantially increase the size of the set of potential frequencies by a factor of $\tau^K$, highlighting how the Kronecker product boosts the activation function’s expressiveness.

\section{Neural Tangent Kernel}
\label{NTK-analysis}
The Neural Tangent Kernel (NTK) is a significant concept in the theoretical understanding of neural networks, particularly in the context of their training dynamics \citep{jacot2018neural}. To be self-contained, we provide an explanation of the NTK and its background in kernel methods. We believe this will be beneficial for readers, as previous papers on implicit neural representation using the NTK concept have not adequately explained the NTK or the significance of its eigenvalues and eigenfunctions.


A kernel is a function \( K(\mathbf{x}, \mathbf{\tilde{x}}) \) used in integral transforms to define an operator that maps a function \( f \) to another function \( T_f \) through the integral equation
\[
T_f(\mathbf{x}) = \int K(\mathbf{x}, \mathbf{\tilde{x}}) f(\mathbf{\tilde{x}}) \, d\mathbf{\tilde{x}}.
\]
Since \( T_f \) is a linear operator with respect to \( f \), we can discuss its eigenvalues and eigenfunctions. The eigenvalues and eigenfunctions of a kernel are the scalar values \(\lambda\) and the corresponding functions \(\zeta(\mathbf{x})\) that satisfy the following equation \citep{ghojogh2021reproducing}
\[
\int K(\mathbf{x}, \mathbf{\tilde{x}}) \zeta(\mathbf{\tilde{x}}) \, d\mathbf{\tilde{x}} = \lambda \zeta(\mathbf{x}).
\]

In the context of neural networks, the concept of a kernel becomes particularly remarkable when analyzing the network's behavior in the infinite-width limit. Kernels in machine learning, such as the Radial Basis Function (RBF) kernel or polynomial kernel, are used to measure similarity between data points in a high-dimensional feature space. These kernels allow the application of linear methods to non-linear problems by implicitly mapping the input data into a higher-dimensional space \citep{braun2005spectral}.

The NTK extends this idea by considering the evolution of a neural network's outputs during training. When a neural network is infinitely wide, its behavior can be closely approximated by a kernel method. In this case, the kernel in question is the NTK, which emerges from the first-order Taylor series approximation (or tangent plane approximation) of the network's outputs.

Formally, for a neural network \( f(\mathbf{x}; \boldsymbol{\theta}) \) with input \( \mathbf{x} \) and parameters \( \boldsymbol{\theta} \), the NTK, denoted as \( K^{(L)}(\mathbf{x}, \mathbf{\Tilde{x}}) \), is defined as:
\[
K^{(L)}(\mathbf{x}, \mathbf{\Tilde{x}}) = \langle \nabla_{\boldsymbol{\theta}} f(\mathbf{x}; \boldsymbol{\theta}),\nabla_{\boldsymbol{\theta}} f(\mathbf{\Tilde{x}}; \boldsymbol{\theta})\rangle,
\]
where \( \nabla_{\boldsymbol{\theta}} f(\mathbf{x}; \boldsymbol{\theta}) \) represents the gradient of the network output with respect to its parameters.

There are two methods for calculating the NTK: the analytic approach and the empirical approach \citep{novak2019neural,chen2022neural}. In the paper, we derived the analytic NTK of a neural network that uses our activation function in~\autoref{sec:ntk}. However, for our experimental purposes, we utilized the empirical NTK. It is worth noting that calculating the NTK for real-world networks is highly challenging, and typically not computationally possible \citep{mohamadi2023fast}.

Similarly to NTK computation, there are analytic and empirical methods to calculate the eigenvalues and eigenfunctions of a kernel \citep{williams2000effect}. These values play a crucial role in characterizing neural network training. For instance, it has been shown that the eigenvalues of the NTK determine the convergence rate \citep{wang2022and,bai2023physics}. Specifically, components of the target function associated with kernel eigenvectors having larger eigenvalues are learned faster \citep{wang2022and,tancik2020fourier}. In fully-connected networks, the eigenvectors corresponding to higher eigenvalues of the NTK matrix generally represent lower frequency components \citep{wang2022and}. Furthermore, the eigenfunctions of an NTK can illustrate how effectively a model learns a signal dictionary \citep{yuce2022structured}.

Figure \ref{fig:ntk_eigenfunctions} illustrates the eigenfunctions of various NTKs using different activation functions. As shown, the STAF activation function results in finer eigenfunctions, which intuitively enhance the ability to learn and reconstruct higher frequency components. Additionally, Figure \ref{fig:ntk_eigenvalues} presents the eigenvalues of different NTKs with various activation functions. The results indicate that STAF produces higher eigenvalues, leading to a faster convergence rate during training. Moreover, STAF also generates a greater number of eigenvalues, compared to ReLU and SIREN. Having more eigenvalues is beneficial because it suggests a richer and more expressive kernel, capable of capturing a wider range of features and details in the data. \added{We emphasize that~\Cref{fig:ntk_combined} presents the empirical NTK eigenfunctions for the image-fitting setting. To support the broader relevance of this phenomenon, we include an analogous empirical NTK analysis on audio in Appendix~\ref{sec:appendix_audio_ntk}, which shows the same overall qualitative pattern.}
\begin{figure}[t]
    \centering
    \subfigure[]
        {\includegraphics[width=0.41\textwidth]{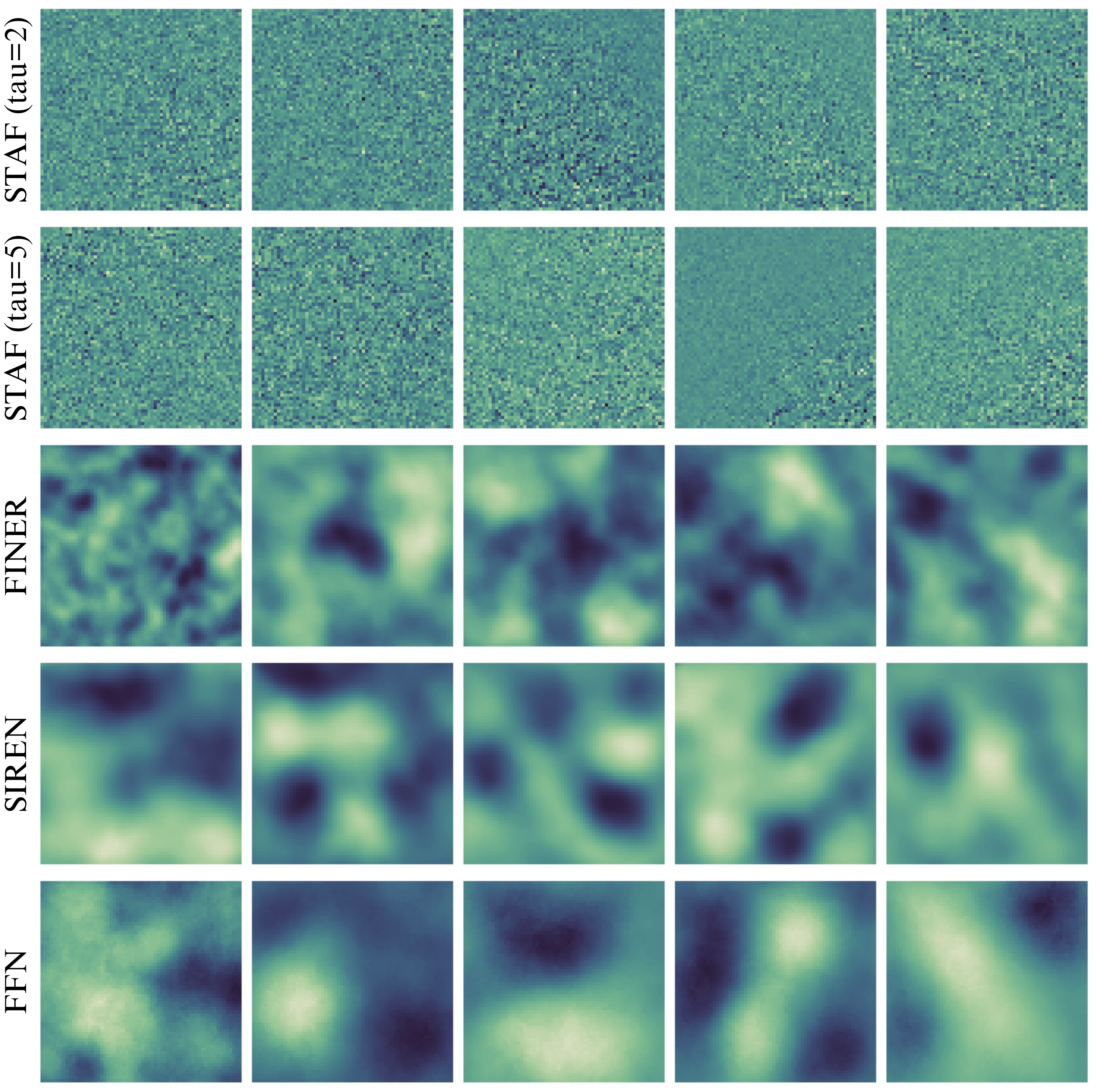}\label{fig:ntk_eigenfunctions}}
    \hfill
    \subfigure[]
        {\includegraphics[width=0.54\textwidth]{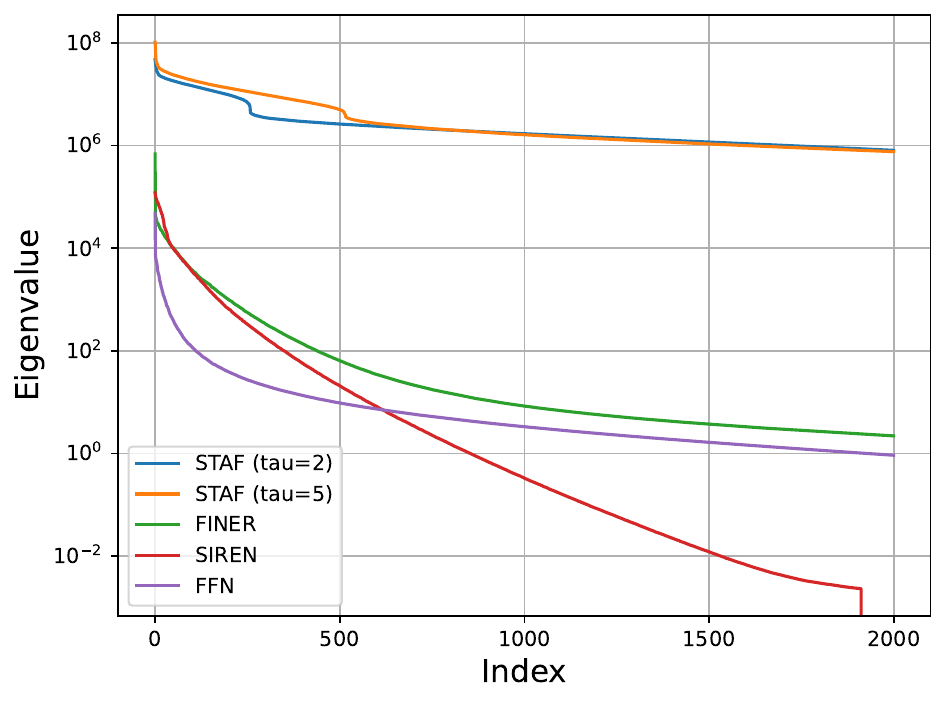}\label{fig:ntk_eigenvalues}}
\caption{\textbf{Empirical NTK analysis.}
(a) The first five eigenfunctions of the empirical NTK for \textbf{STAF} ($\tau\!=\!2,5$), \textbf{FINER}, \textbf{SIREN}, and \textbf{FFN}. 
(b) The corresponding NTK eigenvalue spectra.
\textbf{STAF} exhibits highly consistent eigenfunctions and spectra across $\tau$, indicating that competitive performance can be retained with smaller $\tau$ (i.e., fewer parameters), consistent with Table~\ref{tab:tau_ablation}. 
Although \textbf{SIREN} is closely related to \textbf{STAF} with $\tau\!=\!1$ in form, the different initialization and the use of trainable frequencies/phases lead to markedly different NTK characteristics and a substantial performance gap.}
    \label{fig:ntk_combined}
\end{figure}

\added{\emph{Practical implication of the NTK analysis.} The NTK discussion should be interpreted as guidance about optimization behavior, not just about asymptotic theory. In our setting, trainable sinusoidal parameters change the kernel spectrum and can increase the prominence of components that are otherwise learned slowly. Practically, this suggests that STAF is useful when early recovery of fine structure matters, but it does not imply elimination of spectral bias. The empirical NTK spectra also indicate that competitive behavior can often be retained with smaller $\tau$, which supports the use of small $\tau$ together with layer-wise sharing as a strong default configuration.}

\section{Experimental Results} \label{Experimental_Results}
\begin{figure*}[t]
    \centering    
    \includegraphics[width=0.95\textwidth]{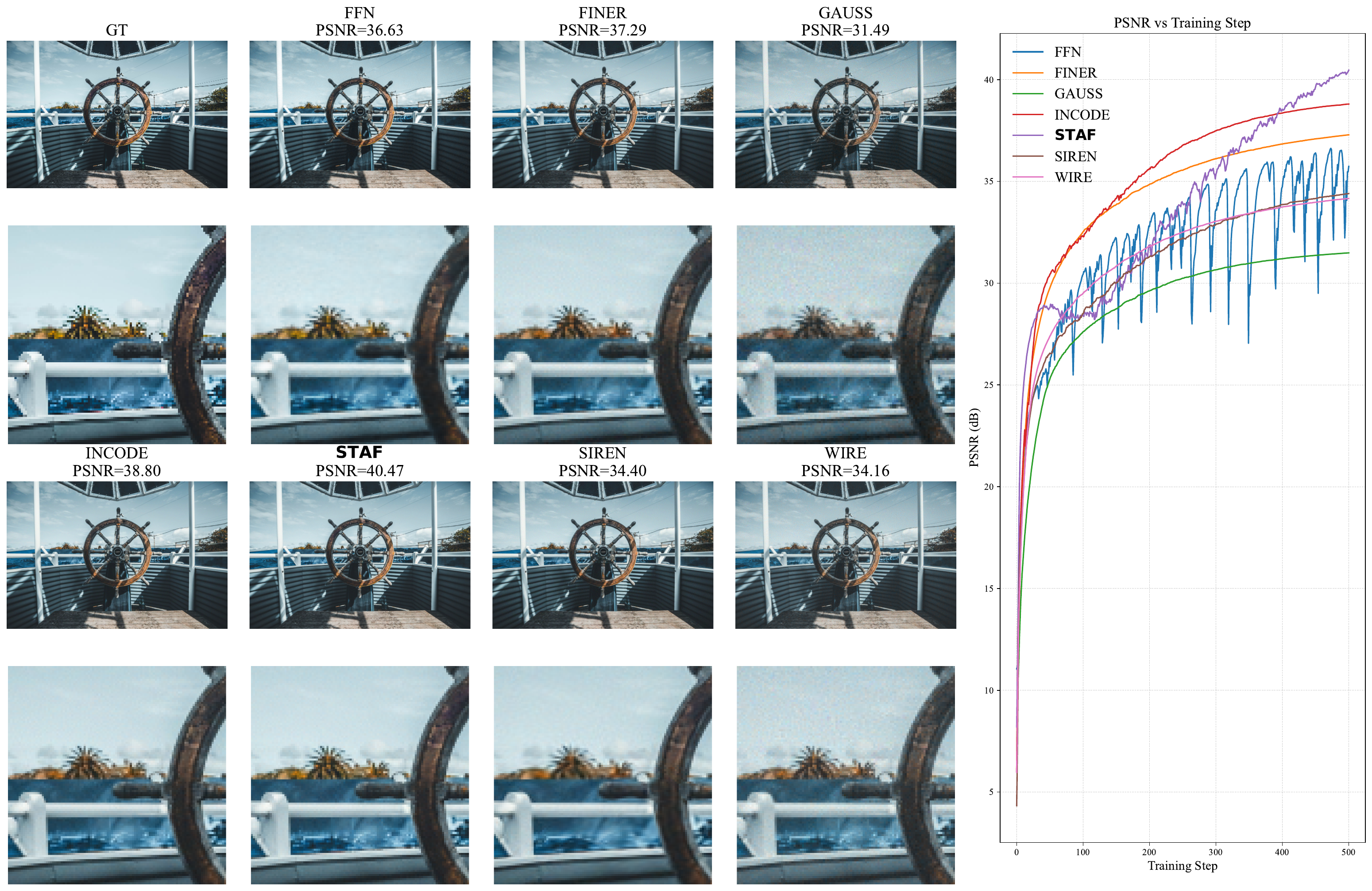}\\
    \caption{\textbf{Image representation quality and convergence.}
    \textit{Left:} qualitative reconstructions using \textbf{STAF} and competing activation functions.
    \textit{Right:} PSNR trajectories over 500 optimization iterations.
    All methods fit the same DIV2K image~\citep{div2k}, downsampled by $4\times$ to $510\times339$; \added{\textbf{STAF} exhibits more sustained improvement during training and achieves the highest final PSNR, while several baselines plateau earlier but at lower final quality.}}
    \label{fig:celtic}
\end{figure*}
We evaluated SOTA models for image, audio, and shape representations, inverse problems such as super-resolution and image denoising, and NeRFs. Specifically, we used an MLP architecture with 3 hidden layers and 256 hidden nodes. The models tested included INCODE, FINER, WIRE, Gauss, FFN, SIREN, ReLU with positional encoding (PEMLP), and MFN~\citep{kazerouni2024incode, liu2024finer, saragadam2023wire, ramasinghe2022beyond, Siren, tancik2020fourier, fathony2020multiplicative}. All experiments used NVIDIA RTX 3090 or A40 GPUs. Our implementation builds on SIREN, WIRE, and INCODE codebases. Learning rates followed optimal settings from original works, and all models were trained with Adam for consistency. STAF was initialized as described in~\cref{sec:model-initialization}, while other models followed their original initialization strategies. We used \(\tau = 5\) for all tasks, except image denoising (\(\tau = 2\)). We also included more ablation studies and implementation details in \textit{Appendix}.
\subsection{Signal Representations}
\noindent\textbf{Image:} We evaluated image representation across multiple datasets and resolutions: 
\texttt{DIV2K}~\citep{div2k} ($\sim510 \times 340$, \cref{fig:celtic,fig:image_rep,tab:div2k_results}), 
\texttt{KODAK}~\citep{kodak_dataset,KodakSuiteOriginal} (24 images, $256 \times 256$, \cref{sec:kodak_and_celeba}), 
\texttt{CelebA}~\citep{liu2015faceattributes} (19{,}867 images, $128 \times 128$, \cref{sec:kodak_and_celeba}), 
and the high-resolution \texttt{Tokyo} panorama~\citep{Dobson2018ShibuyaPanorama} ($6144 \times 2324$, \cref{sec:tokyo}). 

\cref {fig:celtic} (left) shows that STAF produces sharper edges and more faithful texture recovery than the baselines, surpassing the second-best method, INCODE, by +1.67 dB. The right panel quantifies reconstruction quality through PSNR over 500 training steps, showing STAF's effectiveness in addressing the capacity-convergence gap in INR models.  \cref{fig:map} shows activation maps learned during the image reconstruction task. STAF produces more detailed and higher-quality reconstructions compared to SIREN and WIRE, highlighting its ability to capture complex features more effectively.

\noindent\textbf{Shape:} The quantitative and qualitative results of the shape representation are shown in~\cref{tab:sdf} and~\cref{fig:sdf}. Using the Stanford 3D Scanning Repository~\citep{Stanford3DScanRepo} and following the INODE strategy~\citep{kazerouni2024incode}, we generated an occupancy volume by sampling points on a grid \(512 \times 512 \times 512\), assigning 1 to voxels inside the object and 0 outside. The results demonstrate STAF's capability to effectively capture both fine and coarse 3D shape details, achieving higher Intersection over Union (IoU) and lower Chamfer distance (CD) than other methods.

\begin{table}[t]
    \centering
    \begin{minipage}[t]{0.49\textwidth}
        \centering
    \caption{\textbf{3D shape reconstruction (SDF).} IoU~($\uparrow$) and Chamfer Distance (CD)~($\downarrow$) on four shapes, with \textbf{Avg.} over all shapes. \colorbox{purple}{\kern-\fboxsep best\kern-\fboxsep} and \colorbox{light}{\kern-\fboxsep second-best\kern-\fboxsep}. \textbf{STAF} achieves the highest average IoU and the lowest average CD, indicating consistently more accurate surfaces across diverse geometries.}        
        \label{tab:sdf}
        \renewcommand{\arraystretch}{1.2}
        \resizebox{\textwidth}{!}{%
            \begin{tabular}{l|cc|cc|cc|cc|cc}
                \toprule
                \textbf{Method} 
                & \multicolumn{2}{c|}{\textbf{Armadillo}} 
                & \multicolumn{2}{c|}{\textbf{Dragon}} 
                & \multicolumn{2}{c|}{\textbf{Lucy}} 
                & \multicolumn{2}{c|}{\textbf{Thai}} 
                & \multicolumn{2}{c}{\textbf{Avg.}} \\
                \cmidrule(lr){2-3} \cmidrule(lr){4-5} \cmidrule(lr){6-7} \cmidrule(lr){8-9} \cmidrule(lr){10-11}
                & IoU & CD & IoU & CD & IoU & CD & IoU & CD & IoU & CD \\
                \midrule
                PEMLP & 0.9958 & 3.70e-7 & 0.9966 & 2.73e-7 & 0.9920 & 1.80e-6 & 0.9911 & 2.05e-6 & 0.9939 & 1.12e-6 \\
                SIREN       & 0.9962 & 3.62e-7 & \cellcolor{light}0.9971 & \cellcolor{light}2.60e-7 & 0.9892 & 2.19e-6 & \cellcolor{light}0.9929 & \cellcolor{light}9.59e-7 & 0.9939 & 9.43e-7 \\
                WIRE        & 0.9721 & 6.54e-6 & 0.9749 & 4.46e-6 & 0.9554 & 2.06e-5 & 0.9507 & 1.55e-5 & 0.9633 & 1.18e-5 \\
                FINER       & \cellcolor{light}0.9965 & 3.57e-7 & 0.9958 & 3.06e-7 & \cellcolor{light}0.9962 & \cellcolor{light}1.49e-6 & 0.9923 & 1.15e-6 & \cellcolor{light}0.9952 & \cellcolor{light}8.24e-7 \\
                INCODE      & 0.9964 & \cellcolor{light}3.54e-7 & 0.9969 & 2.65e-7& 0.9946 & 1.60e-6 & 0.9924 & 1.42e-6 & 0.9951 & 9.10e-7 \\
                \midrule
                \textbf{STAF} & \cellcolor{purple}0.9972 & \cellcolor{purple}{3.51e-7} & \cellcolor{purple}{0.9973} & \cellcolor{purple}{2.55e-7} & \cellcolor{purple}{0.9971} & \cellcolor{purple}{1.39e-6} & \cellcolor{purple}{0.9935} & \cellcolor{purple}{9.20e-7} & \cellcolor{purple}{0.9963} & \cellcolor{purple}7.29e-7 \\
                \bottomrule
            \end{tabular}%
        }
    \end{minipage}
    \hfill
    \begin{minipage}[t]{0.49\textwidth} 
        \centering
\caption{\textbf{NeRF view synthesis.} PSNR~($\uparrow$) and LPIPS~($\downarrow$) across five objects. \colorbox{purple}{\kern-\fboxsep best\kern-\fboxsep} and \colorbox{light}{\kern-\fboxsep second-best\kern-\fboxsep}. \textbf{STAF} attains the best PSNR on most objects while substantially reducing LPIPS, showing improved fidelity \emph{and} perceptual quality compared to prior activations.}        
        \label{tab:nerf_results}
        \renewcommand{\arraystretch}{1.2}
        \resizebox{\textwidth}{!}{%
            \begin{tabular}{l|cc|cc|cc|cc|cc}
                \toprule
                \textbf{Method} 
                & \multicolumn{2}{c|}{\textbf{LEGO}} 
                & \multicolumn{2}{c|}{\textbf{Drums}} 
                & \multicolumn{2}{c|}{\textbf{Chair}} 
                & \multicolumn{2}{c|}{\textbf{Hotdog}} 
                & \multicolumn{2}{c}{\textbf{Ship}} \\
                \cmidrule(lr){2-3} \cmidrule(lr){4-5} \cmidrule(lr){6-7} \cmidrule(lr){8-9} \cmidrule(lr){10-11}
                & PSNR & LPIPS & PSNR & LPIPS & PSNR & LPIPS & PSNR & LPIPS & PSNR & LPIPS \\
                \midrule
                PEMLP & 25.96 & \cellcolor{light}0.127 & 22.26 & \cellcolor{purple}{0.147} & 28.49 & \cellcolor{purple}{0.084} & 31.96 & \cellcolor{purple}{0.053} & 25.65 & \cellcolor{light}0.217 \\
                Gauss       & 25.15 & 0.143 & 22.05 & 0.167 & 28.87 & \cellcolor{light}0.087 & 32.39 & \cellcolor{light}0.056 & 25.07 & 0.222 \\
                SIREN       & 26.27 & 0.159 & 22.94 & 0.168 & 29.71 & \cellcolor{light}0.087 & 32.85 & 0.058 & 26.00 & 0.220 \\
                WIRE        & 25.31 & 0.150 & 21.89 & 0.165 & 28.63 & 0.088 & 32.14 & 0.061 & 25.77 & 0.225 \\
                FINER       & \cellcolor{light}26.62 & 0.152 & \cellcolor{light}23.21 & 0.175 & \cellcolor{light}29.93 & \cellcolor{light}0.087 & \cellcolor{purple}{33.64} & 0.058 & \cellcolor{light}26.20 & 0.229 \\
                \midrule
                \textbf{STAF} & \cellcolor{purple}{26.74} & \cellcolor{purple}{0.107} & \cellcolor{purple}{23.24} & \cellcolor{light}0.156 & \cellcolor{purple}{30.17} & \cellcolor{purple}{0.084} & \cellcolor{light}33.29 & 0.058 & \cellcolor{purple}{26.28} & \cellcolor{purple}{0.206} \\
                \bottomrule
            \end{tabular}%
        }
    \end{minipage}
\end{table}

\noindent\textbf{Audio:} For the audio task, we used a 7-second clip from Bach’s Cello Suite No. 1: Prelude~\citep{Siren}, sampled at 44,100 Hz. \cref{fig:audio} illustrates the waveforms and reconstruction errors, where STAF demonstrates the highest PSNR, the lowest reconstruction error, and superior fidelity.

\subsection{Inverse Problems}
The results in~\cref{fig:4x_super_res,fig:img_denoising} show that STAF \added{is particularly effective} in both super-resolution and denoising. As interpolants, INRs carry inherent biases that can be exploited for inverse problems such as super-resolution. In $4\times$ super-resolution, STAF achieves the best performance (30.54 dB PSNR, 0.89 SSIM), surpassing INCODE (29.88 dB) and FFN (29.41 dB). While SIREN and FINER recover textures reasonably well, they fail to capture high-frequency details, and Gauss suffers from heavy blurring. For denoising, we simulated severe low-light conditions by adding Poisson-distributed photon noise (mean count = 10), producing highly corrupted images. STAF again achieves the highest PSNR (24.19 dB), effectively suppressing noise and retaining details, whereas FINER and FFN show color artifacts. These results demonstrate STAF’s effectiveness in high-resolution recovery and inverse problem settings. \added{It should be noted that while we use $\tau = 5$ for most tasks as a practical quality–efficiency trade-off, we found that denoising benefits from a smaller value, $\tau = 2$. Our intuition is that denoising differs from reconstruction tasks because part of the high-frequency content is nuisance noise rather than signal. While larger $\tau$ increases the expressiveness of the sinusoidal basis and is helpful for recovering fine structure, it can also make the model more likely to fit noise. In this setting, a smaller $\tau$ acts as a stronger implicit regularizer, encouraging smoother reconstructions and improving denoising quality.}

\begin{figure}[t]
    \centering
    \includegraphics[width=\textwidth]{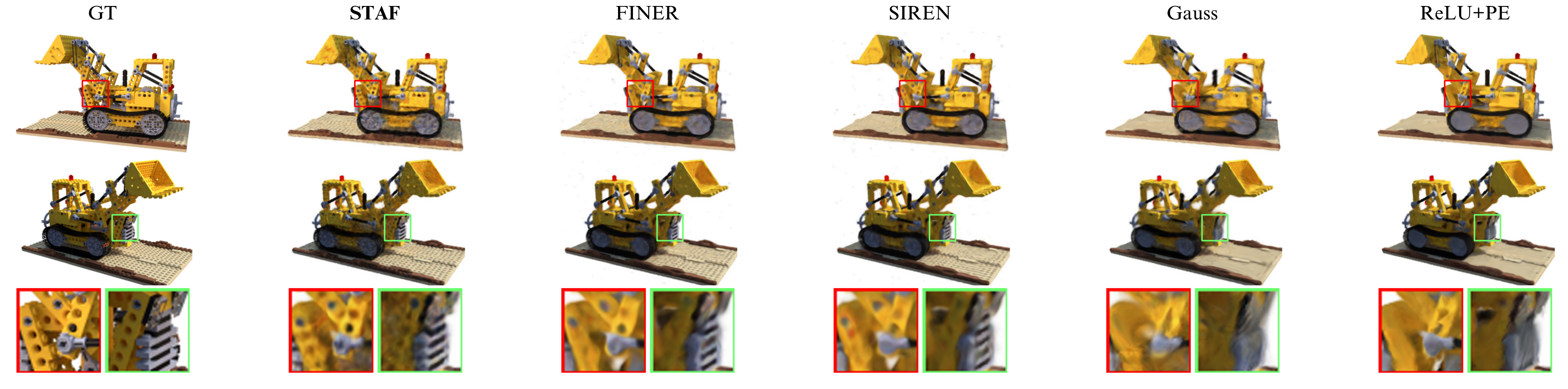} 
\caption{\textbf{NeRF novel view synthesis comparison.} Predicted novel views using \textbf{STAF} and competing activations, alongside the ground truth. Insets (colored boxes) zoom into fine geometric and texture details; \textbf{STAF} better preserves sharp edges and reduces blurring/ghosting artifacts in the reconstructed views.}    
\vspace{-1em}
    \label{fig:nerf_visual}
\end{figure}

\subsection{Neural Radiance Fields (NeRFs)}
NeRFs~\citep{mildenhall2020nerf} utilize INRs and volumetric rendering, where MLPs with ReLU activations and positional encoding are trained to model scenes for novel view synthesis. These models learn a continuous function over 3D space $(x, y, z)$ and view directions $(\theta, \phi)$ to predict color and density at each location. This setup enables the reconstruction of new views by simulating how rays travel through the scene from different camera viewpoints. We evaluate the use of STAF within the NeRF framework, deliberately omitting positional encodings. Our results, shown in~\cref{tab:nerf_results} and~\cref{fig:nerf_visual}, indicate that \added{STAF is competitive and often best in PSNR across the evaluated scenes}. Training details are provided in~\cref{sec:nerf}.

\added{Across the evaluated tasks, STAF helps most when the target signal contains mixed frequencies or repetitive fine detail, particularly in settings without positional encodings. Its advantages are smaller when the baseline already captures the dominant structure well or when the task is less constrained by frequency capacity. We also do not claim that STAF dominates every baseline on every metric; rather, it provides a favorable trade-off between fidelity, optimization behavior, and parameter efficiency in a broad but bounded range of INR settings.}
\section{Conclusion}
We presented a unified view of trainable sinusoidal activations for INRs and instantiated it with STAF, a Fourier-series activation whose parameters are learned. Our theory explains increased frequency capacity and improved optimization via NTK analysis and a Kronecker-equivalence construction; our initialization ensures well-scaled post-activations. \added{Empirically, STAF is a strong and broadly useful member of the sinusoidal activation family for INRs, with the largest gains appearing on tasks with mixed frequencies or repetitive fine detail, especially without positional encodings.} Periodic activations do not remove spectral bias outright, but they improve the practical capacity--convergence trade-off, offering a simple, general recipe for higher-fidelity INRs.

\bibliography{tmlr}
\bibliographystyle{tmlr}

\newpage
\tableofcontents
\newpage
\appendix
\section*{Appendix}
\label{appendix}
\addtocontents{toc}{\protect\setcounter{tocdepth}{2}}
\section{Analytic NTK}
\label{sec:ntk}
In this section, we compute the analytic NTK for a neural network that uses the proposed activation function (STAF), following the notation from \citep{modernml2024}. Interested readers can also refer to \citep{jacot2018neural} and \citep{golikov2022neural}. However, we chose \citep{modernml2024} for its clarity and ease of understanding. According to \citep{modernml2024}, the NTK of an activation function for a neural network with $L-1$ hidden layers is as follows.
\begin{theorem}
(Theorem 1 of \citep{modernml2024}, Lecture 6) For $\boldsymbol{x} \in \mathcal{S}^{d-1}$, let $f_{\boldsymbol{x}}^{(L)}(\boldsymbol{w}):\mathbb{R}^p \rightarrow\mathbb{R}$ denote a neural network with $L-1$ hidden layers such that:
\begin{equation}
    f_{\boldsymbol{x}}^{(L)}(\boldsymbol{w}) = \boldsymbol{W^{(L)}} \frac{1}{\sqrt{F_{L-1}}} \phi \left( \boldsymbol{W^{(L-1)}} \frac{1}{\sqrt{F_{L-2}}} \phi \left( \ldots \boldsymbol{W^{(2)}} \frac{1}{\sqrt{F_1}} \phi \left( \boldsymbol{W^{(1)}} \boldsymbol{x} \right) \ldots \right) \right);
\end{equation}
where $W^{(i)} \in \mathbb{R}^{F_i \times F_{i-1}}$ for $i \in \{1, \ldots, L\}$ with $F_0 = d$, $F_L = 1$, and $\phi : \mathbb{R} \rightarrow \mathbb{R}$ is an element-wise activation function. As $F_1, F_2, \dots, F_{L-1} \to \infty$ in order, the Neural Network Gaussian Process (NNGP), denoted as $\Sigma^{(L)}$, and the NTK, denoted as $K^{(L)}$, of $f_{\boldsymbol{x}}(\boldsymbol{w})$ are given by:
\begin{equation}
    \begin{aligned}
    &\Sigma^{(L)}(\boldsymbol{x},\Tilde{\boldsymbol{x}}) = \check{\phi}\left(\Sigma^{(L-1)}(\boldsymbol{x},\Tilde{\boldsymbol{x}})\right);\quad \Sigma^{(0)}(\boldsymbol{x},\Tilde{\boldsymbol{x}})=\boldsymbol{x}^T \Tilde{\boldsymbol{x}} \\
    & K^{(L)}(\boldsymbol{x},\Tilde{\boldsymbol{x}}) = \Sigma^{(L)}(\boldsymbol{x},\Tilde{\boldsymbol{x}}) + K^{(L-1)}(\boldsymbol{x},\Tilde{\boldsymbol{x}})\check{\phi'}\left(\Sigma^{(L-1)}(\boldsymbol{x},\Tilde{\boldsymbol{x}})\right);\\
    & K^{(0)}(\boldsymbol{x},\Tilde{\boldsymbol{x}}) = \boldsymbol{x}^T \Tilde{\boldsymbol{x}}
    \end{aligned}
\end{equation}
where $\check{\phi}: [-1, 1] \rightarrow \mathbb{R}$ is the dual activation for $\phi$, and is calculated as follows:
\begin{equation}
\check{\phi}(\xi) = \mathbb{E}_{(u,v) \sim \mathcal{N}(0,\boldsymbol{\Lambda})} [\phi(u)\phi(v)] \quad\text{where}~ \boldsymbol{\Lambda} = 
\begin{bmatrix}
1 & \xi \\
\xi & 1 
\end{bmatrix}.
\end{equation}
Furthermore, $\phi$ is normalized such that $\check{\phi}(1) = 1$.
\end{theorem}

Consequently, it suffices to calculate $\check{\phi}$. It has been calculated in the following theorem. Just like what mentioned in \citep{wang2023learning}, we assume that the optimization of neural networks with STAF can be decomposed into two phases, where we learn the coefficients of STAF in the first phase and then train the parameters of the neural network in the second phase. This assumption is reasonable as the number of parameters of STAF is far less than that of networks, and they quickly converge at the early stage of training. \added{Empirically, we support this approximation with two experiments reported in~\Cref{sec:ntk_empirical_validation}: (i) tracking the dynamics of the sinusoidal parameters $(\Omega,\Phi,C)$ during training, which shows that their updates are concentrated in the early stage and then stabilize, and (ii) a freeze-after-warmup experiment, where these parameters are frozen after an initial training period and the model still achieves nearly the same reconstruction quality.} As a result, in the following theorem, all the parameters except weights are fixed, since they have been obtained in the first phase of training.

\begin{theorem} \label{analytic_NTK}
    Let $\rho^*$ be the proposed activation function (STAF). Then
    \begin{align}
    &\check{\rho^*}(\xi)= \sum_{i=1}^{\tau} \sum_{j=1}^{\tau} C_i C_j \Delta_{i,j} \nonumber \\
    &= \frac{1}{2}\sum_{i=1}^{\tau} \sum_{j=1}^{\tau} C_i C_j e^{\frac{-1}{2}\left(\Omega_i^2+\Omega_j^2\right)}\left(e^{\Omega_i \Omega_j \xi} \cos(\Phi_i - \Phi_j)+ e^{-\Omega_i \Omega_j \xi} \cos(\Phi_i + \Phi_j)\right)
\end{align}

Therefore,
\begin{equation}
\scalebox{0.98}{$
\check{\rho^*}'(\xi)= \frac{1}{2} \sum_{i=1}^{\tau} C_i \Omega_i \sum_{j=1}^{\tau} \left[C_j \Omega_j e^{\frac{-1}{2}(\Omega_i^2+\Omega_j^2)}\left(e^{\Omega_i \Omega_j \xi} \cos(\Phi_i - \Phi_j) - e^{-\Omega_i \Omega_j \xi} \cos(\Phi_i + \Phi_j)\right)\right].
$}
\end{equation}
\end{theorem}
\begin{proof}
\begin{align} \label{dual}
\check{\rho^*}(\xi) &=\mathbb{E}_{(u,v) \sim \mathcal{N}(0,\boldsymbol{\Lambda})} [\rho^*(u)\rho^*(v)] \nonumber \\
&= \mathbb{E}_{(u,v) \sim \mathcal{N}(0,\boldsymbol{\Lambda})} 
\left[
    \sum_{i=1}^{\tau} C_i \sin(\Omega_i u + \Phi_i) \sum_{i=1}^{\tau} C_i \sin(\Omega_i v + \Phi_i)
\right]
\nonumber\\
&= \mathbb{E}_{(u,v)\sim\mathcal{N}(0,\boldsymbol{\Lambda})}
\left[
    {\sum}_{i=1}^\tau {\sum}_{j=1}^\tau C_iC_j\sin(\Omega_ iu+\Phi_ i)
        \sin(\Omega_ jv+\Phi_ j)
\right] \nonumber\\
&= \sum_{i=1}^{\tau} \sum_{j=1}^{\tau} C_i C_j \mathbb{E}_{(u,v) \sim \mathcal{N}(0,\boldsymbol{\Lambda})} 
\Bigl(\sin(\Omega_i u + \Phi_i)\sin(\Omega_j v + \Phi_j)\Bigr).
\end{align}

So, we need to compute the following expectation:
\begin{equation} \label{delta}
\Delta_{i,j}=\mathbb{E}_{(u,v) \sim \mathcal{N}(0,\boldsymbol{\Lambda})} \left(\sin(\Omega_i u + \Phi_i)\sin(\Omega_j v + \Phi_j)\right)
\end{equation}

Note that for a random vector $\mathbf{X}=(X_1,\dots,X_d)^T$ with mean vector $\boldsymbol{\mu}$ and covariance matrix $\boldsymbol{\Lambda}$, the joint probability density function (PDF) is as follows:
\begin{equation}
    f_{\mathbf{X}}(\mathbf{x})=(2\pi)^{-d/2}\det(\boldsymbol{\Lambda})^{-1/2}e^{(\frac{-1}{2}(\mathbf{x}-\boldsymbol{\mu})^T\boldsymbol{\Lambda}^{-1}(\mathbf{x}-\boldsymbol{\mu}))}.
\end{equation}
As a result, since $\boldsymbol{\Lambda}^{-1} = \frac{1}{1 - \xi^2} \begin{bmatrix} 1 & -\xi \\ -\xi & 1 \end{bmatrix}$, we will have:
\begin{align} \label{joint_pdf}
f_{U,V}(u,v) &= \frac{1}{2\pi\sqrt{1 - \xi^2}} e^{-\frac{1}{2} 
\begin{pmatrix} u & v \end{pmatrix} 
\boldsymbol{\Lambda}^{-1} 
\begin{pmatrix} u \\ v \end{pmatrix}} = \frac{1}{2\pi\sqrt{1 - \xi^2}} e^{\frac{-1}{2(1-\xi^2)} 
\begin{pmatrix} u & v \end{pmatrix}
\begin{pmatrix} 1 & -\xi \\ -\xi & 1 \end{pmatrix}
\begin{pmatrix} u \\ v \end{pmatrix}} \nonumber\\
&= \frac{1}{2\pi\sqrt{1-\xi^2}} e^{\frac{-(u^2- 2\xi uv +v^2)}{2(1-\xi^2)}}.
\end{align}

Consequently, using \Cref{dual,delta}, we have
\begin{align} \label{delta_in_terms_of_I1}
\Delta_{i,j}&= \int_{-\infty}^{\infty} \int_{-\infty}^{\infty}
\Bigl( \sin(\Omega_i u + \Phi_i) \sin(\Omega_j v + \Phi_j)f_{U,V}(u,v) \Bigr) dudv \nonumber \\
&=\frac{1}{2\pi\sqrt{1 - \xi^2}}
\int_{-\infty}^\infty
\sin(\Omega_j v + \Phi_j)I_1dv;
\end{align}
where
\begin{align}
&I_1 = \int_{-\infty}^\infty \sin(\Omega_i u+\Phi_i) e^{\frac{-(u^2 - 2\xi uv + v^2)}{2(1 -\xi^2)}}du = e^{\frac{-v^2}{2(1-\xi^2 )}} \int_{-\infty}^{\infty} \sin(\Omega_i u+\Phi_i) e^{\frac{-(u^2-2\xi uv)}{2(1-\xi^2)}}du  \nonumber \\
&= e^\frac{-v^2+\xi^2 v^2}{2(1-\xi^2)} \int_{-\infty}^{\infty} \sin(\Omega_i u+\Phi_i) e^\frac{-(u^2-2\xi uv+\xi^2 v^2)}{2(1-\xi^2)}du = e^{-v^2/2}\int_{-\infty}^{\infty} \sin(\Omega_i u+\Phi_i) e^\frac{-(u-\xi v)^2}{2(1-\xi^2)}du
\end{align}
By assuming $\eta=u-\xi v$ we will have:
\begin{align}
    I_1 &= e^{-v^2/2}\int_{-\infty}^{\infty}\sin(\Omega_i (\eta+\xi v)+\Phi_i) e^\frac{-\eta^2}{2(1-\xi^2)}d\eta
\end{align}
Before going further, we need to consider the following lemma.
\begin{lemma} \label{integral_lemma}
\begin{align}
    \int_{-\infty}^{\infty}\cos(\alpha u + \beta)e^{-\gamma u^2}du = \sqrt{\frac{\pi}{\gamma}}e^{-\frac{\alpha^2}{4\gamma}}\cos\beta,  \label{Improper_integral_1} \\
    \int_{-\infty}^{\infty}\sin(\alpha u + \beta)e^{-\gamma u^2}du = \sqrt{\frac{\pi}{\gamma}}e^{-\frac{\alpha^2}{4\gamma}}\sin\beta \label{Improper_integral_2}
\end{align}
The proof is provided in \eqref{proof_of_integral_lemma}.
\end{lemma}

Let $\alpha=\Omega_i$, $\beta=\Omega_i\xi v+\Phi_i$, and $\gamma=\frac{1}{2(1-\xi^2)}$. As a result of equation \eqref{Improper_integral_2}, we have
\begin{align}
    I_1 &= e^{-v^2/2} \sqrt{2\pi(1-\xi^2)} e^\frac{-\Omega_i^2}{2/(1-\xi^2)}\sin(\Omega_i \xi v+\Phi_i) \nonumber \\
    &= \sqrt{2\pi(1-\xi^2)} e^\frac{-(v^2+\Omega_i^2 (1-\xi^2 ))}{2}\sin(\Omega_i \xi v+\Phi_i)
\end{align}
Therefore, based on \eqref{delta_in_terms_of_I1}, we will have
\begin{align}
&\Delta_{i,j}= \frac{1}{2\pi\sqrt{1-\xi^2}} 
    \int_{-\infty}^{\infty} 
    \Bigl[\sin(\Omega_j v + \Phi_j)\sqrt{2\pi(1-\xi^2)}e^\frac{-(v^2+\Omega_i^2 (1-\xi^2 ))}{2}
    \sin(\Omega_i \xi v+\Phi_i)\Bigr]dv\nonumber \\
&= \frac{e^\frac{(-\Omega_i^2(1-\xi^2)}{2}}{\sqrt{2\pi}}\int_{-\infty}^\infty 
    \Bigl[\sin(\Omega_j v+ \Phi_j)e^{-v^2/2}\sin(\Omega_i \xi v+\Phi_i)\Bigr] dv \nonumber \\
&= \frac{e^{-\Omega_i^2 (1-\xi^2)/2}}{\sqrt{2\pi}}
  \int_{-\infty}^\infty e^{-v^2/2}~\aleph~dv
\end{align}
where
\begin{align}
    \aleph &= \frac{1}{2}\bigl[\cos(v(\Omega_i\xi - \Omega_j)+\Phi_i-\Phi_j) -\cos(v(\Omega_i\xi+\Omega_j) + \Phi_i + \Phi_j)\bigr]
\end{align}
Therefore,
\begin{align} \label{final_delta}
&\Delta_{i,j} = \frac{e^{-\Omega_i^2(1-\xi^2)/2}}{2\sqrt{2\pi}} \left(\sqrt{2\pi} e^{-(\Omega_i\xi - \Omega_j)^2/2} \cos(\Phi_i - \Phi_j)+\sqrt{2\pi} e^{-(\Omega_i\xi + \Omega_j)^2/2} \cos(\Phi_i + \Phi_j)\right)\nonumber \\
&= \frac{e^{-\Omega_i^2 (1-\xi^2)/2}}{2} \left(e^{- (\Omega_i\xi - \Omega_j)^2/2} \cos(\Phi_i - \Phi_j)+e^{- (\Omega_i\xi + \Omega_j)^2/2} \cos(\Phi_i + \Phi_j)\right) \nonumber \\
&= \frac{e^\frac{-\Omega_i^2 (1-\xi^2)}{2} e^{\frac{- (\Omega_i^2 \xi^2+\Omega_j^2)}{2}}}{2} \left(e^{\Omega_i \Omega_j \xi} \cos(\Phi_i - \Phi_j)+e^{-\Omega_i \Omega_j \xi} \cos(\Phi_i + \Phi_j)\right) \nonumber \\
&= \frac{e^{\frac{-1}{2}\left(\Omega_i^2+\Omega_j^2\right)}}{2} \left(e^{\Omega_i \Omega_j \xi} \cos(\Phi_i - \Phi_j)+e^{-\Omega_i \Omega_j \xi} \cos(\Phi_i + \Phi_j)\right)
\end{align}

As a result of \Cref{dual,final_delta}, we have
\begin{align}
    \check{\rho^*}(\xi)&= \sum_{i=1}^{\tau} \sum_{j=1}^{\tau} C_i C_j \Delta_{i,j} \nonumber \\
    &= \frac{1}{2}\sum_{i=1}^{\tau} \sum_{j=1}^{\tau} C_i C_j e^{\frac{-1}{2}\left(\Omega_i^2+\Omega_j^2\right)}\left(e^{\Omega_i \Omega_j \xi} \cos(\Phi_i - \Phi_j)+e^{-\Omega_i \Omega_j \xi} \cos(\Phi_i + \Phi_j)\right)
\end{align}
\end{proof}

\addtocontents{toc}{\protect\setcounter{tocdepth}{3}}
\subsubsection{Proof of \eqref{integral_lemma}} \label{proof_of_integral_lemma}
\begin{proof}
    We want to calculate these integrals:
\begin{align}
    I_1 = \int_{-\infty}^{\infty}\cos(\alpha u + \beta)e^{-\gamma u^2}du, \nonumber\\
    I_2 = \int_{-\infty}^{\infty}\sin(\alpha u + \beta)e^{-\gamma u^2}du
\end{align}
By adding them we will have
\begin{align}
I_1+iI_2 &=\int_{-\infty}^{\infty}e^{-\gamma u^2}\bigl(\cos(\alpha u +\beta)+i\sin(\alpha u +\beta)\bigr)du =\int_{-\infty}^{\infty} e^{i (\alpha u + \beta)} e^{-\gamma u^2} du \nonumber \\
&= e^{i \beta} \int_{-\infty}^{\infty} e^{-\gamma(u^2 + \frac{\alpha i}{\gamma} u)} du= e^{i \beta} \int_{-\infty}^{\infty} e^{-\gamma(u^2 + \frac{\alpha i}{\gamma} u - \frac{\alpha^2}{4\gamma^2})} e^{-\frac{\alpha^2}{4\gamma}} du\nonumber \\
&= e^{-\frac{\alpha^2}{4\gamma} + i \beta} \int_{-\infty}^{\infty} e^{-\gamma(u^2 + \frac{\alpha i}{\gamma} u - \frac{\alpha^2}{4\gamma^2})} du = e^{-\frac{\alpha^2}{4\gamma} + i \beta} \underbrace{\int_{-\infty}^{\infty} e^{-\gamma\left(u + \frac{\alpha i}{2\gamma}\right)^2} du}_{I_3}
\end{align}
where $i$ is the unit imaginary number. Since we know that the integral of an arbitrary Gaussian function is
\begin{equation}
    \int_{-\infty}^{\infty}  e^{-a(x+b)^2}\,dx= \sqrt{\frac{\pi}{a}},
\end{equation}
we will have $I_3=\sqrt{\frac{\pi}{\gamma}}$. Therefore,
\begin{equation}
    I_1+iI_2 = \sqrt{\frac{\pi}{\gamma}}e^{-\frac{\alpha^2}{4\gamma} + i \beta} = \sqrt{\frac{\pi}{\gamma}}e^{-\frac{\alpha^2}{4\gamma}}(\cos\beta+i\sin\beta)
\end{equation}
As a result,
\begin{equation}
    I_1=\sqrt{\frac{\pi}{\gamma}}e^{-\frac{\alpha^2}{4\gamma}}\cos\beta,\quad
    I_2=\sqrt{\frac{\pi}{\gamma}}e^{-\frac{\alpha^2}{4\gamma}}\sin\beta.
\end{equation}
\end{proof}

\begin{figure*}[!th]
    \centering
    \includegraphics[width=\textwidth]{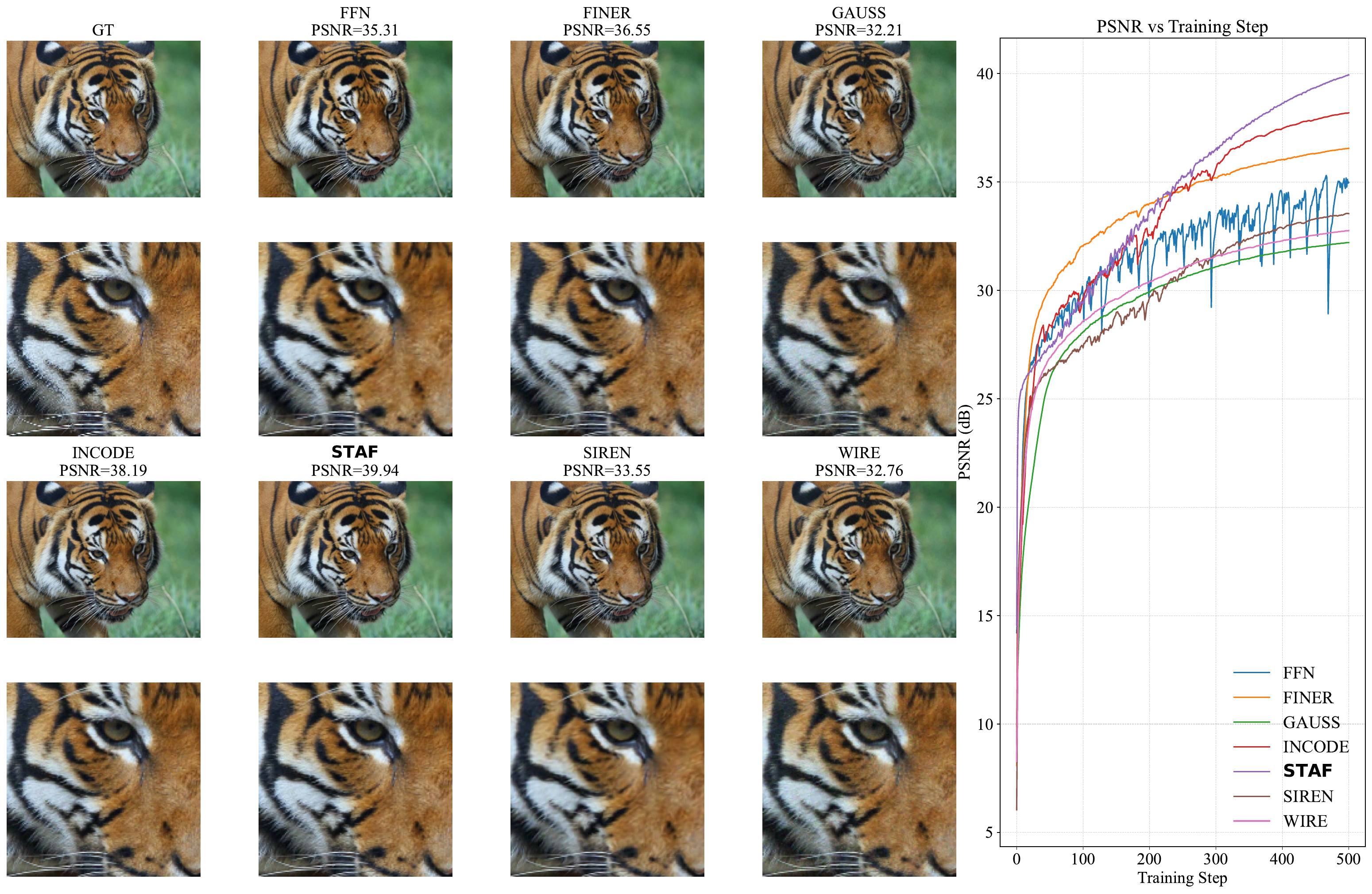}\\
    \includegraphics[width=\textwidth]{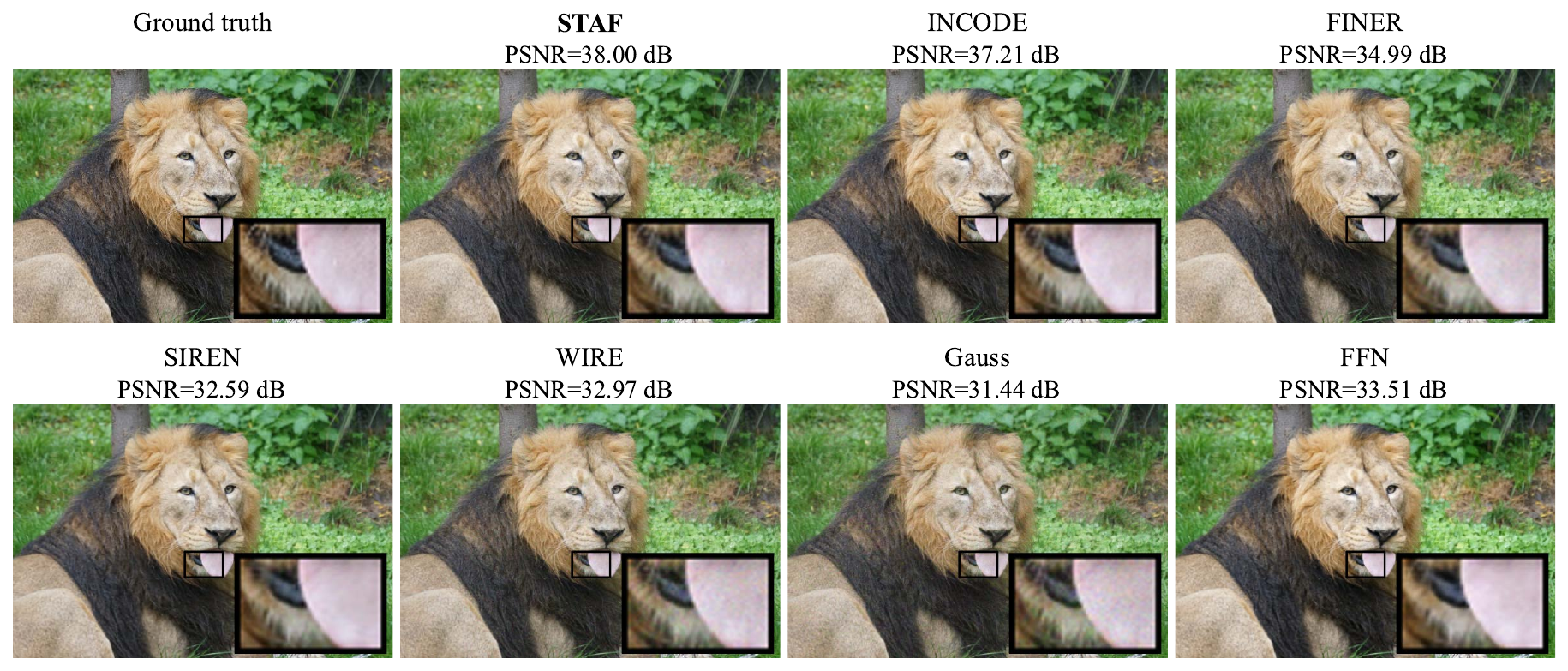}
\caption{\textbf{Qualitative and quantitative comparison of image representations.}
\textit{Top:} reconstructions of a DIV2K image using \textbf{STAF} and competing activations (PSNR shown above each result), with a zoomed-in crop highlighting high-frequency details.
\textit{Right:} PSNR versus training step over 500 iterations, where \textbf{STAF} converges faster and reaches the highest final PSNR.
\textit{Bottom:} a second example with inset crops, showing that \textbf{STAF} better preserves fine textures and sharp boundaries compared to prior activations.}    
    \label{fig:image_rep}
\end{figure*}
\begin{figure}[!th]
    \centering    
    \includegraphics[width=\textwidth]{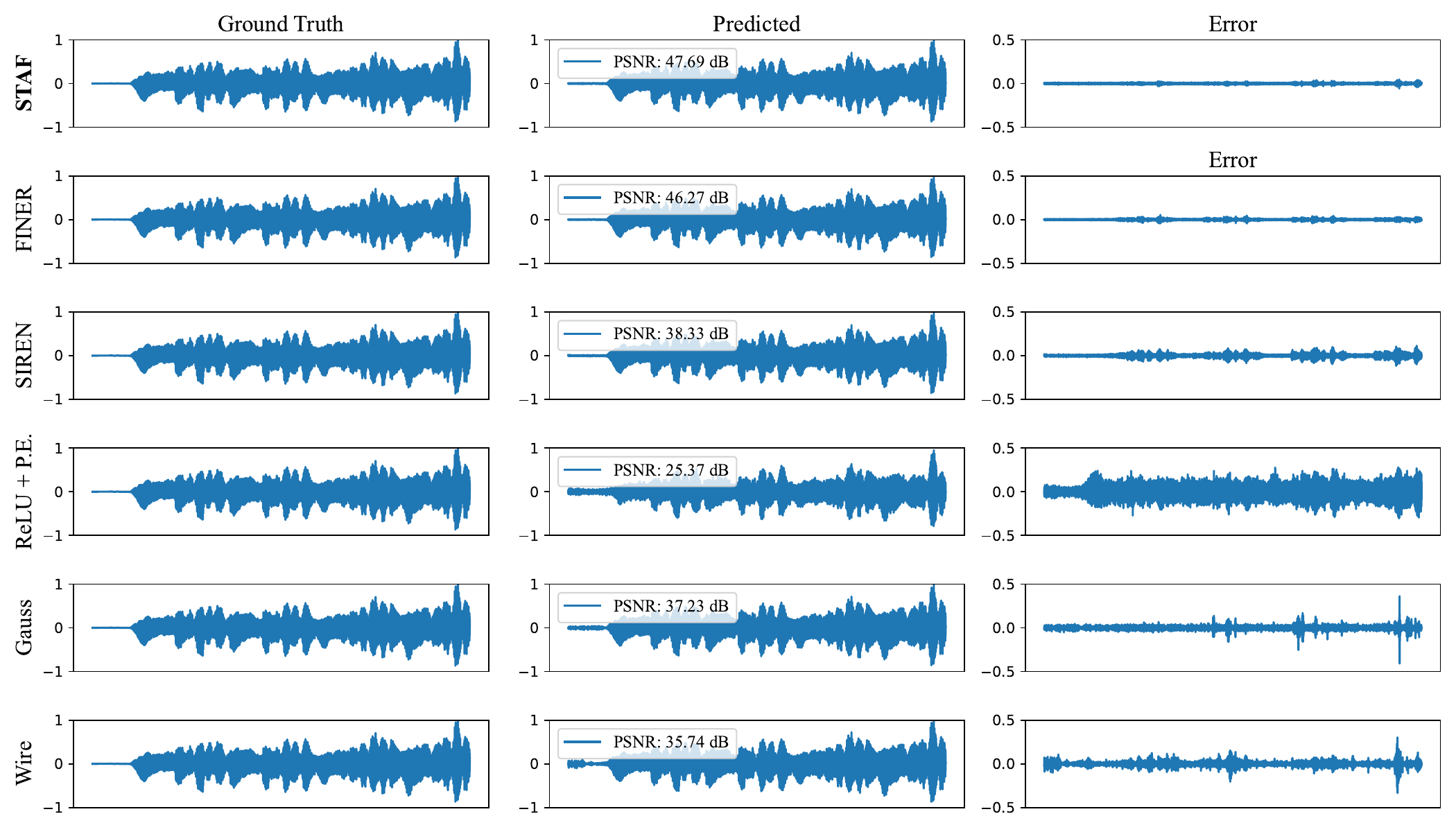}
\caption{\textbf{Audio representation comparison.} For each activation, we show the \textit{ground-truth} waveform (left), the \textit{reconstructed} waveform (middle; PSNR reported), and the \textit{reconstruction error} (right). \textbf{STAF} yields the smallest error and the highest PSNR, indicating more faithful recovery of both low- and high-amplitude temporal structures.}    \label{fig:audio}
\end{figure}
\begin{figure}[!th]
    \centering    
    \includegraphics[width=\textwidth]{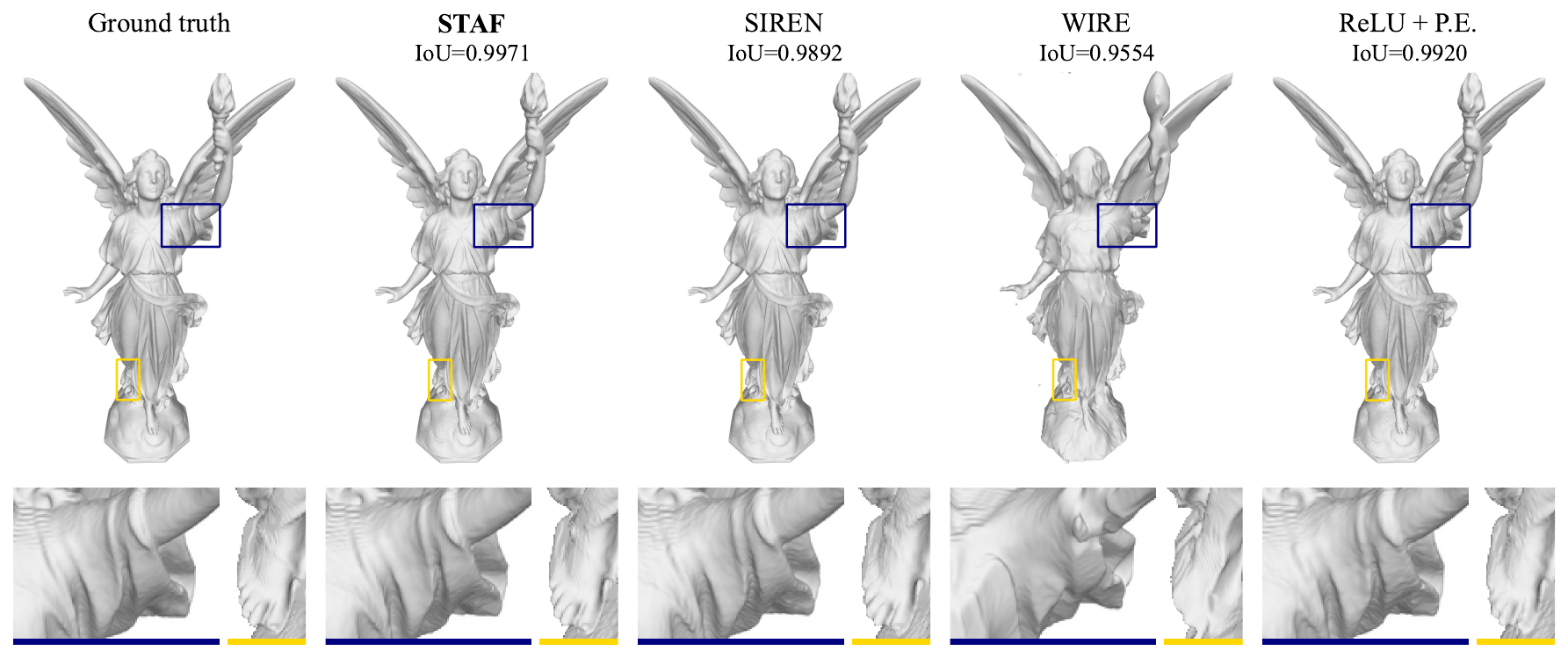}
\caption{\textbf{3D shape representation comparison.} Reconstruction results for a representative shape (IoU reported above each method). Colored boxes indicate zoom-in regions (blue/yellow) shown below. \textbf{STAF} produces sharper local geometry and fewer surface artifacts in the zoomed views, consistent with its higher IoU compared to alternative activations.}    \label{fig:sdf}
\end{figure}

\begin{figure*}[!th]
    \centering
    \includegraphics[width=0.95\textwidth]{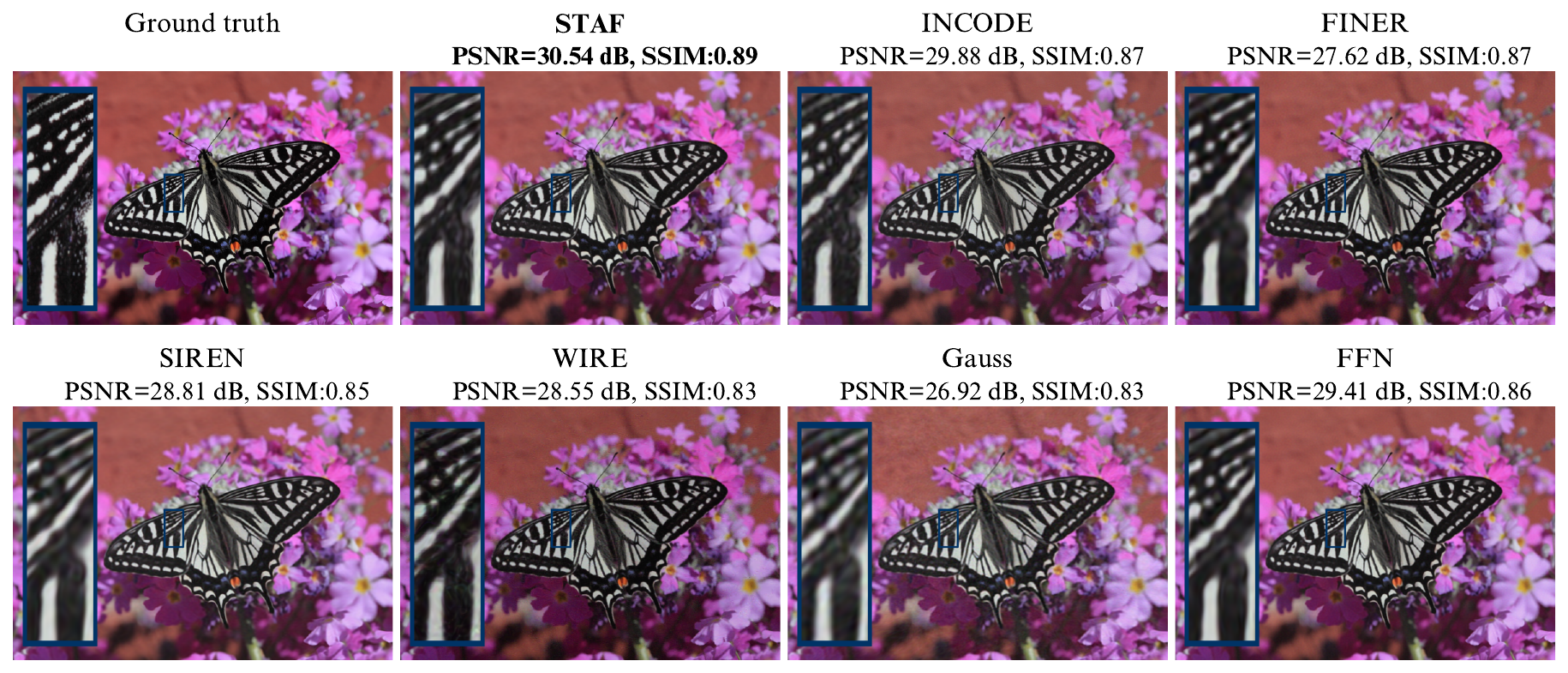}\\
\caption{\textbf{$4\times$ super-resolution comparison.} Qualitative results for \textbf{STAF} and competing activations (PSNR/SSIM shown above each reconstruction). Insets highlight a challenging high-frequency region, where \textbf{STAF} better preserves fine stripe patterns and edge sharpness, consistent with its higher PSNR/SSIM.}    \label{fig:4x_super_res}
\end{figure*}

\begin{figure*}[!th]
    \centering
    \begin{tabular}{cc} 
        \begin{subfigure}
            \centering
            \includegraphics[width=0.72\textwidth]{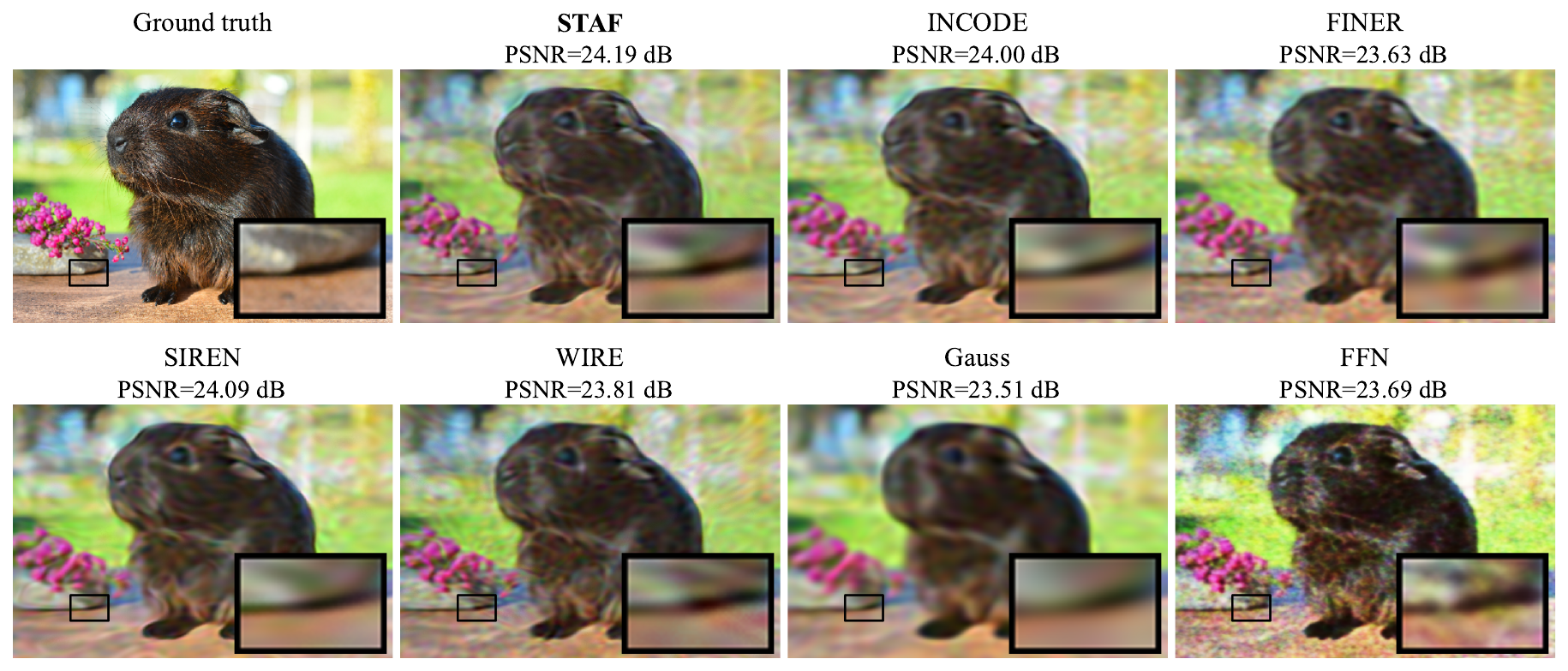}
            \label{fig:denoising_1}
        \end{subfigure} &
        \raisebox{10mm}{ 
            \begin{subfigure}
                \centering
                \includegraphics[width=0.24\textwidth]{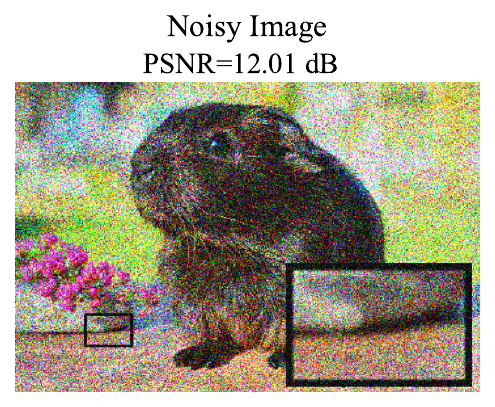}
                \label{fig:denoising_2}
            \end{subfigure}
        }
    \end{tabular}
    \vspace{-1em}
\caption{\textbf{Image denoising comparison.} Reconstructions from the same noisy input (right; PSNR shown) using \textbf{STAF} and competing activations (PSNR reported above each result). Insets highlight a textured region, where \textbf{STAF} better suppresses noise while preserving edges and smooth color transitions, consistent with its higher PSNR.}    \label{fig:img_denoising}
\end{figure*}

\section{Additional Experimental Results}
\label{sec:continue}
\subsection{Neural Radiance Fields (NeRFs)}
\label{sec:nerf}
To implement our NeRF experiments, we adopted an approach inspired by~\citep{saragadam2023wire}, utilizing the publicly available \texttt{torch-ngp} framework~\citep{tang2022torch} for training. Our architecture consists of two separate MLPs: one predicts the volumetric density ($\sigma$), while the other outputs the color values (RGB). Both networks are implemented as 4-layer MLPs with 128 hidden units per layer. The sigma network received only spatial coordinates $(x, y, z)$, while the RGB network also included viewing direction parameters $(\theta, \phi)$. We conducted experiments on five different scenes, using 100 training images, each downsampled to $400 \times 400$ resolution. The generalization of the model was evaluated using 200 additional test views. All training was performed on an NVIDIA A40 GPU with 48 GB of memory. During training, we used the learning rate $3 \times 10^{-4}$ for STAF, $5 \times 10^{-4}$ for FINER following their codebase, and followed the same optimized learning rates using~\citep{kazerouni2024incode} for other methods. For sinusoidal-based methods, we used $\omega_0 = 40$ (SIREN, WIRE, STAF), $\omega_0 = 30$ for FINER, and $\sigma_0 = 40$ (WIRE and Gauss). Apart from ReLU, no positional encoding was used for other nonlinearities to isolate their representational capabilities.

\subsection{Results on the Tokyo Image}
\label{sec:tokyo}
We additionally evaluated STAF on a high-resolution image to demonstrate its scalability and ability to reconstruct fine details. STAF both quantitatively and qualitatively outperforms all baseline methods, which is depicted in~\cref{fig:tokyo}. We have highlighted key differences in~\cref{fig:tokyo_patch}. As shown, STAF provides a more accurate reconstruction by preserving fine details and colors. For example, in the \textit{first row}, STAF correctly reconstructs the text in green, matching the ground truth, while other methods tend to reconstruct it in blue. Furthermore, in the traffic sign image \textit{(last row)}, only STAF successfully reconstructs the \texttt{"No Parking"} sign, faithfully preserving the circular blue background, the red border, and the diagonal slash, closely matching the ground truth. In contrast, InstantNGP~\citep{muller2022instant} and SIREN fail to recover the fine structure and distort both color and shape. FINER partially retains the circular shape but loses critical color fidelity. These results demonstrate the superior performance of STAF in capturing fine-grained visual details, which is also supported by the provided error map. In this experiment, we use the PyTorch InstantNGP training pipeline, substituting the original ReLU activations with FINER, STAF, and SIREN. 

\subsection{Results on the Kodak, CelebA, and DIV2K datasets}
\label{sec:kodak_and_celeba}
Tables~\ref{tab:kodak_comparison} and~\ref{tab:celebA_comparison} summarize the quantitative comparisons using PSNR, SSIM, and LPIPS metrics that respectively assess pixel-wise quality, structural similarity, and perceptual quality. On the Kodak dataset \citep{kodak_dataset,KodakSuiteOriginal}, which consists of 24 high-quality $256 \times 256 \times 3$ images, STAF achieves the highest PSNR and SSIM scores, indicating excellent fidelity and structural similarity in the reconstructed images. Similarly, on the CelebA dataset \citep{liu2015faceattributes}, containing 19,867 $128 \times 128\times 3$ images, STAF outperforms other methods, showcasing its ability to generalize across diverse and large-scale datasets. Surprisingly, our results are sensitive to the LPIPS metric, ranking fourth-best on both datasets. This discrepancy stems from the well-known distortion–perception trade-off. STAF minimizes pixel-wise losses (e.g., MSE), which favor smooth reconstructions and lead to high PSNR/SSIM by approximating the conditional mean. However, LPIPS measures perceptual similarity using deep features that are sensitive to high-frequency details and textures, features often averaged out in pixel-accurate reconstructions, resulting in lower perceptual scores. We also report PSNR values on five randomly selected samples from the DIV2K~\citep{div2k} $512 \times 512$ dataset in~\Cref{tab:div2k_results}, comparing STAF with alternative activation functions. Across all samples, STAF consistently achieves relatively higher reconstruction quality. 

\begin{table}[h]
\centering
\label{tab:main_comparison}
\begin{minipage}{0.48\textwidth}
\scriptsize
\centering
\caption{Kodak dataset (24 images, 1000 training iterations). \colorbox{purple}{\kern-\fboxsep best\kern-\fboxsep} and \colorbox{light}{\kern-\fboxsep second-best\kern-\fboxsep} performance.}
\label{tab:kodak_comparison}
\renewcommand{\arraystretch}{1.2}
\resizebox{\textwidth}{!}{%
\begin{tabular}{lccc}
\toprule
\textbf{Method} & \textbf{PSNR} $\uparrow$ & \textbf{SSIM} $\uparrow$ & \textbf{LPIPS} $\downarrow$ \\
\midrule
STAF ($\tau = 5$) & \colorbox{purple}{57.13 $\pm$ 9.12} & \colorbox{purple}{0.993 $\pm$ 0.082} & {$0.011 \pm 0.052$} \\
SIREN & $29.82 \pm 2.31$ & $0.843 \pm 0.054$ & $0.201 \pm 0.081$ \\
FINER & $43.81 \pm 2.45$ & $0.986 \pm 0.010$ & \colorbox{light}{$0.004 \pm 0.004$} \\
WIRE & $29.02 \pm 2.54$ & $0.828 \pm 0.074$ & $0.101 \pm 0.096$ \\
INCODE & \colorbox{light}{$47.97 \pm 1.23$} & \colorbox{light}{$0.992 \pm 1.535$} & \colorbox{purple}{0.001 $\pm$ 0.001} \\
MFN & $44.61 \pm 1.52$ & $0.986 \pm 0.006$ & \colorbox{purple}{0.001 $\pm$ 0.001} \\
Gauss & $31.75 \pm 2.34$ & $0.886 \pm 0.026$ & $0.068 \pm 0.030$ \\
ReLU + PE & $28.37 \pm 2.012$ & $0.761 \pm 0.049$ & $0.182 \pm 0.047$ \\
\bottomrule
\end{tabular}}
    \end{minipage}
    \hfill
    \begin{minipage}{0.48\textwidth}
    \scriptsize
    \centering
    \caption{CelebA dataset (19,867 images, 250 training iterations). \colorbox{purple}{\kern-\fboxsep best\kern-\fboxsep} and \colorbox{light}{\kern-\fboxsep second-best\kern-\fboxsep} performance.}
    \label{tab:celebA_comparison}
    \renewcommand{\arraystretch}{1.2}
    \resizebox{\textwidth}{!}{%
    \begin{tabular}{lccc}
    \toprule
    \textbf{Method} & \textbf{PSNR} $\uparrow$ & \textbf{SSIM} $\uparrow$ & \textbf{LPIPS} $\downarrow$ \\
    \midrule
    STAF ($\tau = 5$) & \colorbox{purple}{70.932 $\pm$ 4.248} & \colorbox{purple}{0.9970 $\pm$ 0.0075} & $0.0005 \pm 0.0017$ \\
    SIREN & $58.251 \pm 2.322$ & $0.9955 \pm 0.0025$ & $0.0006 \pm 0.0004$ \\
    FINER & $48.96 \pm 2.136$ & $0.9932 \pm 0.0029$ & \colorbox{light}{$0.0003 \pm 0.0003$} \\
    WIRE & $39.283 \pm 2.403$ & $0.9773 \pm 0.0144$ & $0.0082 \pm 0.0052$ \\
    MFN & $44.124 \pm 2.912$ & $0.9890 \pm 0.0059$ & $0.0021 \pm 0.0021$ \\
    INCODE & \colorbox{light}{$59.853 \pm 5.523$} & \colorbox{light}{$0.9964 \pm 0.0019$} & \colorbox{purple}{0.0001 $\pm$ 0.0003} \\
    GAUSS & $47.547 \pm 2.807$ & $0.9934 \pm 0.0055$ & \colorbox{light}{$0.0003 \pm 0.0004$} \\
    ReLU + PE & $31.088 \pm 2.318$ & $0.8899 \pm 0.0372$ & $0.0865 \pm 0.0414$ \\
    \bottomrule
    \end{tabular}}
    \end{minipage}
\end{table}

\section{Ablation Studies}
\label{sec:additional}
In this section, we present ablation studies to demonstrate the effectiveness of STAF.

\begin{figure}[t]
    \centering
     \includegraphics[width=0.8\textwidth]{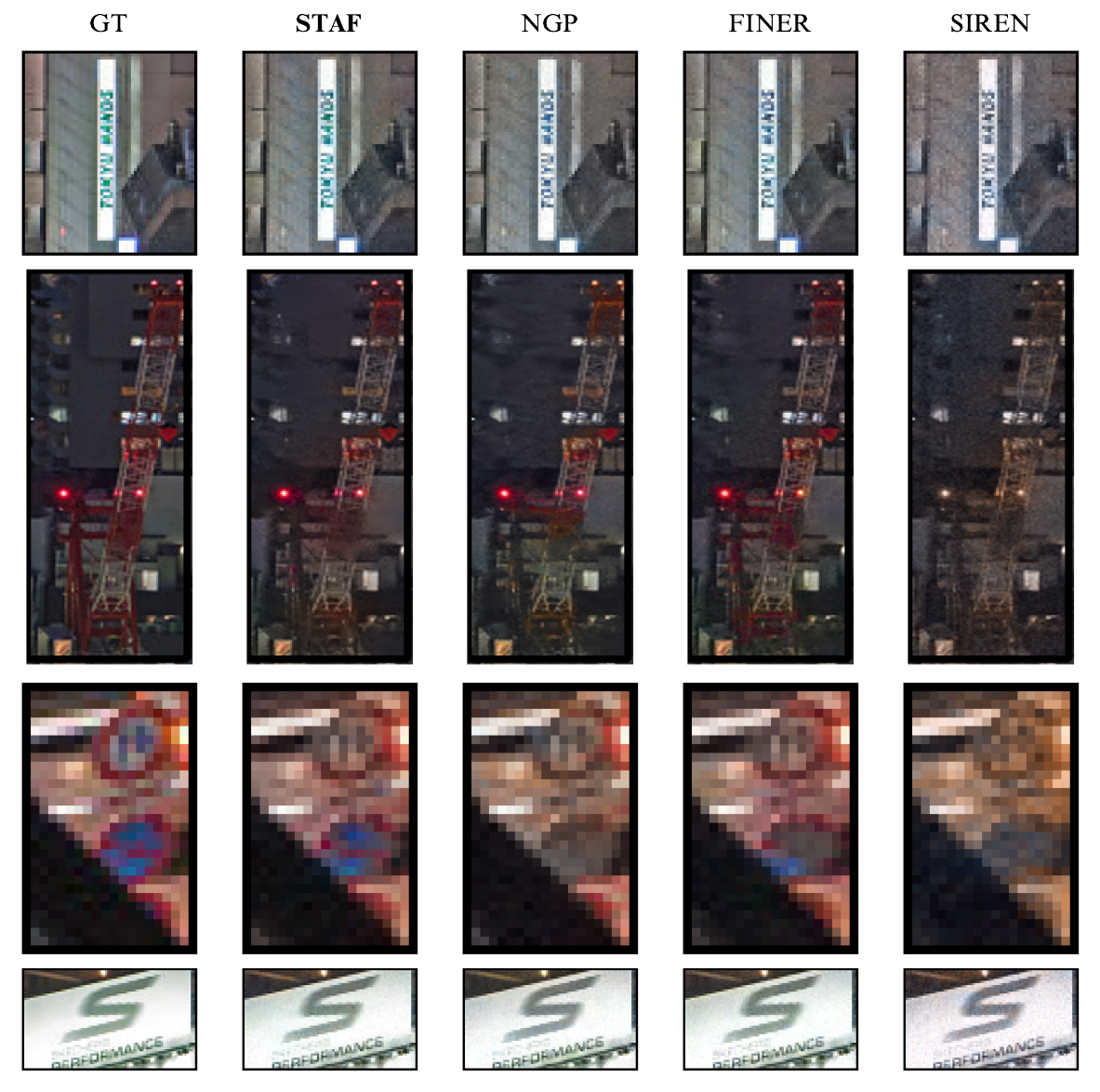}
\caption{\textbf{Tokyo image patch reconstruction.} Cropped regions from the Tokyo scene reconstructed by \textbf{STAF} and competing representations (NGP, FINER, SIREN), compared to the ground truth. Across patches, \textbf{STAF} better preserves fine structures (e.g., text strokes and lattice edges), maintains more accurate color/contrast in dark regions, and reduces blotchy noise and over-smoothing artifacts that appear in the baselines.}    
    \label{fig:tokyo_patch}
\end{figure}

\begin{figure}[!th]
    \centering
    \begin{minipage}[b]{0.48\textwidth}
        \centering
        \includegraphics[width=\textwidth]{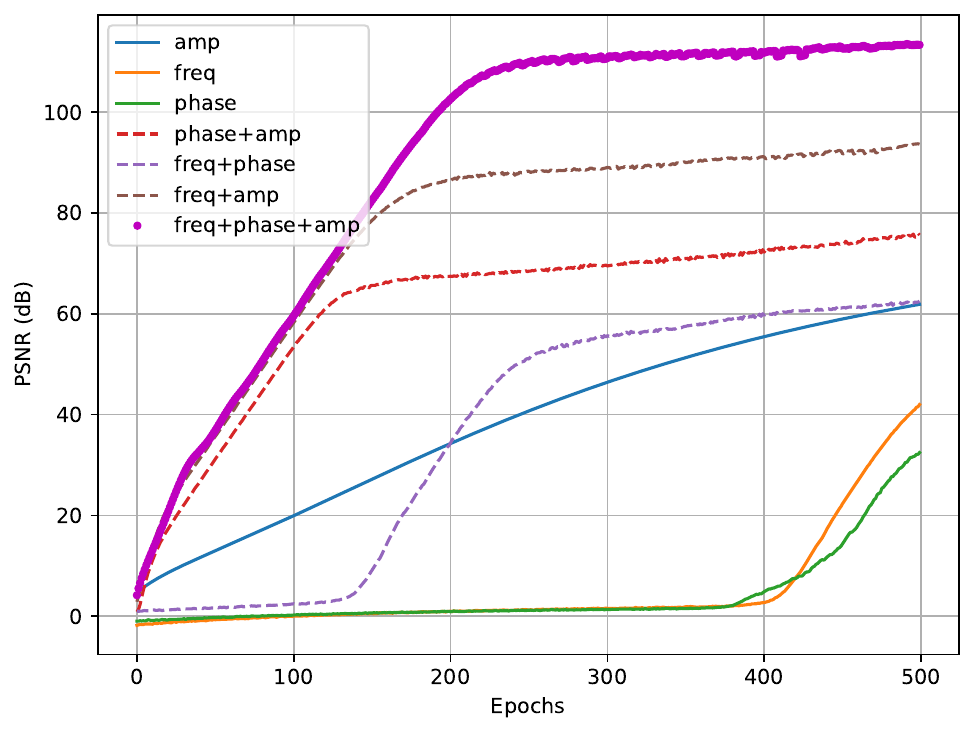}
\caption{\textbf{Ablation of modulation components.} PSNR over training epochs when enabling amplitude (\textbf{amp}), frequency (\textbf{freq}), and phase (\textbf{phase}) individually and in combination. Each component alone provides limited gains, while combining them yields strong improvements; the full \textbf{freq+phase+amp} model converges fastest and attains the highest final PSNR, indicating that the three factors are complementary for recovering high-fidelity details.}        \label{fig:component1}
    \end{minipage}
    \hfill
    \begin{minipage}[b]{0.48\textwidth}
        \centering
        \includegraphics[width=\textwidth]{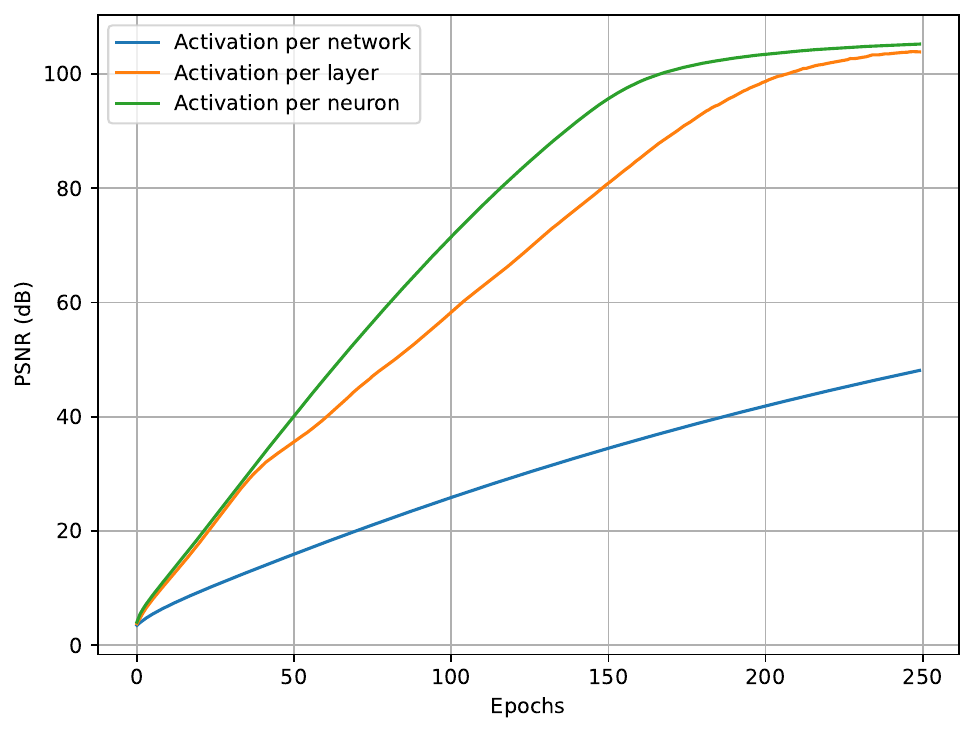}
\caption{\textbf{Granularity of activation selection.} PSNR versus training epochs when applying the activation choice at different levels: \textit{per-network} (single activation shared across all layers), \textit{per-layer}, and \textit{per-neuron}. Finer granularity improves reconstruction quality, with \textit{per-layer} achieving near-\textit{per-neuron} performance while being simpler and more parameter-efficient; we therefore adopt \textbf{per-layer} activation in all experiments.}        \label{fig:network_act_type}
    \end{minipage}
\end{figure}

\begin{figure}[!th]
    \centering
    \begin{minipage}[b]{0.48\textwidth}
        \centering
        \includegraphics[width=\textwidth]{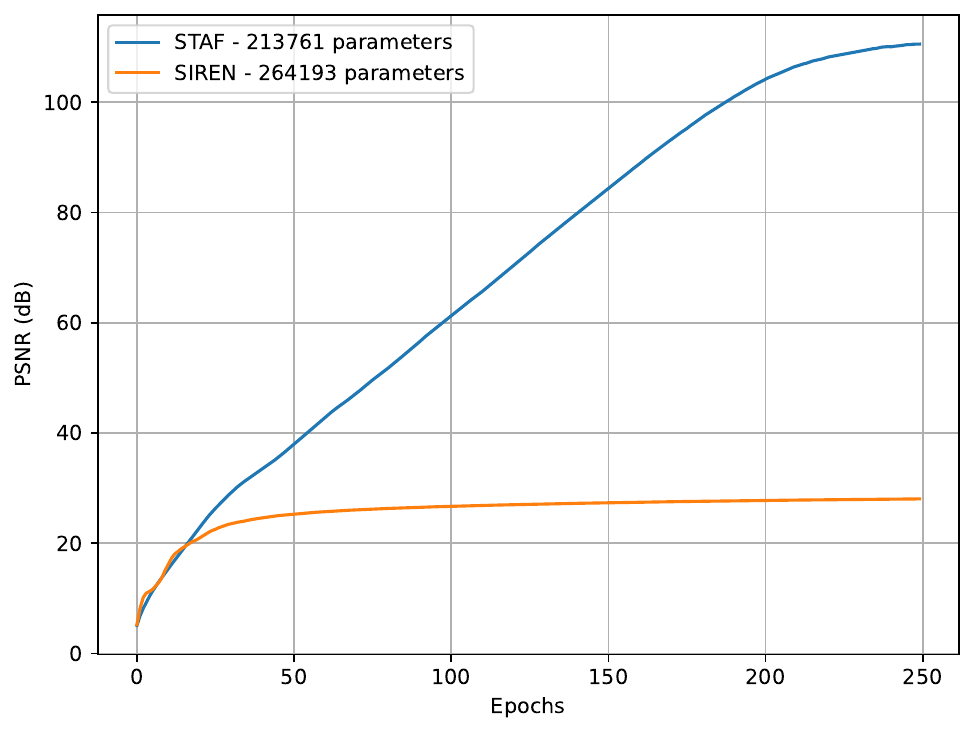}
\caption{\textbf{STAF vs. SIREN.} PSNR over 250 training epochs. Despite using fewer parameters (\textbf{STAF}: 213{,}761 vs. \textbf{SIREN}: 264{,}193), \textbf{STAF} consistently converges faster and achieves higher PSNR throughout training.}
        \label{fig:match_params}
    \end{minipage}
    \hfill
    \begin{minipage}[b]{0.48\textwidth}
        \centering
        \includegraphics[width=\textwidth]{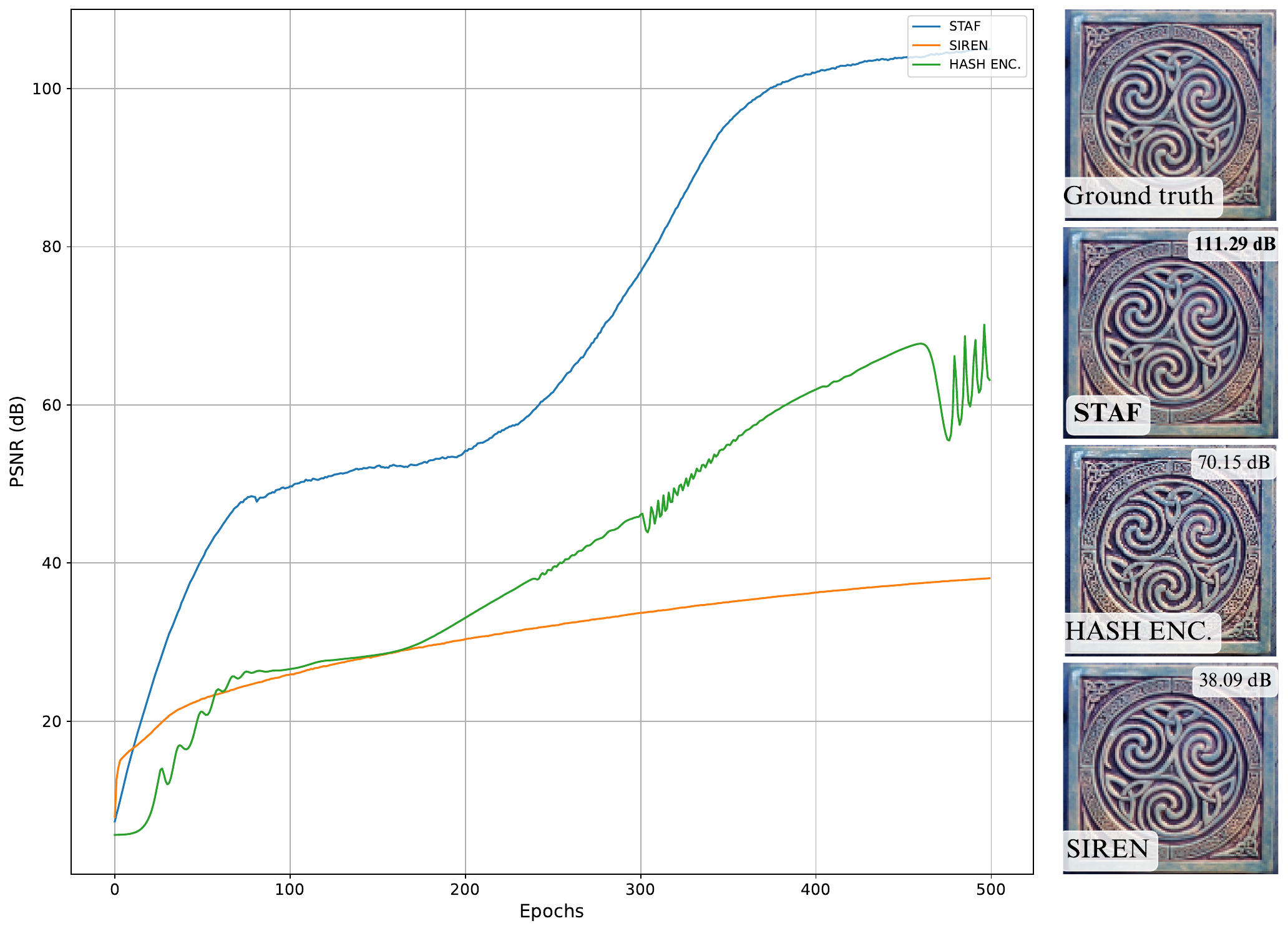}
\caption{\textbf{Single-image reconstruction comparison.} PSNR versus training epochs for STAF, SIREN, and Hash Encoding. STAF converges fastest and attains the highest final PSNR, followed by Hash Encoding, while SIREN lags behind throughout training.}        
\label{fig:hash}
    \end{minipage}
\end{figure}

\subsection{Impact of Amplitude, Frequency, and Phase}
\cref{fig:component1} illustrates the PSNR (dB) over 500 iterations for different component combinations: \textbf{amplitude} ($C_i$'s), \textbf{frequency} ($\Omega_i$'s), \textbf{phase} ($\Phi_i$'s), and their interactions. The model leveraging all three components (freq + phase + amp) achieves the highest PSNR, significantly outperforming individual and partial combinations. This confirms the importance of integrating amplitude, frequency, and phase in the model design for optimal performance, and validates our initial design choices and mathematical analysis.

Additionally, this graph highlights the varying importance of the parameters in our model. Specifically, the amplitudes exhibit the highest significance, followed by the frequencies, with the phases contributing the least. These findings provide valuable guidance for parameter reduction in scenarios with limited training time or hardware resources, enabling more efficient model optimization.

\subsection{Effect of Number of Sinusoids ($\tau$)}
We conducted an ablation study on the number of sinusoids $\tau$ to investigate its impact on model performance and computational efficiency, as summarized in Table~\ref{tab:tau_ablation}. The results were obtained on the Kodak dataset~\citep{kodak_dataset}, and training and inference times were measured on a single NVIDIA T4 GPU with 16GB of memory. We train models for 500 iterations.

By varying $\tau$, we observed that increasing the number of sinusoids improves reconstruction quality, reflected in higher PSNR and SSIM, and lower LPIPS scores, as well as convergence stability (lower variance). However, these improvements come at the cost of increased training and inference times, and higher model complexity. Therefore, selecting an appropriate $\tau$ involves a trade-off between performance and computational overhead, which can be adjusted depending on application-specific constraints. We set $\tau = 5$ as a good trade-off between quality and training time for all tasks, except for denoising, where a smaller value of $\tau = 2$ yielded better performance.

\added{Table~\ref{tab:tau_ablation} also shows that inference time scales approximately linearly with $\tau$, from 0.079s at $\tau=2$ to 1.268s at $\tau=50$. Thus, while increasing $\tau$ improves reconstruction quality, it also incurs a substantial latency cost. We therefore view $\tau$ as a direct quality–latency knob: $\tau=5$ is a strong practical default for most tasks, smaller values are preferable in latency-sensitive settings, and much larger values are best reserved for offline or quality-critical scenarios where runtime is less important.}

\begin{table}[t]
\centering
\begin{minipage}{0.48\textwidth}
\scriptsize
\centering
\caption{Effect of varying $\tau$ on reconstruction quality, training, and inference time.}
\label{tab:tau_ablation}
\renewcommand{\arraystretch}{1.2}
\resizebox{\textwidth}{!}{%
\begin{tabular}{lccccc}
\toprule
& $\tau=2$ & $\tau=5$ & $\tau=10$ & $\tau=20$ & $\tau=50$ \\
\midrule
\textbf{PSNR (dB) ↑} & & & & & \\
Average  & 37.07 & 38.80 & 39.14 & {41.27} & {41.41} \\
Variance & 20.99 & 13.58 & 13.44 & {12.74} & {7.63} \\
\midrule
\textbf{SSIM ↑} & & & & & \\
Average  & 0.947 & 0.965 & 0.968 & {0.971} & {0.972} \\
Variance & 0.0022 & 0.0014 & 0.0011 & {0.0007} & {0.0001} \\
\midrule
\textbf{LPIPS ↓} & & & & & \\
Average  & 0.021 & 0.012 & 0.010 & {0.006} & {0.005} \\
Variance & 0.0020 & 0.0002 & 0.0001 & {6.91e-5} & {5.18e-5} \\
\midrule
\textbf{Training time (s)} & {47.38} & {92.89} & 162.43 & 325.53 & 742.94 \\
\textbf{Inference time (s)} & {0.079} & {0.157} & 0.273 & 0.556 & 1.268 \\
\bottomrule
\end{tabular}}
\end{minipage}
\hfill
\begin{minipage}{0.48\textwidth}
\scriptsize
\centering
\caption{PSNR results on five randomly selected samples from the DIV2K 512$\times$512 dataset.}
\label{tab:div2k_results}
\renewcommand{\arraystretch}{1.2}
\resizebox{\textwidth}{!}{%
\begin{tabular}{lccccccc}
\toprule
\# & STAF & FINER & SIREN & FFN & PEMLP & WIRE & GAUSS \\
\midrule
Sample 1 & \textbf{40.88} & 37.28 & 35.91 & 35.31 & 29.21 & 32.53 & 31.19 \\
Sample 2 & \textbf{34.93} & 34.90 & 30.81 & 31.56 & 23.42 & 31.28 & 30.47 \\
Sample 3 & \textbf{33.38} & 31.69 & 29.64 & 30.29 & 21.63 & 30.52 & 27.43 \\
Sample 4 & \textbf{40.23} & 38.01 & 36.94 & 35.19 & 25.32 & 32.89 & 31.65 \\
Sample 5 & \textbf{34.81} & 33.70 & 33.09 & 33.01 & 22.81 & 30.68 & 27.69 \\
\bottomrule
\end{tabular}}
\end{minipage}
\end{table}

\subsection{Comparative Analysis of Activation Strategies}
\Cref{fig:network_act_type} aligns with the described strategies in \hyperref[sec:implement_str]{Section 3.4} for implementing STAF's parametric activation functions. The per-neuron activation (\textcolor{green}{green} curve) achieves the highest PSNR, demonstrating superior expressiveness, but at the cost of a significant parameter increase, as expected. The network-wide activation (\textcolor{blue}{blue} curve) shows the weakest performance, reflecting limited expressiveness due to shared activation functions across the entire network. The layer-wise activation (\textcolor{orange}{orange} curve) offers a balanced trade-off, achieving nearly the same performance as per-neuron activation while requiring far fewer additional parameters (e.g., 225 parameters for a 3-layer MLP with 25 terms). This supports its use as an efficient and effective strategy, as highlighted in \hyperref[sec:implement_str]{Section 3.4}.


\subsection{Computational Cost}
We provide a detailed analysis in Table~\ref{tab:model_comparison}, comparing all trainable and non-trainable methods in terms of parameter count and FLOPs to better assess the computational cost and complexity of STAF on a single image reconstruction task (on the tiger image in~\cref{fig:image_rep}). As shown, increasing $\tau$ results in only a minimal rise in parameter count, 60 parameters for $\tau = 5$ and 120 for $\tau = 10$, which is negligible relative to the total number of model parameters. This increase yields improved reconstruction quality without a significant rise in FLOPs. Moreover, it is worth noting that although our STAF models have fewer parameters than INCODE, MFN, FFN, and ReLU+PE, they outperform all of them across all metrics. We used the Calflops package~\citep{calculateflops} to calculate the FLOP count.

\begin{table}[ht]
    \centering
    \caption{Performance comparison across models in terms of parameter count, FLOPs, PSNR, SSIM, and LPIPS. FLOPs are calculated using the Calflops package~\citep{calculateflops} with an input shape of (1, 1, 2).}
    \label{tab:model_comparison}
    \resizebox{\textwidth}{!}{%
    \begin{tabular}{l|c|c|c|c|c|c}
    \toprule
    \textbf{Model} & \textbf{\#Parameters} & \textbf{\#FW FLOPs} & \textbf{\#Total (FW+BW) FLOPs} & \textbf{PSNR (dB) ↑} & \textbf{SSIM ↑} & \textbf{LPIPS ↓} \\
    \midrule
    STAF ($\tau = 5$)   & 198,975 & 395.78 K & 1.19 M   & {39.94} & {0.973} & {0.005} \\
    STAF ($\tau = 10$)  & 199,035 & 395.78 K & 1.19 M   & {40.84} & {0.979} & {0.003} \\
    \midrule
    INCODE              & 436,775 & 6.85 G   & 20.55 G  & 38.19 & 0.962 & 0.011 \\
    FINER               & 198,915 & 395.78 K & 1.19 M   & 36.55 & 0.950 & 0.017 \\
    WIRE                & 99,915  & 198.38 K & 595.13 K & 32.76 & 0.885 & 0.058 \\
    MFN                 & 204,291 & 398.85 K & 1.2 M    & 33.87 & 0.942 & 0.020 \\
    Gauss               & 198,915 & 395.78 K & 1.19 M   & 32.21 & 0.888 & 0.070 \\
    FFN                 & 204,035 & 406.02 K & 1.22 M   & 35.31 & 0.944 & 0.037 \\
    SIREN               & 198,915 & 395.78 K & 1.19 M   & 33.55 & 0.928 & 0.062 \\
    ReLU + PE           & 206,083 & 411.14 K & 1.23 M   & 30.10 & 0.856 & 0.150 \\
    \bottomrule
    \end{tabular}%
    }
\end{table}

\subsection{Initialization}
Note that while our initialization technique is designed to address the shortcomings of SIREN's, it remains compatible with existing approaches like SIREN, as it operates on the coefficients rather than the weights. As demonstrated in~\cref{fig:activation_2x2}, the central limit theorem remains valid for the post-activation values, ensuring that the distribution of dot-product values follows $\mathcal{N}(0,1)$ across different $\tau$ values. In addition, with increasing $\tau$, STAF exhibits a greater increase in maximum frequency compared to SIREN, allowing it to capture a broader range of frequencies more effectively.

\begin{figure}[!th]
    \centering

    \resizebox{\linewidth}{!}{%
    \begin{minipage}{\linewidth}
        \centering
        \subfigure[SIREN]{
            \includegraphics[width=0.38\linewidth]{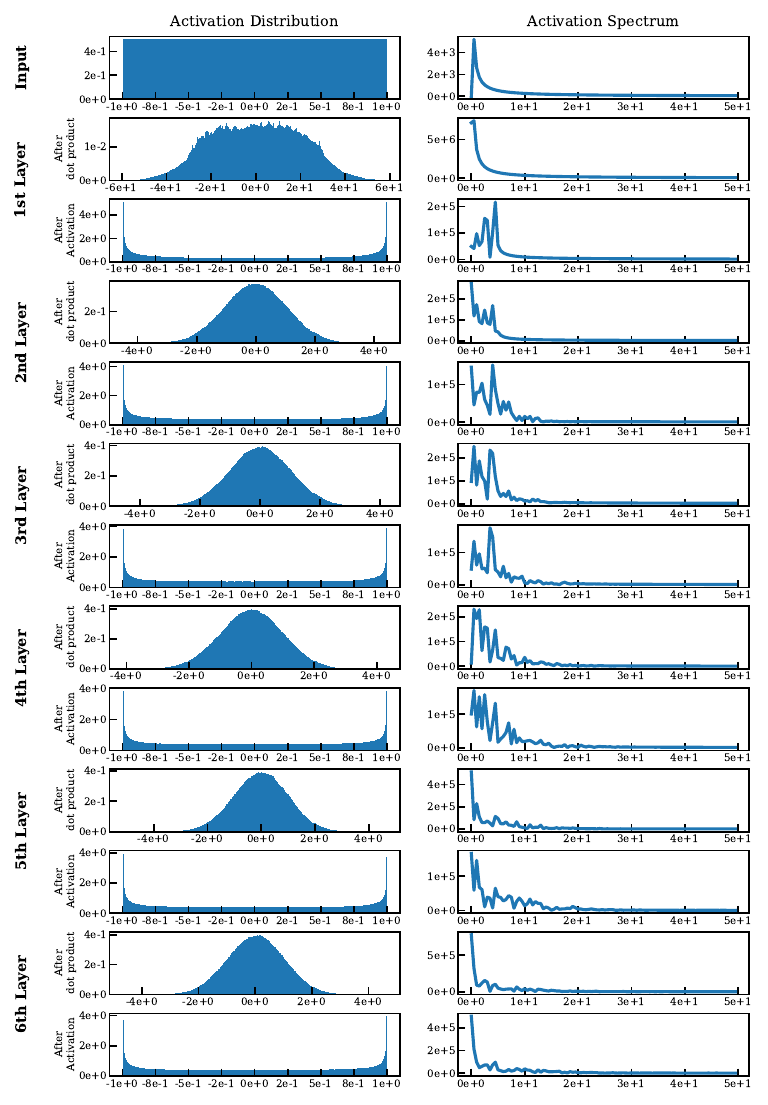}
            \label{fig:activations_siren}
        }
        \hfill
        \subfigure[STAF ($\tau=5$)]{
            \includegraphics[width=0.38\linewidth]{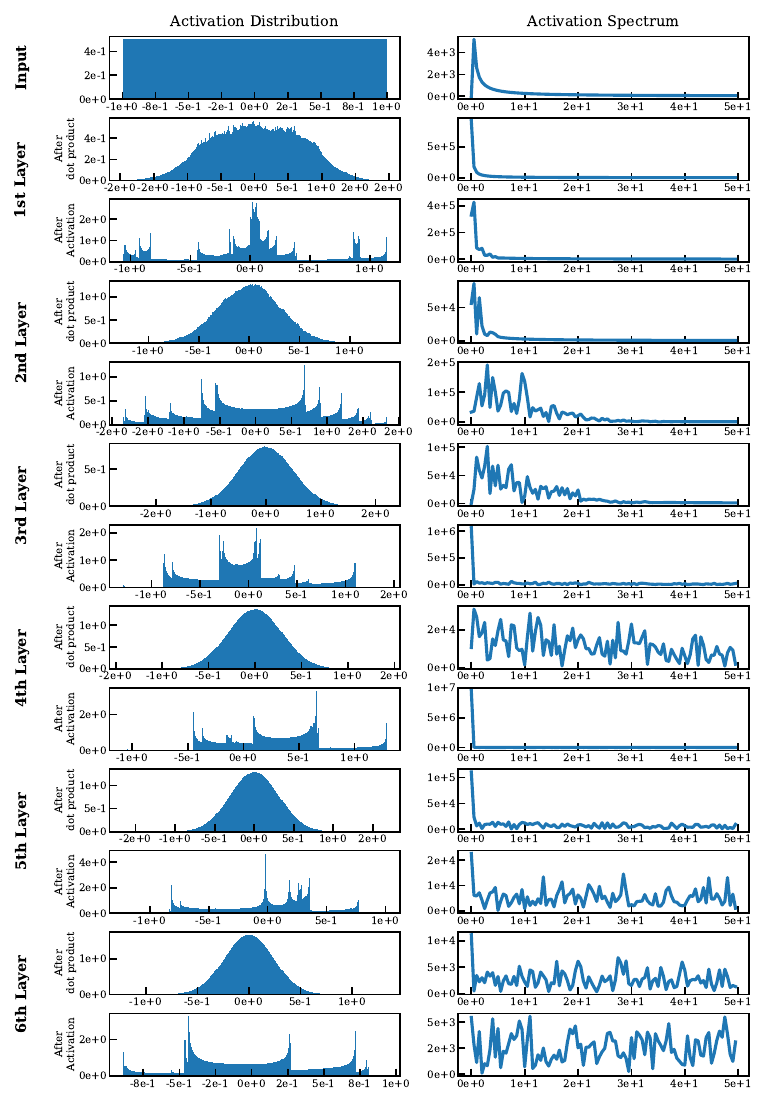}
            \label{fig:activations_staf2_ta_5}
        }

        \subfigure[STAF ($\tau=10$)]{
            \includegraphics[width=0.38\linewidth]{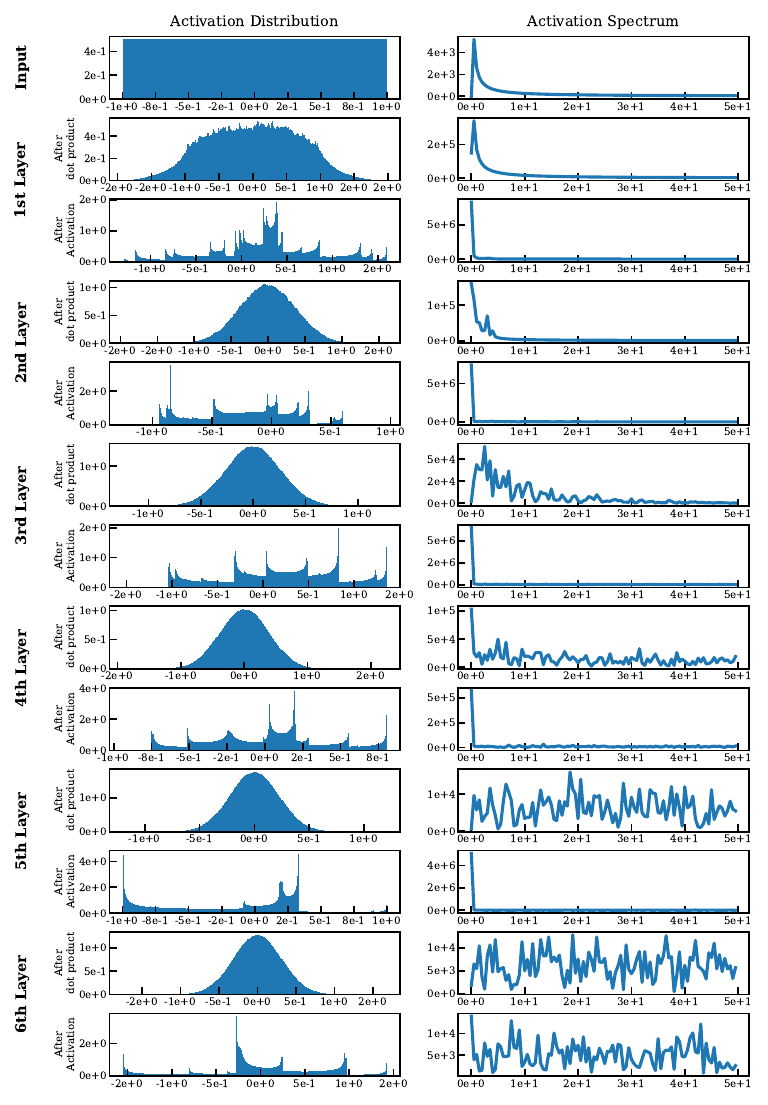}
            \label{fig:activations_staf2_ta_10}
        }
        \hfill
        \subfigure[STAF ($\tau=20$)]{
            \includegraphics[width=0.38\linewidth]{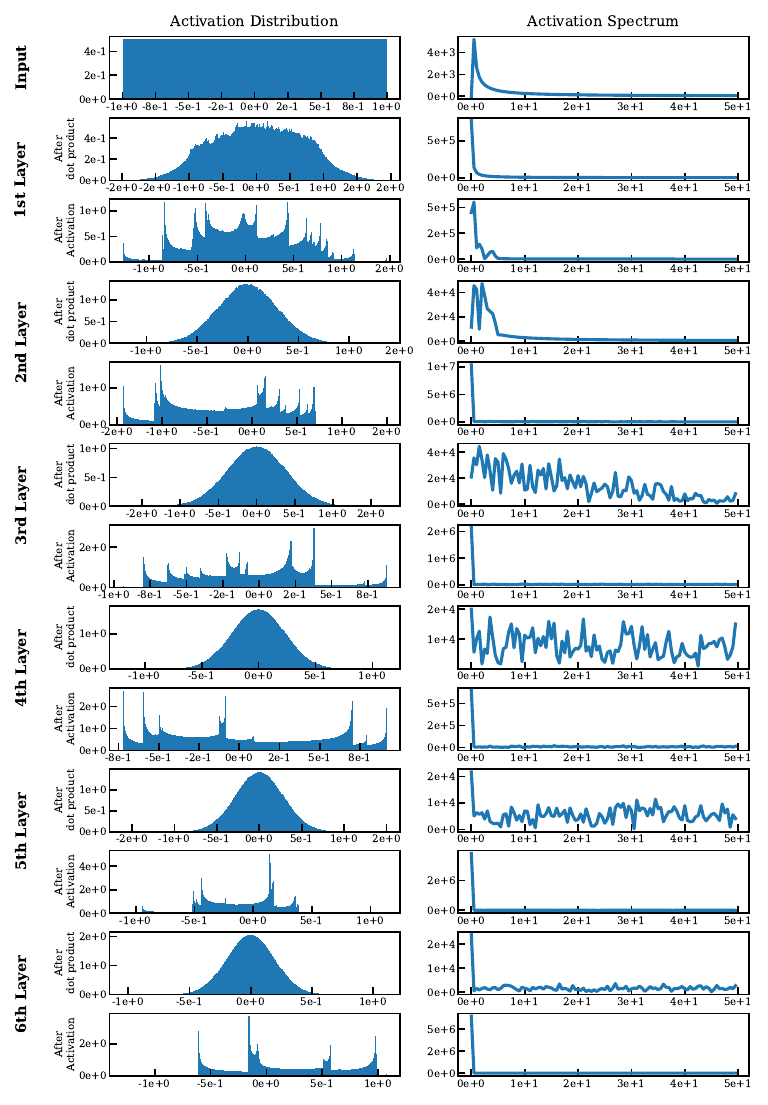}
            \label{fig:activations_staf2_ta_20}
        }
    \end{minipage}%
    }
    
    \caption{Comparison of activation distributions and frequency spectra between SIREN and STAF with $\tau$ values of 5, 10, and 20. Note that the scales may differ across plots; however, to appreciate how effectively STAF captures higher frequencies, observe the maximum frequency in each spectrum.}
    \label{fig:activation_2x2}
\end{figure}

\subsection{Performance Comparison of STAF and SIREN with Similar Parameter Counts}
\cref{fig:match_params} demonstrates the superior performance of STAF compared to SIREN in terms of PSNR (dB) across 250 epochs, despite SIREN having a higher parameter count for the Celtic image. To ensure a balanced evaluation, the default configuration of SIREN was modified by adding one additional layer, resulting in 264,193 parameters for SIREN compared to STAF's 213,761 parameters. This approach avoids extensive parameter tuning for SIREN, offering a practical comparison between the two models. The results clearly show that STAF consistently outperforms SIREN, achieving significantly higher PSNR values throughout the training process. This highlights STAF’s efficiency and effectiveness, even when constrained to a lower parameter count, making it a more suitable choice for tasks requiring high-quality image reconstruction.

\subsection{More Comparative Evaluations}
Figure \ref{fig:hash} presents a comparative analysis of three methods, STAF, SIREN, and Hash Encoding \citep{muller2022instant}. \added{It should be noted that since hash encodings and activation design address different modeling axes, this comparison should be viewed as complementary rather than a direct activation-only benchmark.} The PSNR curves indicate that STAF outperforms both SIREN and Hash Encoding, reaching a PSNR of over 100 dB after 500 epochs. While Hash Encoding shows a notable improvement over SIREN, peaking at around 70 dB, it still falls short of STAF’s superior performance. SIREN, in contrast, exhibits the slowest PSNR growth, achieving only around 38 dB. The qualitative comparisons on the right further support these quantitative results, with STAF closely approximating the ground truth, while Hash Encoding and SIREN produce visibly lower-quality reconstructions. This analysis highlights the advantage of STAF in achieving both higher fidelity and faster convergence in image reconstruction tasks.

\begin{lstlisting}[
    caption={Parameter initialization of STAF in PyTorch.},
    label={lst:init_params},
    float=t,
    numbers=left,
    backgroundcolor=\color{codebg}
]
def init_params(self):
    """
    Initialize parameters for sinusoidal activations.

    - ws: Frequency parameters scaled by omega_0.
    - phis: Phase offsets sampled uniformly from [-pi, pi].
    - bs: Scale factors sampled from a Laplace distribution,
          with signed square root applied.
    """
    # Frequencies: random in [0, 1), scaled by omega_0
    ws = self.omega_0 * torch.rand(self.tau)
    self.ws = nn.Parameter(ws, requires_grad=True)

    # Phases: uniform in [-pi, pi]
    self.phis = nn.Parameter(
        -math.pi + 2 * math.pi * torch.rand(self.tau),
        requires_grad=True
    )

    # Scale factors: Laplace sampling with signed sqrt
    laplace_scale = 2 / self.tau
    laplace_samples = torch.distributions.Laplace(0, laplace_scale).sample((self.tau,))
    self.bs = nn.Parameter(
        torch.sign(laplace_samples) * torch.sqrt(torch.abs(laplace_samples)),
        requires_grad=True
    )
\end{lstlisting}

\begin{figure*}[!th]
    \centering
    \includegraphics[width=0.95\textwidth]{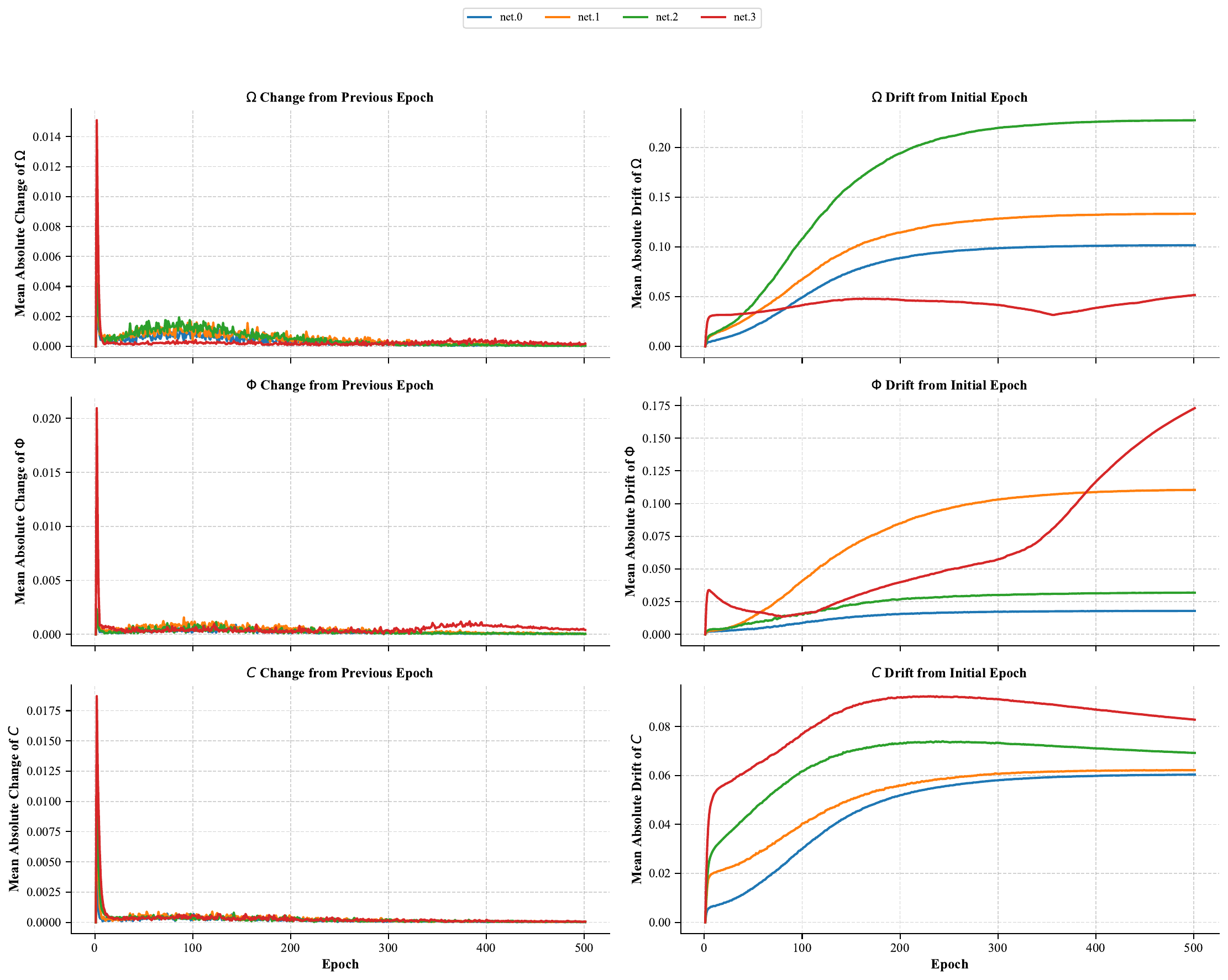}\\
\caption{\added{\textbf{Layer-wise dynamics of the trainable sinusoidal parameters in STAF during training.} The left column reports the mean absolute change from the previous epoch, and the right column reports the mean absolute drift from initialization, for $\Omega$ (top), $\Phi$ (middle), and $C$ (bottom). Different colors correspond to different MLP layers in the network, denoted by \texttt{net.0}--\texttt{net.3}. The updates are concentrated in the early stage of training and then decay toward near-zero values, while the cumulative drift grows early and later plateaus. These trends suggest that the sinusoidal parameters are learned primarily in an early phase and remain relatively stable afterward.}}
\label{fig:staf_dynamics_full}
\end{figure*}

\subsection{Empirical Validation of the Two-Phase NTK Approximation}
\label{sec:ntk_empirical_validation}

The analytic NTK in \Cref{sec:ntk} is derived under a two-phase approximation, where the sinusoidal parameters of STAF are assumed to adapt mainly in an early stage of training, after which the remaining optimization is largely governed by the MLP weights. We emphasize that this approximation is not intended as an exact description of the full optimization dynamics. Rather, it serves as a simplified regime that makes the NTK analysis tractable. To assess whether this approximation is reasonable in practice, we perform two empirical studies.

First, we track the dynamics of the trainable sinusoidal parameters, namely the frequencies $(\Omega)$, phases $(\Phi)$, and amplitudes $(C)$, throughout training. In \Cref{fig:staf_dynamics_full}, we report both the mean absolute change from the previous epoch and the mean absolute drift from initialization, measured separately for each layer. Across layers, the per-epoch updates are large only at the beginning of training and then rapidly decay toward near-zero values, while the cumulative drift from initialization grows early and gradually plateaus. This behavior indicates that the sinusoidal parameters are primarily adjusted in the early stage of training and become relatively stable afterward. While this does not imply that they stop changing entirely, it suggests that their most significant adaptation occurs early, which is consistent with the two-phase approximation used in our NTK derivation.

Second, we perform a freeze-after-warmup experiment on the image fitting task using the image in \cref{fig:celtic}. In the standard full-training setting over 500 iterations, STAF achieves {40.47} PSNR, {0.9804} SSIM, and {0.00577} LPIPS. In the freeze-after-warmup setting, we train normally for the first 200 steps, then freeze the sinusoidal parameters and continue optimizing only the remaining MLP weights. Under this setting, the model still reaches {39.96} PSNR, {0.9804} SSIM, and {0.00632} LPIPS. This corresponds to only a 0.51 dB drop in PSNR, a negligible {$2.8\times10^{-5}$} drop in SSIM, and a modest {$5.5\times10^{-4}$} increase in LPIPS. These results suggest that freezing the sinusoidal parameters after the early stage causes only a minor degradation in final performance, indicating that most of their useful adaptation occurs during warmup.

Taken together, these two experiments support the two-phase view as an approximation in our setting: the sinusoidal parameters adapt most strongly early in training, and the later stages of optimization can be carried largely by the MLP weights. Accordingly, the NTK analysis in \Cref{sec:ntk} should be interpreted as characterizing this approximate regime rather than the exact end-to-end dynamics of STAF over the entire course of training.

\begin{figure*}[!th]
    \centering
    \includegraphics[width=0.75\textwidth]{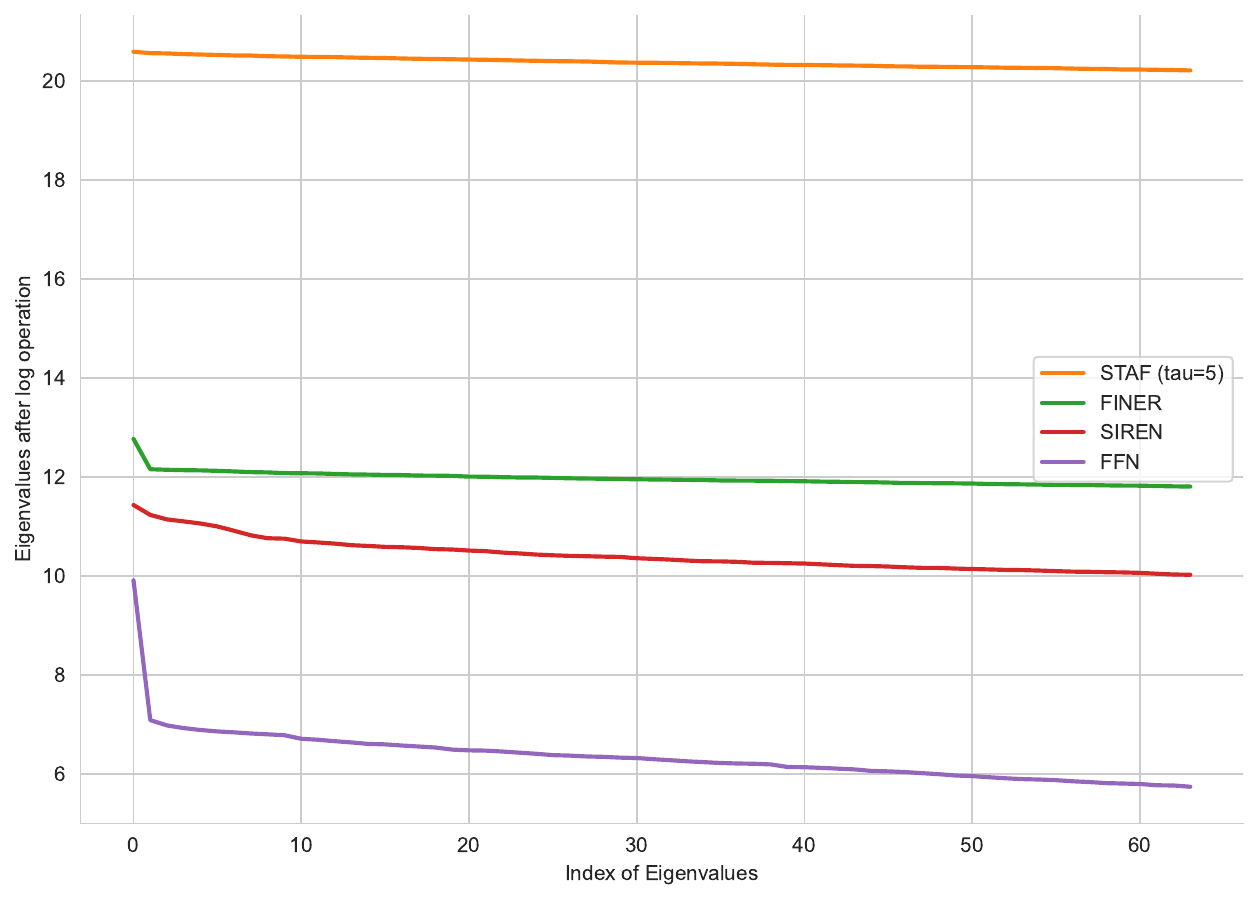}\\
    \vspace{-1em}
\caption{\added{\textbf{Empirical NTK eigenvalue distribution for the audio representation task.} We compare STAF ($\tau=5$), FINER, SIREN, and FFN using the same empirical NTK protocol as in the image-domain analysis. The plot reports the log-transformed NTK eigenvalues sorted by index. STAF exhibits a stronger spectrum, with larger eigenvalues across much of the distribution, indicating a richer kernel and more favorable optimization behavior for fitting oscillatory audio signals.}}
\label{fig:audio_ntk_eigenvalue_distribution}
\end{figure*}
\subsection{Empirical NTK Analysis on Audio}
\label{sec:appendix_audio_ntk}

We further examine the NTK eigenvalue distribution for the audio representation task. Using the same empirical NTK computation protocol as in~\Cref{sec:ntk}, we compare STAF ($\tau=5$) with FINER, SIREN, and FFN. The resulting spectra are shown in \Cref{fig:audio_ntk_eigenvalue_distribution}. Overall, STAF exhibits a stronger NTK spectrum on the audio task, with larger log-eigenvalues over a substantial portion of the spectrum relative to the baselines. This suggests that, in the audio setting as well, STAF induces a kernel with stronger signal components and a richer effective representation space. Consistent with the interpretation in the main text, such a spectrum is associated with more favorable optimization behavior, since components corresponding to larger kernel eigenvalues are typically learned more rapidly.

These results mirror the NTK observations in the image domain and further suggest that the benefits of trainable sinusoidal activations are not limited to visual signals. Rather, STAF also reshapes the NTK spectrum in 1D audio reconstruction, where accurately modeling oscillatory and mixed-frequency structure is essential. We do not interpret this as eliminating spectral bias; instead, it indicates that STAF improves the practical capacity--convergence trade-off by making a broader range of components easier to optimize.

\subsection{More Implementation Details}
We provide the \texttt{PyTorch} implementation of our initialization scheme in~\cref{lst:init_params} to better illustrate our approach. We set $W_0$ (or \texttt{omega\_0} in the code) to 30, and found that 40 also works well in practice. For the image representation task, we use a learning rate of $1\text{e}{-3}$. For shape representation, we use $2.5\text{e}{-4}$ with a marching cubes threshold of 0.5. In the audio task, we adopt a learning rate of $2.5\text{e}{-4}$, with $W_0 = 3000.0$ in the first layer and $W_0 = 30.0$ in the hidden layers. For denoising, we use a learning rate of $1.5\text{e}{-4}$ with $W_0 = 5$. Additionally, for high-resolution DIV2K images, we follow the approach in~\citep{kazerouni2025lift} and incorporate skip connections in the hidden layers.

\section{Proofs}
\addtocontents{toc}{\protect\setcounter{tocdepth}{2}}
\subsection{Distribution of $Y = \sin(aX + b)$ for $X \sim U(-1,1)$} \label{assessing_dist}
The authors of \citep{Siren} claim that regardless of the choice of $b$, if $a > \frac{\pi}{2}$, the output $Y$ follows an arcsine distribution, denoted as $Arcsine(-1,1)$.
\addtocontents{toc}{\protect\setcounter{tocdepth}{3}}
\subsubsection{Assessing SIREN's claim about the distribution of $Y=\sin(ax+b)$}
If SIREN's claim were true, the expected value of $Y$ would be independent of $b$.

\begin{flalign}
    \mathbb{E}[Y]&=\int_{-1}^{1} \sin(ax + b) f_X(x)\, dx = \frac{1}{2}\int_{-1}^{1} (\sin(ax) \cos b + \sin b \cos(ax)) \, dx && \nonumber \\
    &= \frac{1}{2a}\Big[\big(\cos(-a)-\cos(a)\big)\cos b + \big(\sin(a)-\sin(-a)\big)\sin b\Big] = \frac{\sin a \sin b}{a} &&
\end{flalign}
which obviously depends on $a$ and $b$. However, if we want to eliminate $b$ from $\mathbb{E}[Y]$, we can set
\begin{equation} \label{length_of_interval}
    a=n\pi,
\end{equation}
for an $n \in \mathbb{Z}\setminus{\{0\}}$. Next, let us consider the next moments of $Y$, because if the moment-generating function (MGF) of $Y$ exists, the moments can uniquely determine the distribution of $Y$.
\begin{flalign}
    \mathbb{E}[Y^k]=\int_{-1}^{1}\frac{1}{2}\sin^k(ax+b)dx
\end{flalign}
Using \eqref{length_of_interval}, it is equal to $\frac{1}{2n\pi}\int_{-1}^{1}\sin^k(ax+b)dx.$ By assuming $u=ax+b$, we have
\begin{equation}
    \mathbb{E}[Y^k]=\frac{1}{2n\pi}\int_{b-a}^{b-a+2n\pi}\sin^k(u)du.
\end{equation}
Since for each pair of natural numbers $(k,n)$, $2n\pi$ is a period of $\sin^k(u)$, we can write
\begin{equation}
    \mathbb{E}[Y^k]=\frac{1}{2n\pi}\int_{0}^{2\pi}\sin^k(u)du =\begin{cases}
			0, & \text{if $k$ is odd}\\
            \frac{\binom{k}{k/2}}{2^{k}n}, & \text{if $k$ is even}
		 \end{cases}
\end{equation}
If $a = n\pi$, then all moments of $Y$ will be independent of $a$ and $b$. Consequently, one may expect that the distribution of $Y$ is also independent of $a$ and $b$. However, this only occurs when 
$a = n\pi.$ Therefore, the assumption that the distribution of $Y$ is independent of the parameter $b$ (as used in \citep{Siren}) holds only for specific values of $a$.

We next calculate the exact distribution of $Y = \sin(aX+b)$, where $X \sim U(-1,1)$.
\addtocontents{toc}{\protect\setcounter{tocdepth}{3}}
\subsubsection{Finding the Distribution of $Y$} \label{distribution_of_Y}

To obtain the distribution of the random variable
\[
Y=\sin(aX+b)
\]
we use the change-of-variables (Jacobian) method. First note that the probability density function of $X$ is
\begin{equation}
    f_X(x) =
    \begin{cases}
        \dfrac{1}{2}, & -1 \le x \le 1,\\[6pt]
        0, & \text{otherwise}.
    \end{cases}
\end{equation}

Consider the transformation
\begin{equation}
    Y = g(X) = \sin(aX+b).
\end{equation}

To obtain the density of $Y$ we need to find all possible inverses of $g$. If $y=\sin(ax+b)$ then
\begin{equation}
    ax + b = \arcsin(y) + 2k\pi \quad\text{or}\quad ax + b = \pi - \arcsin(y) + 2k\pi,\qquad k\in\mathbb{Z}.
\end{equation}
Hence
\begin{equation}
    x = \frac{\arcsin(y) - b + 2k\pi}{a}
    \quad\text{or}\quad
    x = \frac{\pi - \arcsin(y) - b + 2k\pi}{a}.
\end{equation}

The density of $Y$ is given by
\begin{equation}
    f_Y(y) = \sum_{x\in g^{-1}(y)} f_X(x)\,\left|\frac{dx}{dy}\right|.
\end{equation}
Since
\begin{equation}
    \frac{dy}{dx} = a\cos(ax+b),
\end{equation}
we have
\begin{equation}
    \frac{dx}{dy} = \frac{1}{a\cos(ax+b)},
\end{equation}
and therefore
\begin{equation}
    f_Y(y) = \frac{1}{2}\sum_{x\in g^{-1}(y)\cap[-1,1]} \frac{1}{\big|a\cos(ax+b)\big|}.
\end{equation}

Because at any root $x$ of $\sin(ax+b)=y$ we have $\cos(ax+b)=\pm\sqrt{1-y^2}$, each summand is independent of the particular root and depends only on $y$. Thus
\begin{equation}
    f_Y(y)=\frac{1}{2}\sum_{x\in g^{-1}(y)\cap[-1,1]} \frac{1}{|a|\sqrt{1-y^2}}.
\end{equation}
Let $N(y)$ denote the number of solutions $x\in[-1,1]$ of the equation $\sin(ax+b)=y$. Then
\[
\boxed{%
    f_Y(y)=
    \begin{cases}
        \dfrac{N(y)}{2|a|\sqrt{1-y^2}}, & y\in\mathcal{I}:=\{\sin t:\ t\in[b-a,b+a]\},\\[8pt]
        0, & \text{otherwise},
    \end{cases}
}
\]
where $\mathcal{I}$ is the image of the interval $[-1,1]$ under the map $t=ax+b\mapsto\sin t$ (note $\mathcal{I}\subseteq[-1,1]$).

One can also write $N(y)$ explicitly as the count of integer indices $k$ giving roots in $[-1,1]$:
\begin{equation}
    N(y)=\#\Big\{k\in\mathbb{Z}:\; \frac{\arcsin y + 2k\pi - b}{a}\in[-1,1]\Big\}
    \;+\;
    \#\Big\{k\in\mathbb{Z}:\; \frac{\pi-\arcsin y + 2k\pi - b}{a}\in[-1,1]\Big\}.
\end{equation}

In practice, to determine $N(y)$ one typically considers the interval $[b-a,b+a]$ and, for each $y$, counts the number of points $t$ in that interval such that $\sin t = y$. For example:
\begin{itemize}
    \item If the length of the interval $b+a-(b-a)=2|a|$ is less than $\pi$ and the interval is positioned so that $\cos t$ does not change sign there, then $\sin t$ is monotone on that interval and hence $N(y)=1$. In this case the density simplifies to
    \begin{equation}
        f_Y(y)=\frac{1}{2|a|\sqrt{1-y^2}},\qquad
        y\in\big[\min_{t\in[b-a,b+a]}\sin t,\ \max_{t\in[b-a,b+a]}\sin t\big].
    \end{equation}
    \item If the interval $[b-a,b+a]$ contains more than one oscillation of the sine (for instance if its length exceeds $2\pi$, or it contains several peaks and troughs), then $N(y)$ can be greater than $1$. Inside the interior of the support, $N(y)$ is typically constant and changes in integer steps (blockwise).
\end{itemize}

Let $m$ and $M$ denote the minimum and maximum of $\sin t$ on $[b-a,b+a]$, respectively. In the special case $\tfrac{\pi}{2}\le a<\pi$, by the first case above we obtain
\begin{equation}
    f_Y(y)=\frac{1}{2a\sqrt{1-y^2}},\qquad y\in[m,M].
\end{equation}
Consequently, the cumulative distribution function of $Y$ is
\begin{equation}
    F_Y(y)=\int_{-\infty}^{y}\frac{dt}{2a\sqrt{1-t^2}}
    =
    \begin{cases}
        0, & y < m,\\[6pt]
        \dfrac{1}{2a}\big(\arcsin(y)-\arcsin(m)\big), & m \le y \le M,\\[10pt]
        1, & y > M.
    \end{cases}
\end{equation}

As is evident, since the distribution of $Y$ depends on $m$ and $M$, it depends on $b$ even in this case. Therefore, the assumption in \cite{Siren} that when $a>\frac{\pi}{2}$ the distribution of $Y$ is independent of $b$ is not correct.

\subsection{Proof of Theorem \eqref{initialization_theorem}}
\label{app:proof1}
In this section, we provide a step-by-step proof of Theorem \eqref{initialization_theorem} concerning the initialization scheme of an architecture that leverages STAF.
\begin{theorem} \label{main_th}
Consider the function $Z$ defined as
\begin{equation} \label{definition_of_h}
    Z = \sum_{u=1}^{\tau} C_u \sin\left(\Gamma_u + \Phi_u\right).
\end{equation}
Suppose that the random variables $C_u$ have symmetric distributions and finite moments. Moreover, for each $u$, assume that $\Phi_u \sim U(-\pi,\pi)$. In addition, assume the following independence conditions:
\begin{itemize}
    \item the variables $C_i$ are mutually independent;
    \item for each $i$, the variable $C_i$ is independent of $(\Gamma_i, \Phi_i)$; and
    \item the collections $\{(C_i, \Gamma_i, \Phi_i)\}_{i=1}^{\tau}$ are mutually independent.
\end{itemize}
Then, the moments of $Z$ depend only on $\tau$ and the moments of the $C_u$'s, and all odd-order moments of $Z$ will be zero.
\end{theorem}

\added{The apparent independence of the distribution of the post-activations from $\Gamma_i = \Omega_i \, \boldsymbol{w} \cdot \boldsymbol{x}$ may seem unintuitive. This follows from a simple observation: if $\Phi \sim U(-\pi,\pi)$, then for any fixed $\gamma$, the random variables $\Phi$ and $\Phi + \gamma \ (\mathrm{mod}\ 2\pi)$ are identically distributed, and hence so are $\sin(\Phi)$ and $\sin(\Phi+\gamma)$.\newline We now provide a more detailed justification using the moments of $Z$, which will also be used later.}

\begin{proof}
    For convenience, let us consider $\Gamma_u=\Omega_u\boldsymbol{w}.\boldsymbol{x}$. Based on the multinomial theorem, for every natural number $q$, we have:

\[
Z^q = \sum_{\substack{i_1+\ldots+i_{\tau}=q \\ i_1, \ldots, i_{\tau}\geq 0}} \left[\binom{q}{i_1, \ldots , i_{\tau}}\prod_{u=1}^{{\tau}} \left(C_u \sin(\Gamma_u+\Phi_u)\right)^{i_u}\right].
\]

According to the linearity of expected value:
\begin{align}
    \mathbb{E}[Z^q] &= \sum_{\substack{i_1+\ldots+i_{\tau}=q \\ i_1, \ldots, i_{\tau}\geq 0}} \left[\binom{q}{i_1, \ldots, i_{\tau}}\mathbb{E}\left[\prod_{u=1}^{\tau} \left(C_u \sin(\Gamma_u+\Phi_u)\right)^{i_u}\right]\right]
\end{align}

Given the independence assumptions stated above:
\begin{equation}  \label{h_powered_by_q}
    \mathbb{E}[Z^q] = \sum_{\substack{i_1+\ldots+i_{\tau}=q \\ i_1, \ldots, i_{\tau}\geq 0}} \left[\binom{q}{i_1, \ldots, i_{\tau}}\prod_{u=1}^{\tau} \left[\mathbb{E}[C_u^{i_u}] \mathbb{E}\left[\sin^{i_u}(\Gamma_u+\Phi_u)\right]\right]\right].
\end{equation}

Each choice of $i_1, \ldots, i_{\tau}$ is called a partition for $q$. If $q$ is an odd number, then in each partition of $q$, at least one of the variables, such as $i_k$, is odd. Since the function $C_i$ is symmetric, it follows that $\mathbb{E}[C_k^{i_k}] = 0$. This is because odd-order moments of a symmetric distribution are always zero. Consequently, the expectation $\mathbb{E}\left[\prod_{u=1}^{\tau} \left(C_u \sin(\Gamma_u+\Phi_u)\right)^{i_u}\right]$ also equals zero, as does $\mathbb{E}[Z^q]$.

Now, let us consider the case when $q$ is even. For each partition of $q$, if at least one of its variables is odd, then, as before, we have $\mathbb{E}\left[\prod_{u=1}^{\tau} \left(C_u \sin(\Gamma_u+\Phi_u)\right)^{i_u}\right]=0$. Thus, we can express $q$ as $q=2j_1+ \ldots+2j_{\tau}$ where each $j_k$ is a non-negative integer. According to \eqref{h_powered_by_q}, to obtain the $i_k$-th moment of $Z$, we need to calculate $\mathbb{E}\left[\sin^{i_u}(\Gamma_u+\Phi_u)\right]$. In this case, where $i_u=2j_u$, $\sin^{i_u}\theta$ is an even function, and its Fourier series consists of a constant term and some cosine terms, given by

\begin{equation}
    \sin^{2j_u}\theta = \alpha_0 + \sum_{r=1}^{\infty} \alpha_r \cos(r\theta).
\end{equation}

Hence,
\begin{flalign}
&\mathbb{E}[\sin^{2j_u}(\Gamma_u+\Phi_u)] = \mathbb{E}[\alpha_0+\sum_{r=1}^{\infty}\alpha_r\cos(r(\Gamma_u+\Phi_u))] = \alpha_0+\sum_{r=1}^{\infty}\alpha_r \mathbb{E}[\cos(r\Gamma_u+r\Phi_u)] \nonumber \\
&= \alpha_0+\sum_{r=1}^{\infty}\alpha_r \mathbb{E}[\cos(r\Gamma_u) \cos(r\Phi_u) - \sin(r\Gamma_u) \sin(r\Phi_u)] = \alpha_0+\sum_{r=1}^{\infty}\alpha_r \Xi
\end{flalign}
where
\begin{equation}
    \Xi = \mathbb{E}[\cos(r\Gamma_u)]\mathbb{E}[\cos(r\Phi_u)] - \mathbb{E}[\sin(r\Gamma_u)]\mathbb{E}[\sin(r\Phi_u)].
\end{equation}

Since $r$ is an integer, $r\Phi_u$ will be a period, resulting in $\mathbb{E}[\cos(r\Phi_u)] = \mathbb{E}[\sin(r\Phi_u)] = 0$. Thus, $\mathbb{E}[\sin^{2j_u}(\Gamma_u+\Phi_u)] = \alpha_0$.

Using the formula for the coefficients of the Fourier series, we have:
\begin{equation}
    \alpha_0 = \frac{1}{\pi} \int_{-\pi/2}^{\pi/2} \sin^{2j_u} \theta \, d\theta = \frac{2}{\pi} \int_0^{\pi/2} \sin^{2j_u} \theta \, d\theta 
    \label{Wallis}
    = \frac{2}{\pi} \times \frac{\binom{2j_u}{j_u}}{2^{2j_u}} \times \frac{\pi}{2}
    = \frac{\binom{2j_u}{j_u}}{2^{2j_u}}
\end{equation}


where \eqref{Wallis} is evaluated using the Wallis integral.

To summarize,
\begin{flalign} \label{semi_final}
&\mathbb{E}[Z^q] = \sum_{\substack{j_1 + \dots + j_\tau = \frac{q}{2}, \\ j_1, \dots, j_\tau \geq 0}} \binom{q}{2j_1, \dots, 2j_\tau} \prod_{u=1}^\tau \mathbb{E}[C_u^{2j_u}] \frac{{\binom{2j_u}{j_u}}}{2^{2j_u}} \nonumber \\
&= \sum_{\substack{j_1 + \dots + j_\tau = \frac{q}{2}, \\ j_1, \dots, j_\tau \geq 0}} \left[\left(\binom{q}{2j_1, \dots, 2j_\tau}\prod_{u=1}^\tau\binom{2j_u}{j_u}\right)\prod_{u=1}^\tau \frac{1}{2^{2j_u}} \prod_{u=1}^\tau \mathbb{E}[C_u^{2j_u}]\right]
\end{flalign}
This also accounts for odd-order moments, as it is impossible to select a combination of non-negative integers that sums to a non-integer value.

It is worth noting that:
\begin{flalign} \label{simplification}
    \binom{q}{2j_1, \dots, 2j_\tau} \prod_{u=1}^\tau \binom{2j_u}{j_u} &= \frac{q!}{(2j_1)! \dots (2j_\tau)!} \times \frac{(2j_1)!}{(j_1)!^2} \times \dots \times \frac{(2j_\tau)!}{(j_\tau)!^2} = \frac{q!}{(j_1!)^2 \dots (j_\tau!)^2} \nonumber \\
    &= \binom{q}{j_1, j_1, \dots, j_\tau, j_\tau}
\end{flalign}
Furthermore,
\begin{equation} \label{sum_of_powers}
    \prod_{u=1}^\tau \frac{1}{2^{2j_u}} = \frac{1}{2^{2 \sum_{u=1}^\tau j_u}} = \frac{1}{2^q}
\end{equation}
By utilizing Equations \eqref{semi_final} to \eqref{sum_of_powers}, we can conclude that:
\begin{equation} \label{E_h_powered_by_q}
    \mathbb{E}[Z^q] = \frac{1}{2^q}\sum_{\substack{j_1 + \dots + j_\tau = \frac{q}{2} \\ j_1, \dots, j_\tau \geq 0}} \binom{q}{j_1, j_1, \dots, j_\tau, j_\tau} \prod_{u=1}^\tau \mathbb{E}[C_u^{2j_u}]
\end{equation}
As you can see, the moments of $Z$ depend solely on $\tau$ and the moments of the $C_u$'s.
\end{proof}

Now, our goal is to determine the distribution of the $C_u$'s so that the distribution of $Z$ becomes $\mathcal{N}(0,1)$. To achieve this, let's first consider the following theorem:
\begin{theorem}
    (Page 353 of \citep{Shiryaev2016-xv}) Let $X \sim \mathcal{N}(0,\sigma^2)$. Then
    \begin{equation} \label{q'th_moments_of_X}
        E(X^q) = \begin{cases}
			0, & \text{if $q$ is odd}\\
            \frac{q!}{\frac{q}{2}!~2^{q/2}}\sigma^q, & \text{if $q$ is even}
		 \end{cases}
    \end{equation}
    and these moments pertain exclusively to the normal distribution.
\end{theorem}

In theorem \eqref{main_th}, we proved that for odd values of $q$, $\mathbb{E}[h^q]=0$. Thus, in order to have $Z \sim \mathcal{N}(0,1)$, for even values of $q$, we must have $\mathbb{E}[h^q] = \frac{q!}{\frac{q}{2}!~2^{q/2}}.$ Alternatively, we can express it as
\begin{equation}
    \frac{1}{2^q}\sum_{\substack{j_1 + \dots + j_\tau = \frac{q}{2} \\ j_1, \dots, j_\tau \geq 0}} \binom{q}{j_1, j_1, \dots, j_\tau, j_\tau} \prod_{u=1}^\tau \mathbb{E}[C_u^{2j_u}] = \frac{q!}{\frac{q}{2}!~2^{q/2}}.
\end{equation}
Simplifying further, we obtain
\begin{equation}
    \frac{q!}{2^q} \sum_{\substack{j_1+\dots+j_\tau=\frac{q}{2}\\j_1,\dots,j_\tau\geq 0}} \frac{\prod_{u=1}^\tau \mathbb{E}[C_u^{2j_u}]}{(j_1!)^2\dots(j_\tau!)^2} = \frac{q!}{\frac{q}{2}!~2^{q/2}}.
\end{equation}
This equation can be further simplified to
\begin{equation} \label{desired}
    \sum_{\substack{j_1+\dots+j_\tau=\frac{q}{2}\\j_1,\dots,j_\tau\geq 0}} \frac{\prod_{u=1}^\tau \mathbb{E}[C_u^{2j_u}]}{(j_1!)^2\dots(j_\tau!)^2} = \frac{2^{q/2}}{\frac{q}{2}!}.
\end{equation}
Equation \eqref{desired} provides a general formula that can be utilized in further research. It allows for finding different solutions for $C_u$ under various assumptions (e.g., independence or specific dependencies) and different values of $\tau$. However, in the subsequent analysis, we assume that $C_u$'s are independent and identically distributed (i.i.d) random variables. The following theorem aims to satisfy Equation \eqref{desired}.
\begin{theorem} \label{finding_the_moments}
    Suppose $C_u$'s are i.i.d random variables with the following even-order moments:
    \begin{equation} \label{moments_of_C}
        \mathbb{E}[C_u^{2j}]=\left(\frac{2}{\tau}\right)^j j!
    \end{equation}
    Then, for every non-negative even number $q$, Equation \eqref{desired} holds.\footnote{If you wonder how this solution struck our mind, you can start by solving equation \eqref{desired} for $q=2$ to obtain $\mathbb{E}[h^2]$. Then, using the value of $\mathbb{E}[h^2]$, solve \eqref{desired} for $q=4$ to obtain $\mathbb{E}[h^4]$, and so on. }
\end{theorem}
\begin{proof}
    We begin by simplifying the expression:
\begin{flalign}
&\sum_{\substack{j_1+\dots+j_\tau=\frac{q}{2}\\j_1,\dots,j_\tau\geq 0}} \frac{\prod_{u=1}^\tau \mathbb{E}[C_u^{2j_u}]}{(j_1!)^2\dots(j_\tau!)^2}=\sum_{\substack{j_1+\dots+j_\tau=\frac{q}{2}\\j_1,\dots,j_\tau\geq 0}} \frac{\prod_{u=1}^\tau \left[\left(\frac{2}{\tau}\right)^j j!\right]}{(j_1!)^2\dots(j_\tau!)^2} \nonumber\\
&=\sum_{\substack{j_1+\dots+j_\tau=\frac{q}{2}\\j_1,\dots,j_\tau\geq 0}} \left(\frac{2}{\tau}\right)^{\sum_{u=1}^\tau j_u}\left(\frac{1}{j_1!\dots j_\tau!}\right)=\sum_{\substack{j_1+\dots+j_\tau=\frac{q}{2}\\j_1,\dots,j_\tau\geq 0}} \left(\frac{2}{\tau}\right)^{\frac{q}{2}}\left(\frac{1}{j_1!\dots j_\tau!}\right) \nonumber\\
&=\left(\frac{2}{\tau}\right)^{\frac{q}{2}} \sum_{\substack{j_1+\dots+j_\tau=\frac{q}{2}\\j_1,\dots,j_\tau\geq 0}} \frac{1}{j_1!\dots j_\tau!}=\left(\frac{2}{\tau}\right)^{\frac{q}{2}} \frac{1}{(\frac{q}{2})!} \sum_{\substack{j_1+\dots+j_\tau=\frac{q}{2}\\j_1,\dots,j_\tau\geq 0}} \frac{(\frac{q}{2})!}{j_1!\dots j_\tau!} \nonumber\\
&=\left(\frac{2}{\tau}\right)^{\frac{q}{2}} \frac{1}{(\frac{q}{2})!} \sum_{\substack{j_1+\dots+j_\tau=\frac{q}{2}\\j_1,\dots,j_\tau\geq 0}} \binom{\frac{q}{2}}{j_1,\dots,j_\tau}
\end{flalign}
Based on the multinomial theorem, we can conclude that
\begin{equation}
    \left(\frac{2}{\tau}\right)^{\frac{q}{2}} \frac{1}{(\frac{q}{2})!} \sum_{\substack{j_1+\dots+j_\tau=\frac{q}{2}\\j_1,\dots,j_\tau\geq 0}} \binom{\frac{q}{2}}{j_1,\dots,j_\tau}=\left(\frac{2}{\tau}\right)^{\frac{q}{2}} \frac{\tau^\frac{q}{2}}{(\frac{q}{2})!}=\frac{2^{\frac{q}{2}}}{(\frac{q}{2})!}
\end{equation}
\end{proof}

Also note that according to Theorem \eqref{main_th}, the odd-order moments of $Z$ are zero, just like a normal distribution.

\begin{corollary}
    Let $Z$ be the random variable defined in \eqref{definition_of_h}. Additionally, assume that the $C_u$'s ($1 \leq u \leq \tau$) used in the definition of $Z$, are i.i.d  random variables with even moments as defined in theorem \eqref{finding_the_moments}.
    Then $Z\sim \mathcal{N}(0,1)$.
\end{corollary}
\begin{proof}
    We know that if the MGF of a distribution exists, then the moments of that distribution can uniquely determine its PDF. That is, if $X$ and $Y$ are two distributions and for every natural number $k$, $E(X^k)=E(Y^k)$, then $X=Y$.

    In Theorem \eqref{finding_the_moments}, we observed that the moments of $Z$ are equal to the moments of a standard normal distribution. Since the MGF of this distribution exists, $Z\sim \mathcal{N}(0,1)$.
\end{proof}

Now, let's explore which distribution can produce the moments defined in \eqref{moments_of_C}. To have an inspiration, note that for a centered Laplace random variable $X$ with scale parameter $b$, we have the PDF of $X$ as
\begin{equation}
    f_X(x)=\frac{1}{2b}e^\frac{-|x|}{b}
\end{equation}
and the moments of $X$ given by
\begin{equation}
    \mathbb{E}[X^q] = \begin{cases}
			0, & \text{if $q$ is odd}\\
            \frac{b^q}{q!}, & \text{if $q$ is even}
		 \end{cases}
\end{equation}
Hence, the answer might be similar to this distribution. If we assume $Y=sgn(X)\sqrt{|X|}$, since $Y$ is symmetric, all of its odd-order moments are zero. Now, let us calculate its even-order moments:
\begin{equation}
    \mathbb{E}[Y^{2q}] = \mathbb{E}[|X|^q] = \int_{-\infty}^{\infty} |x|^q \frac{1}{2b} e^{-\frac{|x|}{b}} dx = 2\int_{0}^{\infty} |x|^q \frac{1}{2b} e^{-\frac{|x|}{b}} dx = \frac{1}{b} \int_{0}^{\infty} x^q e^{-\frac{x}{b}} dx
\end{equation}
By assuming $u=\frac{x}{b}$, we will have
\begin{equation} \label{even_moments_of_Y}
    \mathbb{E}[Y^{2q}] = \int_{0}^{\infty} (bu)^q e^{-u} du = b^q \int_{0}^{\infty} u^q e^{-u} du = b^q \Gamma(q+1) = b^q q!
\end{equation}
By assuming $b=\frac{2}{\tau}$, \eqref{moments_of_C} will be obtained.

The next theorem will obtain the probability distribution function of $Y$.

\begin{theorem} \label{final_pdf}
    Let $X$ be a centered Laplace random variable with scale parameter $b$, and $Y=sgn(X)\sqrt{|X|}$. Then
    \begin{equation}
        f_Y(y)=\frac{|y|}{b}e^\frac{-y^2}{b}
    \end{equation}
\end{theorem}
\begin{proof}
        Let $A=Y^2=|X|$. Therefore,
    \begin{equation}
        M_{A}(t)=\sum_{k=0}^{\infty}\frac{t^k\mathbb{E}[|X|^k]}{k!}
    \end{equation}
    As we calculated in \eqref{even_moments_of_Y},
    $\mathbb{E}[|X|^k]=b^k~k!$. Therefore,
    \begin{equation}
        M_{A}(t)=\sum_{k=0}^{\infty}\frac{t^k \cdot b^k~k}{k!}=\sum_{k=0}^{\infty}(bt)^k=\frac{1}{1-bt}=\frac{\frac{1}{b}}{\frac{1}{b}-t}
    \end{equation}
    that is the MGF of exponential distribution with parameter $\frac{1}{b}$. That is,
    \begin{equation}
        f_A(a)=\frac{1}{b}e^{\frac{-a}{b}}
    \end{equation}
    Therefore, using the fact that $A$ is a always non-negative, we consider non-negative values $a^2$ to describe its cumulative distribution function.
    \begin{equation} \label{exp_cdf}
        F_A(y^2)=\mathbb{P}(A\leq y^2)=1-e^{\frac{-y^2}{b}}
    \end{equation}
    On the other hand, if $y \geq 0$,
    \begin{equation}
       \mathbb{P}(A\leq y^2)=\mathbb{P}(Y^2 \leq y^2) = \mathbb{P}(-y \leq Y \leq y)
    \end{equation}
    Since we want $Y$ to be symmetric, we assume\footnote{In fact, the assumption that $Y$ is symmetric is not unexpected, since all odd-order moments of $Y$ are zero. But there are some non-symmetrical distributions whose all odd-order moments are zero \citep{Churchill1946-in}. Nevertheless, under some assumptions, it can be shown that a distribution is symmetric if and only if all its odd-order moments are zero. However, we don't use this claim in this paper.}
    \begin{flalign} \label{symmetry_assumption}
        &\mathbb{P}(-y \leq Y \leq y) = 2~\mathbb{P}(0 \leq Y \leq y) =2~(\mathbb{P}(Y \leq y) - \frac{1}{2}) = 2F_Y(y) - 1, \quad y\geq 0
  \end{flalign}
  Using equations \eqref{exp_cdf} to \eqref{symmetry_assumption}, we draw conclusion that
  \begin{equation} \label{CDF_equality}
      2F_Y(y) - 1 = 1-e^{\frac{-y^2}{b}},\quad y\geq 0
  \end{equation}
  By differentiating both sides of \eqref{CDF_equality} with respect to $y$, we will have
  \begin{equation}
      2f_Y(y) = \frac{2y}{b}e^{\frac{-y^2}{b}},\quad y\geq 0
  \end{equation}
  Therefore,
  \begin{equation}
      f_Y(y) = \frac{y}{b}e^{\frac{-y^2}{b}},\quad y\geq 0
  \end{equation}
  Since we assumed $y \geq 0$ in the above equations, and we supposed that $Y$ is symmetric,
  \begin{equation}
      f_Y(y) = \frac{|y|}{b}e^{\frac{-y^2}{b}},\quad y\in\mathbb{R}
  \end{equation}
  Just to make sure that our assumption about the symmetry of $Y$ was correct (or sufficed for our purpose), let us check the even-order moments of $Y$. The odd-orders are zero based on the symmetry.
  \begin{equation}
      \mathbb{E}[Y^{2k}] = \int_{-\infty}^{\infty} y^{2k} \left(\frac{|y|}{b} e^{-\frac{y^2}{b}}\right) dy = \frac{2}{b} \int_{0}^{\infty} y^{2k+1} e^{-\frac{y^2}{b}} dy
  \end{equation}
  Setting $y^2=t$ and $\frac{1}{b}=s$, leads to the following equation:
  \begin{equation}
      \mathbb{E}[Y^{2k}] = \frac{1}{b} \int_{0}^{\infty} t^k e^{-st} dt
  \end{equation}
  That is the Laplace transform of $t^k$. Therefore,
  \begin{equation}
      \mathbb{E}[Y^{2k}] = s\frac{\Gamma(k+1)}{s^{k+1}} = \frac{k!}{s^k} = b^k k!
  \end{equation}
\end{proof}

In summary, in this section we calculated the initial coefficients of our activation function as described in Theorem \eqref{final_pdf}, where we set $b=\frac{2}{\tau}$. Consequently, if we denote the post-activation of layer $l$ by $\boldsymbol{z^{(l)}}$, we will have $z^{(l)}_i\sim\mathcal{N}(0,1)$ for all $l\in\{2,3,\ldots,L-1\}$, and $i\in\{1,\ldots F_l\}$. This result can be proved by induction on $l$, using the fact that, based on the theorems in this section, the PDF of $Z$ is independent of the PDF of $\boldsymbol{x}$.




\subsection{Proof of Theorem \eqref{Kronecker_theorem}}
\label{app:proof_Kronecker_theorem}
Before proving the theorem, note the following remark:

\begin{remark}
    Let $X$ be a $\chi_1 \times \chi_2$ matrix, and $Y$ be a $\gamma_1 \times \gamma_2$ matrix. Then, according to \citep{Ashendorf2014-Supplement, Albrecht2023ProductKA}:
\begin{equation} \label{Kronecker_product_indices}
    (X\otimes Y)_{i,j} = x_{\lceil i/\gamma_1\rceil,~\lceil i/\gamma_2\rceil}~y_{(i-1)\%\gamma_1+1,~(j-1)\%\gamma_2+1}.
\end{equation}
\end{remark}

Now, let us consider each pair of layers as a block, where the first two layers form the first block, the second two layers form the second block, and so on. We prove the theorem by induction on the block numbers. The proof consists of three parts:

\textbf{Part 1)} Consider the weight matrix and bias vector given by:
\begin{equation}
	\boldsymbol{\overline{W^{(l)}}}=\boldsymbol{\Omega}\otimes \boldsymbol{W^{(l)}},\quad\quad \boldsymbol{\overline{B^{(l)}}} = \boldsymbol{\Phi}\otimes \boldsymbol{J}_{F_l,1}.
\end{equation}
We then define
\begin{equation}
	\begin{bmatrix}
		\overline{a^{(l)}_1} & \overline{a^{(l)}_2} & \ldots & \overline{a^{(l)}_{\tau F_l}}
	\end{bmatrix}^{tr}
	= \boldsymbol{\overline{W^{(l)}}}\boldsymbol{z^{(l-1)}}+\boldsymbol{\overline{B^{(l)}}},
\end{equation}
and
\begin{equation}
	\overline{z^{(l)}_p} = \rho(\overline{a^{(l)}_p})\quad\forall ~p\in\{1,2,\ldots,\tau F_l\}.
\end{equation}
Additionally, define
\begin{equation}  \label{next_layer_preactivation_1}
    \Tilde{a}^{(l+1)} = \left(\boldsymbol{C}^{tr}\otimes W^{(l+1)}_{i,:}\right)\overline{z^{(l)}},
\end{equation}
where $W^{(l+1)}_{i,:}$ denotes the $i$'th row of $W^{(l+1)}$. Then, we can observe that
\begin{equation} \label{next_layer_preactivation_2}
	\Tilde{a}^{(l+1)} = a^{(l+1)}_i
\end{equation}

\begin{proof}
	First, let us calculate $a^{(l+1)}_{i}$ using activation function $\rho^*$. Note that $a^{(l+1)}=\boldsymbol{W^{(l+1)}}\boldsymbol{z^{(l)}}$. Therefore, $a^{(l+1)}_i=\boldsymbol{W^{(l+1)}}_{i,:}\boldsymbol{z^{(l)}}$. It implies that
	\begin{align} \label{a^{(l+1)}_i}
		a^{(l+1)}_{i} &=\sum_{j=1}^{F_l}W^{(l+1)}_{i,j}z^{(l)}_j =\sum_{j=1}^{F_l}W^{(l+1)}_{i,j}\rho^*\left(a^{(l)}_j\right)=\sum_{j=1}^{F_l}W^{(l+1)}_{i,j}\rho^*\left(\sum_{k=1}^{F_{l-1}}W^{(l)}_{j,k}z^{(l-1)}_k\right) \nonumber \\
		&=\sum_{j=1}^{F_l}W^{(l+1)}_{i,j}\sum_{m=1}^{\tau}\boldsymbol{C}_m\rho\left(\boldsymbol{\Omega}_m\sum_{k=1}^{F_{l-1}}W^{(l)}_{j,k}z^{(l-1)}_k+\boldsymbol{\Phi}_m\right)
	\end{align}
	
	Next, let us calculate $\Tilde{a}^{(l+1)}$. We have
	\begin{align}
		\overline{a^{(l)}_p}&=\left[\overline{W^{(l)}}z^{(l-1)}+\boldsymbol{\overline{B^{(l)}}}\right]_p = \boldsymbol{\overline{W^{(l)}}}_{p,:}\boldsymbol{z^{(l-1)}}+\boldsymbol{\overline{B^{(l)}}}_p = \sum_{k=1}^{F_{l-1}}\left(\boldsymbol{\overline{W^{(l)}}}_{p,k}\boldsymbol{z^{(l-1)}}_k\right) +\boldsymbol{\overline{B^{(l)}}}_p \nonumber \\
  &= \sum_{k=1}^{F_{l-1}}\left(\boldsymbol{\Omega}_{\lceil p/F_l\rceil,\lceil k/F_{l-1}\rceil}W^{(l)}_{1+(p-1)\%F_l,1+(k-1)\%F_{l-1}}z^{(l-1)}_k\right)+\boldsymbol{\Phi}_{\lceil p/F_l\rceil} \label{ceil_usage}
	\end{align}
    Equation \eqref{ceil_usage} is based on equation \eqref{Kronecker_product_indices}. Since $1\leq k\leq F_{l-1}$, it follows that $\lceil k/F_{l-1}\rceil=1$ and $(k-1)\%F_{l-1}=k-1$. As a result,
 \begin{align}
     \overline{a^{(l)}_p}&=\sum_{k=1}^{F_{l-1}}\left(\boldsymbol{\Omega}_{\lceil p/F_l\rceil}W^{(l)}_{1+(p-1)\%F_l,k}z^{(l-1)}_k\right)+\boldsymbol{\Phi}_{\lceil p/F_l\rceil} \nonumber \\
     &=\boldsymbol{\Omega}_{\lceil p/F_l\rceil}\sum_{k=1}^{F_{l-1}}\left(W^{(l)}_{1+(p-1)\%F_l,k}z^{(l-1)}_k\right)+\boldsymbol{\Phi}_{\lceil p/F_l\rceil}
 \end{align}
    Therefore,
	\begin{equation} \label{z_p^{(l)}}
    \overline{z^{(l)}_p}=\rho\left(\boldsymbol{\Omega}_{\lceil p/F_l\rceil}\sum_{k=1}^{F_{l-1}}\left(W^{(l)}_{1+(p-1)\%F_l,k}z^{(l-1)}_k\right)+\boldsymbol{\Phi}_{\lceil p/F_l\rceil}\right)
  \end{equation}
	Consequently,
	\begin{align}
		\Tilde{a}^{(l+1)} &= \sum_{p=1}^{\tau F_l}\left[\boldsymbol{C}^{tr}\otimes W^{(l+1)}_{i,:}\right]_{1,p}\overline{z^{(l)}_p} = \sum_{p=1}^{\tau F_l}\boldsymbol{C}^{tr}_{1,\lceil p/F_l\rceil}W^{(l+1)}_{i,1+(p-1)\%F_l} \overline{z^{(l)}_p} \nonumber \\
        &= \sum_{p=1}^{\tau F_l}\boldsymbol{C}_{\lceil p/F_l\rceil}W^{(l+1)}_{i,1+(p-1)\%F_l} \overline{z^{(l)}_p} \label{single_sum} \\ &=\sum_{j=1}^{F_l}\sum_{m=1}^{\tau}\boldsymbol{W}^{(l+1)}_{i,j}\boldsymbol{C}_m\overline{z^{(1)}_{F_l(m-1)+j}} \label{double_sum}
	\end{align}
	Equation \eqref{double_sum} is obtained as follows: by changing the indices of $\boldsymbol{W}$ and $\boldsymbol{C}$ from equation \eqref{single_sum} to \eqref{double_sum}, we need to change the index of $z^{(l)}$ too. To this end, note that
	\begin{equation} \label{changing_vars}
		m=\lceil p/F_l\rceil, \quad j=1+(p-1)\%F_{l}
	\end{equation}
	If $F_{l}\nmid p$, then $m=1+\lfloor p/F_l \rfloor$. As we know, $p=F_l\lfloor p/F_l \rfloor+p\%F_l$. Therefore, $p=F_l(m-1)+j$. This equation also holds when $F_l\mid p$.
	
	Equation \eqref{double_sum} can be rewritten as follows:
	\begin{align}
		&\sum_{j=1}^{F_l}\boldsymbol{W}^{(l+1)}_{i,j}\sum_{m=1}^{\tau}\boldsymbol{C}_m \overline{z^{(l)}_{F_l(m-1)+j}}
	\end{align}
    where, according to equations \eqref{z_p^{(l)}} and \eqref{changing_vars},
    \begin{equation}
        \overline{z^{(l)}_{F_l(m-1)+j}}=\rho\left(\boldsymbol{\Omega}_m\sum_{k=1}^{F_{l-1}}\left(W^{(l)}_{j,k}z^{(l-1)}_k\right)+\boldsymbol{\Phi}_m\right)
    \end{equation}
    Hence,
    \begin{align}
        \Tilde{a}^{(l+1)}= \sum_{j=1}^{F_l}\boldsymbol{W}^{(l+1)}_{i,j}\sum_{m=1}^{\tau}\boldsymbol{C}_m \rho\left(\boldsymbol{\Omega}_m\sum_{k=1}^{F_{l-1}}\left(W^{(l)}_{j,k}z^{(l-1)}_k\right)+\boldsymbol{\Phi}_m\right)
    \end{align}
    which is equal to $a^{(l+1)}_i$ based on \eqref{a^{(l+1)}_i}.
\end{proof}

\textbf{Part 2)} Let $\boldsymbol{\overline{B^{(l+1)}}}=\boldsymbol{\Phi}\otimes \boldsymbol{J}_{F_{l+1},1}$. We can define $\overline{a^{(l+1)}}$ as follows:
\begin{equation} \label{vector_form}
	\begin{bmatrix}
		\overline{a^{(l+1)}_1} & \overline{a^{(l+1)}_2} & \ldots & \overline{a^{(l+1)}_{\tau(F_{l+1})}}
	\end{bmatrix}^{tr}
	= \boldsymbol{\Omega}\otimes\boldsymbol{a^{(l+1)}} + \boldsymbol{\overline{B^{(l+1)}}}.
\end{equation}

Therefore, using \Cref{next_layer_preactivation_1,next_layer_preactivation_2,vector_form}, we can write
\begin{equation}
	\overline{a^{(l+1)}} = \overline{W^{(l+1)}}~\overline{z^{(l)}}+\boldsymbol{\overline{B^{(l+1)}}}
\end{equation},
where
\begin{equation}
	\overline{W^{(l+1)}}=\boldsymbol{\Omega}\otimes\left(\boldsymbol{C}^{tr}\otimes W^{(l+1)}\right)=\left(\boldsymbol{\Omega}\otimes\boldsymbol{C}^{tr}\right)\otimes W^{(l+1)}.
\end{equation}

Moreover, if we define
\begin{equation}
	\overline{z^{(l+1)}_q} = \rho\left(\overline{a^{(l+1)}_q}\right)\quad\forall~q\in\{1,\ldots,\tau(F_{l+1})\},
\end{equation}
we can observe that
\begin{equation} \label{final}
    {\boldsymbol{z^{(l+1)}}}=\left(\boldsymbol{C}^{tr}\otimes\boldsymbol{I}_{F_{l+1}}\right)\overline{\boldsymbol{z^{(l+1)}}}.
\end{equation}

\begin{proof}
    We know that
    \begin{equation}
        z^{(l+1)}_i=\rho^*(a^{(l+1)}_i)=\sum_{n=1}^{\tau}\rho\left(\boldsymbol{\Omega}_n a^{(l+1)}_i +\boldsymbol{\Phi}_n\right).
    \end{equation}
    Now, let us calculate each entry of the RHS of Equation \eqref{final}
    \begin{equation} \label{ith_entry}
        \left[\left(\boldsymbol{C}^{tr}\otimes\boldsymbol{I}_{F_{l+1}}\right)\overline{\boldsymbol{z^{(l+1)}}}\right]_i = \left[\boldsymbol{C}^{tr}\otimes\boldsymbol{I}_{F_{l+1}}\right]_i\overline{\boldsymbol{z^{(l+1)}}} = \sum_{j=1}^{\tau F_{l+1}}\left(\boldsymbol{C}^{tr}\otimes\boldsymbol{I}_{F_{l+1}}\right)_{i,j}\overline{\boldsymbol{z^{(l+1)}}_j}.
    \end{equation}
    Hence, according to \eqref{Kronecker_product_indices}, we have
    \begin{equation}
        \left[\left(\boldsymbol{C}^{tr}\otimes\boldsymbol{I}_{F_{l+1}}\right)\overline{\boldsymbol{z^{(l+1)}}}\right]_i = \sum_{j=1}^{\tau F_{l+1}}\boldsymbol{C}^{tr}_{\lceil i/F_{l+1} \rceil,\lceil j/F_{l+1}\rceil}\delta_{1+(i-1)\%F_{l+1},1+(j-1)\%F_{l+1}}\overline{\boldsymbol{z^{(l+1)}}_j},
    \end{equation}
    in which $\delta$ refers to Kronecker delta. As a result,
    \begin{equation}
        \left[\left(\boldsymbol{C}^{tr}\otimes\boldsymbol{I}_{F_{l+1}}\right)\overline{\boldsymbol{z^{(l+1)}}}\right]_i = \sum_{j=1}^{\tau F_{l+1}}\boldsymbol{C}_{\lceil j/F_{l+1} \rceil,\lceil i/F_{l+1}\rceil}\delta_{1+(i-1)\%F_{l+1},1+(j-1)\%F_{l+1}}\overline{\boldsymbol{z^{(l+1)}}_j}
    \end{equation}
    Note that $1 \leq i\leq F_{l+1}$. Therefore, $\lceil i/F_{l+1}\rceil=1$, and $(i-1)\%F_{l+1}=i-1$. Hence,
    \begin{equation}
        \left[\left(\boldsymbol{C}^{tr}\otimes\boldsymbol{I}_{F_{l+1}}\right)\overline{\boldsymbol{z^{(l+1)}}}\right]_i = \sum_{j=1}^{\tau F_{l+1}}\boldsymbol{C}_{\lceil j/F_{l+1} \rceil}\delta_{i,1+(j-1)\%F_{l+1}}\overline{\boldsymbol{z^{(l+1)}}_j}.
    \end{equation}
    Also note that $\delta_{i,1+(j-1)\%F_{l+1}}$ is zero, except when $j=kF_{l+1}+i$, in which case $\delta_{i,1+(j-1)\%F_{l+1}}=1$. Thus,
    \begin{align} \label{RHS}
        &\left[\left(\boldsymbol{C}^{tr}\otimes\boldsymbol{I}_{F_{l+1}}\right)\overline{\boldsymbol{z^{(l+1)}}}\right]_i = \sum_{k=0}^{\tau-1}\boldsymbol{C}_{\lceil\left(kF_{l+1}+i\right)/F_{l+1} \rceil}\overline{\boldsymbol{z^{(l+1)}}_{kF_{l+1}+i}} = \sum_{k=0}^{\tau-1}\boldsymbol{C}_{k+\lceil i/F_{l+1} \rceil}\overline{\boldsymbol{z^{(l+1)}}_{kF_{l+1}+i}} \nonumber \\
        &= \sum_{k=0}^{\tau-1}\boldsymbol{C}_{k+1}\overline{\boldsymbol{z^{(l+1)}}_{kF_{l+1}+i}} = \sum_{n=1}^{\tau}\boldsymbol{C}_{n}\overline{\boldsymbol{z^{(l+1)}}_{(n-1)F_{l+1}+i}} = \sum_{n=1}^{\tau}\boldsymbol{C}_{n}\rho\left(\overline{\boldsymbol{a^{(l+1)}}_{(n-1)F_{l+1}+i}}\right).
    \end{align}

    Note that
    \begin{align}
        \overline{\boldsymbol{a^{(l+1)}}_{(n-1)F_{l+1}+i}} &=\boldsymbol{\Omega}_{\lceil\left((n-1)F_{l+1}+i\right)/F_{l+1}\rceil}\boldsymbol{a^{(l+1)}}_{1+\left((n-1)F_{l+1}+i-1\right)\%F_{l+1}}+\boldsymbol{\Phi}_{\lceil\left((n-1)F_{l+1}+i\right)/F_{l+1}\rceil} \nonumber\\
        &=\boldsymbol{\Omega}_{n-1+\lceil i/F_{l+1}\rceil}\boldsymbol{a^{(l+1)}}_{1+\left(i-1\right)\%F_{l+1}}+\boldsymbol{\Phi}_{n-1+\lceil i/F_{l+1}\rceil}
    \end{align}
    Since $\left\lceil\frac{i}{F_{l+1}}\right\rceil = 1$ and $(i-1)\%F_{l+1} = i-1$, we have
    \begin{equation} \label{overline{boldsymbol{a^{(l+1)}}_{(n-1)F_{l+1}+i}}}
        \overline{\boldsymbol{a^{(l+1)}}_{(n-1)F_{l+1}+i}}=\boldsymbol{\Omega}_n\boldsymbol{a^{(l+1)}}_i+\boldsymbol{\Phi}_n
    \end{equation}
    Finally, utilizing Equations \eqref{RHS} and \eqref{overline{boldsymbol{a^{(l+1)}}_{(n-1)F_{l+1}+i}}}, we deduce that
    \begin{equation}
        \left[\left(\boldsymbol{C}^{tr}\otimes\boldsymbol{I}_{F_{l+1}}\right)\overline{\boldsymbol{z^{(l+1)}}}\right]_i = \sum_{n=1}^{\tau}\boldsymbol{C}_{n}\rho\left(\boldsymbol{\Omega}_n\boldsymbol{a^{(l+1)}}_i+\boldsymbol{\Phi}_n\right),
    \end{equation}
    which is equal to the RHS of the Equation \eqref{final}.

    \textbf{Part 3)} Using parts 1 and 2 of the proof, we can state the theorem for arbitrary even values of $L$. By setting $l = 1$ in the previous parts, we obtain
    \begin{equation}
        \boldsymbol{\overline{{W}^{(1)}}} = \boldsymbol{\Omega}\otimes\boldsymbol{W}^{(1)}, \quad \boldsymbol{\overline{{B}^{(1)}}}=\boldsymbol{\Phi}\otimes \boldsymbol{J}_{F_1,1}
    \end{equation}
    and
    \begin{equation}
        \boldsymbol{\overline{{W}^{(2)}}} = \left(\boldsymbol{\Omega}\otimes\boldsymbol{C}^{tr}\right)\otimes\boldsymbol{W}^{(2)},\quad \boldsymbol{\overline{{B}^{(2)}}}=\boldsymbol{\Phi}\otimes \boldsymbol{J}_{F_2,1}.
    \end{equation}
    Thus,
    \begin{equation}
        \overline{\boldsymbol{W}^{(l)}} = \begin{cases}
			\boldsymbol{\Omega}\otimes\boldsymbol{W}^{(l)}, & \text{if $l=1$}\\
            \left(\boldsymbol{\Omega}\otimes\boldsymbol{C}^{tr}\right)\otimes\boldsymbol{W}^{(l)}, & \text{if $l=2$}
		 \end{cases},\quad \boldsymbol{\overline{{B}^{(l)}}}=\boldsymbol{\Phi}\otimes \boldsymbol{J}_{F_l,1}.
    \end{equation}
    In addition, by setting $L=2$, we will have $\overline{f}_{\overline{\theta}}(\boldsymbol{r})=\overline{\boldsymbol{W}^{(3)}}~\overline{\boldsymbol{z}^{(2)}}$. Note that according to the assumptions of the theorem, $\overline{\boldsymbol{W}^{(3)}}=\boldsymbol{C}^{tr}\otimes\boldsymbol{I}_{F_2}$. As a result, $\overline{f}_{\overline{\theta}}(\boldsymbol{r})=\overline{\boldsymbol{W}^{(3)}}~\overline{\boldsymbol{z}^{(2)}}=\left(\boldsymbol{C}^{tr}\otimes\boldsymbol{I}_{F_2}\right)\overline{\boldsymbol{z}^{(2)}}$, which is equal to $\boldsymbol{z}^{(2)}=f_\theta(\boldsymbol{r})$, as derived in \eqref{final}. \eqref{final}. In conclusion, the theorem holds true for $L=2$.
    
    Now, suppose that Equation \eqref{weights_and_biases} holds for $L=2k$. Consequently,
    \begin{equation} \label{induction_h}
        \boldsymbol{z^{(2k)}}=\left(\boldsymbol{C}^{tr}\otimes\boldsymbol{I}_{F_{2k}}\right)\overline{\boldsymbol{z^{(2k)}}}
    \end{equation}
    Now, we aim to analyze the case for $L=2(k+1)$. For this network with two additional layers, we first need to adjust the weight matrix for layer $l=2k+1$. The new weight matrix will be
    \begin{equation}
        \overline{\boldsymbol{W}^{(2k+1)}} = \left(\boldsymbol{\Omega}\otimes\boldsymbol{W}^{(2k+1)}\right)\left(\boldsymbol{C}^{tr}\otimes\boldsymbol{I}_{F_{2k}}\right),
    \end{equation}
    and the weights and the biases of the two new layers will be
    \begin{align}
        \overline{\boldsymbol{W}^{(2k+2)}} = \left(\boldsymbol{\Omega}\otimes\boldsymbol{C}^{tr}\right)\otimes\boldsymbol{W}^{(2k+2)},\quad \boldsymbol{\overline{B}^{(2k+2)}}=\boldsymbol{\Phi}\otimes \boldsymbol{J}_{F_{2k+2},1},\nonumber \\
        \overline{\boldsymbol{W}^{(2k+3)}} = \boldsymbol{C}^{tr}\otimes\boldsymbol{I}_{F_{2k+2}},\quad \boldsymbol{\overline{B}^{(2k+3)}}=\boldsymbol{\Phi}\otimes \boldsymbol{J}_{F_{2k+3},1}.
    \end{align}

    Now, note that
    \begin{equation}
        \overline{\boldsymbol{W}^{(2k+1)}}~\overline{\boldsymbol{z}^{(2k)}} = \left(\boldsymbol{\Omega}\otimes\boldsymbol{W}^{(2k+1)}\right)\left(\boldsymbol{C}^{tr}\otimes\boldsymbol{I}_{F_{2k}}\right)\overline{\boldsymbol{z}^{(2k)}}.
    \end{equation}
    Therefore, by setting $l=2k-1$ in Equation \eqref{final}, or using Equation \eqref{induction_h}, we obtain
    \begin{equation}
        \overline{\boldsymbol{W}^{(2k+1)}}~\overline{\boldsymbol{z}^{(2k)}} = \left(\boldsymbol{\Omega}\otimes\boldsymbol{W}^{(2k+1)}\right)\boldsymbol{z}^{(2k)}
    \end{equation}
    This is analogous to feeding $\boldsymbol{z}^{(2k)}$ into a neural network whose first layer has the weight matrix $\boldsymbol{\Omega}\otimes\boldsymbol{W}^{(2k+1)}$.  Since the additional weight matrices and biases are consistent with Parts 1 and 2 of the proof, we can conclude that
    \begin{equation}
        \overline{f}_{\overline{\theta}}(\boldsymbol{r})=\boldsymbol{z}^{(2k+2)}=f_\theta(\boldsymbol{r}).
    \end{equation}
\end{proof}

\subsection{Proof of Lemma \eqref{rational_elusive}}
\label{app:proof_rational_elusive}
\begin{proof}
    Let $[a_{r,1},a_{r,2},\ldots,a_{r,T}]\in\mathbb{Q}^T$ be the $r$'th row of $\boldsymbol{\Psi}^{tr}$. Now, define a matrix $\boldsymbol{\hat{A}}$ which is identical to $\boldsymbol{A}$ except for its $r$'th row. This modified row is constructed as follows:
\begin{equation}
    \hat{a}_{r,i} = \frac{\sqrt{p_i}}{10^{-\eta}\lfloor 10^\eta\sqrt{p_i}\rfloor}\left(\psi_{r,i}+\epsilon[\psi_{r,i}=0]\right)
\end{equation}
in which $p_i$ is the $i$'th prime number, $\epsilon$ is the machine precision, $[.]$ is Iverson bracket, and $\eta$ is a large enough natural number such that $\frac{\sqrt{p_i}}{10^{-\eta}\lfloor 10^\eta\sqrt{p_i}\rfloor}\approx 1$ (to avoid significant changes in the matrix). At the same time, we must have $|\frac{\sqrt{p_i}}{10^{-\eta}\lfloor 10^\eta\sqrt{p_i}\rfloor}-1|\geq \epsilon$ (to prevent it from becoming a rational number).

Let $\alpha_i:=\frac{\hat{a}_{r,i}}{\sqrt{p_i}}$. Then, $\alpha_i\in\mathbb{Q}\setminus\{0\}$. Now assume that there is $S=[s_1,...,s_T]^{tr}\in Ker(\boldsymbol{\hat{A}})\cap \mathbb{Q}^{T}$. Consequently,
\begin{equation}
    \sum_{i=1}^{T}\hat{a}_{r,i}s_i=0
\end{equation}
As a result,
\begin{equation}
    \sum_{i=1}^{T}\alpha_i\sqrt{p_i}s_i=0
\end{equation}
Note that $\alpha_i s_i\in\mathbb{Q}$. Furthermore, The square roots of all prime numbers are linearly independent over $\mathbb{Q}$ \citep{stewart2022galois}. As a result, $\alpha_i s_i=0$ for all $i$. Since $\alpha_i\neq 0$, we must have $s_i=0$ for all $i$, that is, $Ker(\boldsymbol{\hat{A}})\cap \boldsymbol{Q}^T=\boldsymbol{O}$.\footnote{Note that all algebraic numbers are computable. This analysis was founded on the computability and expressibility of the square roots of prime numbers in a machine. However, most of the computable numbers are rounded or truncated when stored in a machine. Nevertheless,  it is possible to demonstrate theoretically or through simulation that increasing precision can make the aforementioned analysis always feasible.}
\end{proof}


\section{Exact Expressive Power of a 2-layer STAF Network}

In Theorem~\ref{expressive_power}, we discussed the set of potential frequencies. 
Following ~\cite{novello2025tuner}, one can determine the exact set of frequencies using the Jacobi–Anger expansion. 
By Theorem~\ref{Kronecker_theorem}, STAF admits a Kronecker-equivalent sinusoidal network. 
This allows us to directly apply the following results of ~\cite{novello2025tuner}.

\begin{theorem}[{\cite[Thm.~1]{novello2025tuner}}] \label{TUNER_Thm1}
Each hidden neuron $h_i$ of a 3-layer sinusoidal MLP has an amplitude-phase expansion of the form
\[
h_i(x) = \sum_{\mathbf{k}\in \mathbb{Z}^m} \alpha_{\mathbf{k}} \, \sin(\beta_{\mathbf{k}} x + \lambda_{\mathbf{k}}),
\]
where $\beta_{\mathbf{k}} = \langle \mathbf{k}, \omega\rangle$, $\lambda_{\mathbf{k}} = \langle \mathbf{k}, \varphi\rangle + b_i$, and
\[
\alpha_{\mathbf{k}} = \prod_j J_{k_j}(W_{ij}),
\]
with $J_k$ the Bessel functions of the first kind.
\end{theorem}

\begin{theorem}[{\cite[Thm.~2]{novello2025tuner}}] \label{TUNER_Thm2}
The magnitudes of the amplitudes in the above expansion satisfy
\[
|\alpha_{\mathbf{k}}| \leq \prod_{j=1}^m 
\left( \frac{|W_{ij}|}{2} \right)^{|k_j|}
\frac{1}{|k_j|!}.
\]
\end{theorem}

In our setting, the Kronecker equivalence established in Theorem~\ref{Kronecker_theorem} implies that a 2-layer STAF network corresponds to a 3-layer sinusoidal MLP with weights
\[
\omega = \Omega \otimes W^{(1)} \in \mathbb{R}^{\tau F_1 \times F_0},
\qquad 
W = (\Omega \otimes C_S^{tr}) \otimes W^{(2)};
\]
where $\omega \in \mathbb{R}^{m\times d}$. Thus, we have $m = \tau F_1$ and $d = F_0$. 
By adapting the above theorems to our notation, we obtain
\begin{equation}
    \alpha_{\mathbf{k}} 
= \prod_j J_{k_j}\!\left( 
\Omega_{\lceil i/F_2 \rceil,1} 
\cdot (C_S)_{\lceil j/F_1 \rceil \% \tau + 1,\,1} 
\cdot W^{(2)}_{\,i \% F_2,\, j \% F_1}
\right),
\end{equation}
and the amplitude bound
\begin{equation}
    \alpha_{\mathbf{k}} 
\leq \prod_{j=1}^{\tau F_1} 
\left( 
  \frac{(C_S)_{\lceil j/F_1 \rceil \% \tau + 1,\,1}\,
        W^{(2)}_{\, i \% F_2,\, j \% F_1}}{2}
\right)^{|k_j|}
\frac{1}{|k_j|!}.
\end{equation}

\begin{proof}
Let $C_S$ denote the matrix $C$ in Theorem~\ref{Kronecker_theorem}, and $C_T$ denote the matrix $C$ in TUNER’s notation. Similarly, we use $W$ (without superscripts or bars) to refer to weights in TUNER, while $(\overline{W^{(\cdot)}}
$ follows STAF’s notation. Then
\begin{equation}
    W = \overline{W^{(2)}} = (\Omega \otimes C_S^{tr}) \otimes W^{(2)}, 
    \qquad 
    C_T = \overline{W^{(3)}} = C_S^{tr} \otimes I_{F_2}.
\end{equation}

As a result, based on \ref{TUNER_Thm1}, the resulting function can be written as
\begin{equation}
    h_i(x) = \sum_{\mathbf{k} \in \mathbb{Z}^m} 
    \alpha_{\mathbf{k}} \, \sin(\beta_{\mathbf{k}} x + \lambda_{\mathbf{k}}),
\end{equation}
where
\begin{align} \label{alpha_k}
    \alpha_{\mathbf{k}} 
    &= \prod_j J_{k_j}\!\left( W_{i,j} \right)
\end{align}
Note that based on the \eqref{Kronecker_product_indices},
\begin{align} \label{W_ij}
    W_{ij} &= \big( (\Omega \otimes C_S^{tr}) \otimes W^{(2)} \big)_{ij} \nonumber \\
    &= (\Omega \otimes C_S^{tr})_{\lceil i/F_2 \rceil,\, \lceil j/F_1 \rceil} 
        \cdot W^{(2)}_{(i \% F_2)+1,\, (j \% F_1)+1} \nonumber \\
    &= (\Omega \otimes C_S^{tr})_{\lceil i/F_2 \rceil,\, \lceil j/F_1 \rceil / \tau}
  \cdot (C_S^{tr})_{1,\, (\lceil j/F_1 \rceil \% \tau) + 1}
  \cdot W^{(2)}_{(i \% F_2)+1,\,(j \% F_1)+1} \nonumber \\
  &= \Omega_{\left\lceil \tfrac{i}{F_2} \right\rceil,1} \cdot (C_S)_{\left\lceil \tfrac{j}{F_1} \right\rceil \% \tau + 1,\,1} \cdot W^{(2)}_{\,i \% F_2,\, j \% F_1}
\end{align}
Consequently, according to \eqref{alpha_k} and \eqref{W_ij},
\begin{align}
    &\alpha_{\mathbf{k}} = \prod_{j} J_{k_j} \Bigl( \Omega_{\left\lceil \tfrac{i}{F_2} \right\rceil,1} \cdot (C_S)_{\left\lceil \tfrac{j}{F_1} \right\rceil \% \tau + 1,\,1} \cdot W^{(2)}_{\,i \% F_2,\, j \% F_1} \Bigr).
\end{align}

In addition, the frequency term can be written as
\begin{equation}
    \beta_{\mathbf{k}} 
= \langle\mathbf{k}, \omega \rangle
= \big\langle \mathbf{k},\, \Omega \otimes W^{(1)} \big\rangle
= \sum_{\ell=1}^m k_\ell \, (\Omega \otimes W^{(1)})_\ell.
\end{equation}

\medskip
\noindent
Based on Theorem \ref{TUNER_Thm2}, the amplitudes satisfy
\begin{equation}
    \alpha_{\mathbf{k}} 
\leq \prod_{j=1}^m 
\left( \frac{|W_{ij}|}{2} \right)^{|k_j|}
\frac{1}{|k_j|!}.
\end{equation}

Therefore, using \eqref{W_ij},
\begin{equation}
    \alpha_{\mathbf{k}}
\leq \prod_{j=1}^{\tau F_1} 
\left( 
  \frac{|\Omega_{\left\lceil i/F_2 \right\rceil,1}\,(C_S)_{\lceil j/F_1 \rceil \% \tau + 1,\,1}\,
        W_{i \% F_2,\, j \% F_1}^{(2)}|}{2}
\right)^{|k_j|}
\frac{1}{|k_j|!}.
\end{equation}

Grouping terms, we have
\begin{equation}
    \prod_{j=1}^{\tau F_1} 
\bigl( (C_S)_{\lceil j/F_1 \rceil \% \tau + 1,\,1} \bigr)^{|k_j|}
= \prod_{i=1}^\tau (C_S)_{i,1}^{\,\sum_{j=1}^{\tau F_1} |k_j|}.
\end{equation}
Therefore,
\begin{equation}
    \alpha_{\mathbf{k}} \leq 
\left| \Omega_{\lceil i/F_2 \rceil, 1} 
\prod_{i=1}^{\tau} (C_S)_{i,1} \right|^{\sum_{j=1}^{\tau F_1} |k_j|}
\prod_{j=1}^{\tau F_1} \left( \frac{|W_{\,i \% F_2, \, j \% F_1}^{(2)}|}{2} \right)^{|k_j|} 
\frac{1}{|k_j|!}
\end{equation}
\end{proof}

\section{An Alternative Proof of the Asymptotic Behavior of Delannoy Numbers} \label{Asymptotic_proof}
\begin{proof}
	Let $V(A,B)$ denote a Delannoy number. This number represents the number of lattice points inside and on the $L1$ sphere (or equivalently, the cells in the von Neumann neighborhood) in $A$ dimensions with radius $B$. We also know that one of the properties of these numbers is symmetry with respect to the arguments \cite{kiselman2012asymptotic}; that is, for any $A$ and $B$, we have:
	\begin{equation}
		V(A,B) = V(B,A).
	\end{equation}
	
	Hence,
	\begin{equation}
		\frac{|V(\tau T,K)|}{|V(T,K)|} = \frac{|V(K,\tau T)|}{|V(K,T)|}.
	\end{equation}
	The second ratio means that if the radius of the sphere is scaled by $\tau$, the number of lattice points inside and on the $L1$ sphere is multiplied accordingly. It is known that the number of lattice points can be approximated using the volume of the sphere. Moreover, when the radius is scaled by $\tau$ in $K$ dimensions, the volume of an object is multiplied by $\tau^K$. Consequently, the number of lattice points also grows approximately by a factor of $\tau^K$. Therefore,
	\begin{equation}
		\frac{|V(\tau T,K)|}{|V(T,K)|} \sim \tau^K.
	\end{equation}
\end{proof}

\section{Discussion and Limitations}
\textbf{When does STAF help?} Tasks whose target signals contain mixed frequencies or repetitive fine detail benefit most, especially when positional encodings are absent.

\textbf{Tuning.} Very large $\tau$ increases compute and may slow late-phase convergence. We recommend small $\tau$ with layer-wise sharing and the unit-variance initialization.

\textbf{What we do not claim.} We do not claim to eliminate spectral bias or to dominate all baselines on all metrics. Our results indicate consistent improvements in convergence and fidelity across diverse settings, with clear but bounded limitations.

\added{\textbf{Perceptual quality.} While STAF consistently improves distortion-based metrics such as PSNR and SSIM, it does not always achieve the best LPIPS. We therefore distinguish between distortion fidelity and perceptual quality, rather than claiming uniform superiority across reconstruction metrics. In an additional image-fitting experiment on~\cref{fig:celtic}, STAF trained with the original MSE objective achieves 40.47 PSNR, 0.9804 SSIM, and 0.00577 LPIPS, whereas adding an LPIPS term improves LPIPS to 0.00305 but reduces PSNR and SSIM to 37.47 and 0.9616, respectively. This behavior is consistent with the standard distortion--perception trade-off: optimizing more strongly for perceptual similarity can improve LPIPS, but often at the expense of pixel-wise reconstruction fidelity. These results suggest that perceptual or hybrid objectives may provide a more suitable operating point when perceptual quality is prioritized.}

\begin{figure}[t]
    \centering
    \includegraphics[width=0.8\textwidth]{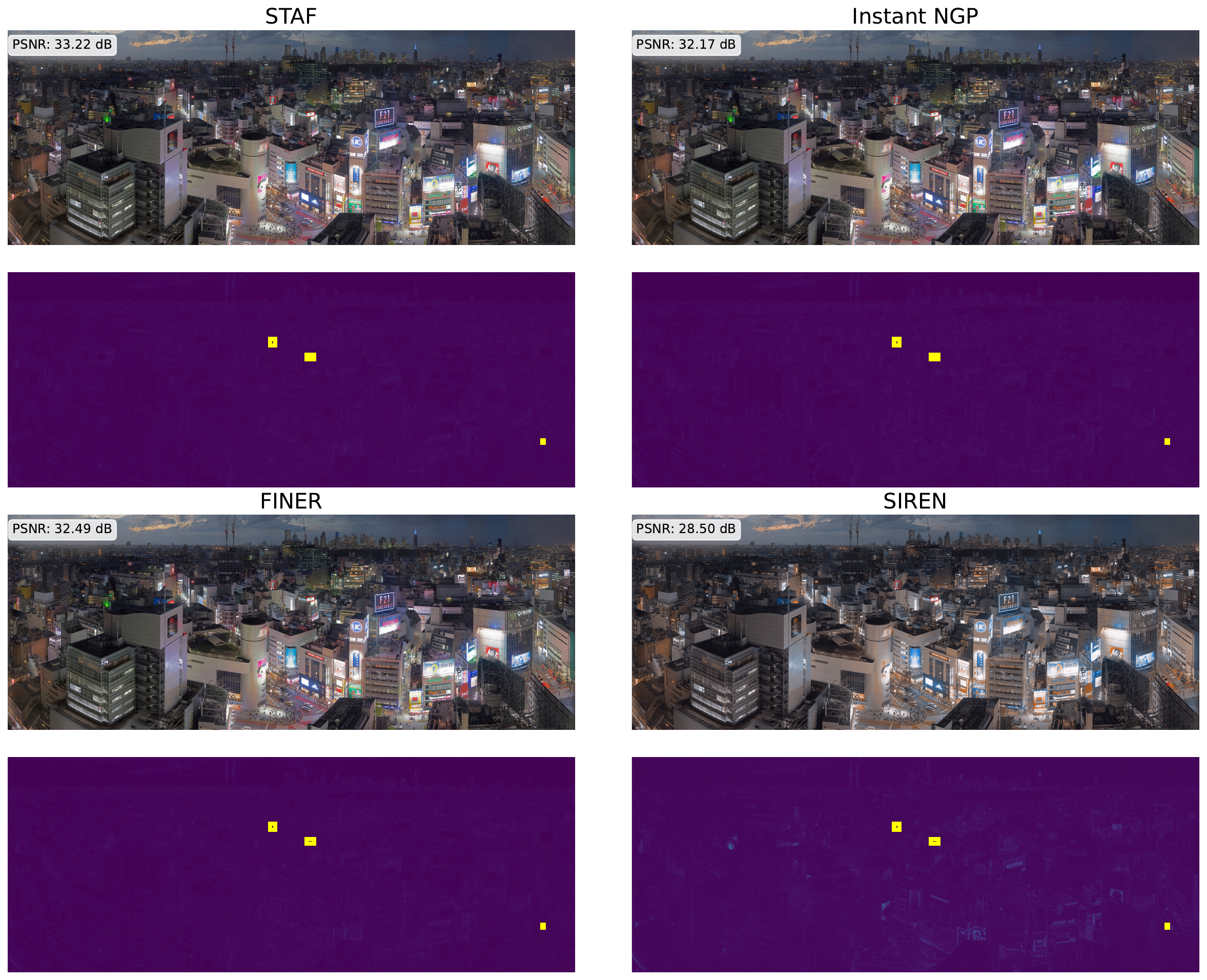}
    \includegraphics[width=0.8\textwidth]{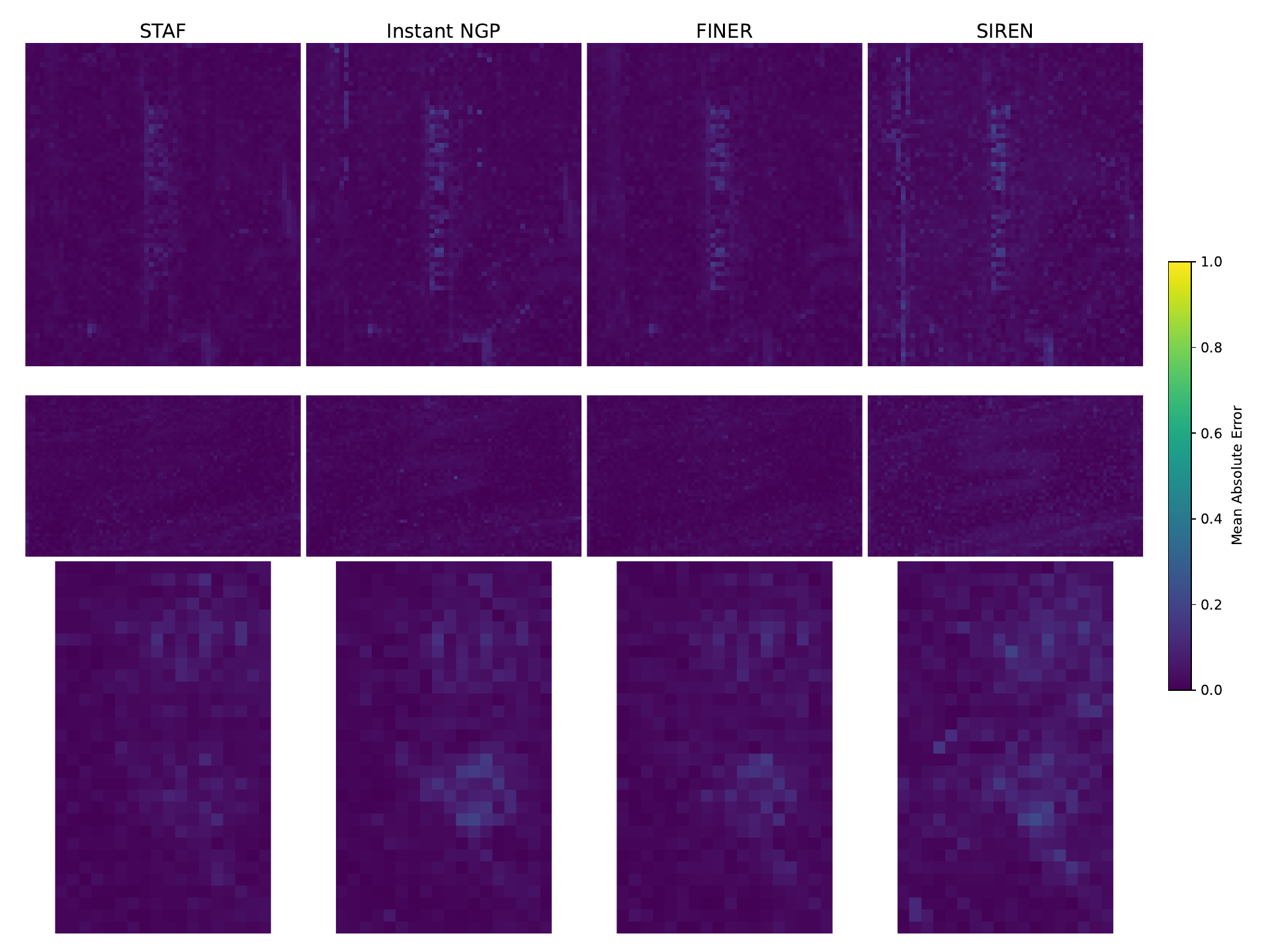}
    \caption{Comparative visualization of Tokyo image reconstruction and their error map using \textbf{STAF} and other activation functions. Yellow zoom-in regions are displayed in the last three rows.}
    \label{fig:tokyo}
\end{figure}

\end{document}